\documentclass{article}

\PassOptionsToPackage{numbers, compress}{natbib}

\usepackage[preprint]{neurips_2023}

\usepackage[subrefformat=parens]{subcaption}

\usepackage[utf8]{inputenc} 
\usepackage[T1]{fontenc}    
\usepackage{hyperref}       
\usepackage{url}            
\usepackage{booktabs}       
\usepackage{amsfonts}       
\usepackage{nicefrac}       
\usepackage{microtype}      
\usepackage{xcolor}         

\usepackage{amsmath}
\usepackage{amssymb}
\usepackage{mathtools}
\usepackage{amsthm}
\usepackage{mathrsfs}
\usepackage{DotArrow}

\usepackage{algpseudocode}
\usepackage{algorithm,multicol}

\theoremstyle{plain}
\newtheorem{theorem}{Theorem}[section]
\newtheorem{proposition}[theorem]{Proposition}
\newtheorem{lemma}[theorem]{Lemma}
\newtheorem{corollary}[theorem]{Corollary}
\theoremstyle{definition}
\newtheorem{definition}[theorem]{Definition}

\theoremstyle{remark}

\newcommand{\indep}{\mathrel{\perp\mspace{-10mu}\perp}}

\title{Generalized Balancing Weights via Deep Neural Networks}

\author{%
  Yoshiaki Kitazawa \\
  NTT DATA Mathematical Systems Inc.\\
  Data Mining Division\\
  1F Shinanomachi Rengakan, 35, Shinanomachi\\
  Shinjuku-ku, Tokyo, 160-0016, Japan\\
  \texttt{kitazawa@msi.co.jp}
}

\begin{document}

\maketitle

\begin{abstract}
    Estimating causal effects from observational data is a central problem in many domains.
    A general approach is to balance covariates with weights such that the distribution of the data mimics randomization.
    We present generalized balancing weights, {\em Neural Balancing Weights} (NBW), to estimate the causal effects of an arbitrary mixture of discrete and continuous interventions.
    The weights were obtained through direct estimation of the density ratio between the source and balanced distributions by optimizing the variational representation of $f$-divergence.
    For this, we selected $\alpha$-divergence as it presents efficient optimization because it has an estimator whose sample complexity is independent of its ground truth value and unbiased mini-batch gradients;  moreover, it is advantageous for the vanishing-gradient problem.
    In addition, we provide the following two methods for estimating the balancing weights: improving the generalization performance of the balancing weights and checking the balance of the distribution changed by the weights.
    Finally, we discuss the sample size requirements for the weights as a general problem of a curse of dimensionality when balancing multidimensional data.
    Our study provides a basic approach for estimating the balancing weights of multidimensional data using variational $f$-divergences.
\end{abstract}

\section{Introduction}
Estimating causal effects from observational data is a central problem in many application domains,
including public health, social sciences, clinical pharmacology, and clinical decision-making.
One standard approach is balancing covariates with weights that are the same as the density ratios between the source
and  balanced distributions, such that their distribution mimics randomization.
Many methods have been developed to estimate the balancing weights, such as
inverse propensity weighting (IPW) \citet*{rosenbaum1984reducing},
augmented inverse propensity weighting (AIPW) \cite{robins1994estimation},
generalized propensity score (GPS) \cite{imai2004causal},
covariate balancing propensity score (CBPS) \cite{imai2014covariate},
overlap weighting \cite{li2018balancing}, and
entropy balancing (EB) \cite{hainmueller2012entropy,tubbicke2022entropy}.
However, these methods are limited to categorical or continuous interventions.

In this study, we propose generalized balancing weights to estimate the causal effects of an arbitrary mixture of discrete and continuous interventions.
To the best of our knowledge, no causal inference method focusing on the balancing weights exists for this problem.
We approach this problem by directly estimating the density ratio, more precisely, the Radon--Nikod\'{y}m derivatives,
between the source and balanced distributions using a neural network algorithm by optimizing a variational representation of a $f$-divergence.
$f$-divergences, whose values are greater than or equal to zero and considered zero if the two distributions are equal, are the statistics used to measure the closeness of the two distributions.
The optimal functions for the variational representations derived from $f$-divergences with the Legendre transform correspond to the density ratio between the distributions \cite{nguyen2010estimating}.
An approach to estimate the density ratio by optimizing a variational representation of a $f$-divergence was developed in the domain adaptation region \cite{sugiyama2012density}.

However, optimizing the $f$-divergences, including estimating the density ratio, is challenging.
This is due to the following reasons.
First, for KL-divergence, the dominant $f$-divergence, 
the requirements for sample size increase exponentially with the true amount of the divergence \cite{mcallester2020formal,song2019understanding}.
Second, a naive gradient estimate over mini-batch samples leads to a biased estimate of the full gradient \cite{bellemare2017cramer}.
Third, gradients of neural networks often vanish when the estimated probability ratios are close to zero \cite{arjovsky2017towards}.

To avoid the first problem, we focus on $\alpha$-divergence, which is a subgroup of $f$-divergence.
$\alpha$-divergence has an estimator whose sample complexity is independent
of its ground truth value and unbiased mini-batch gradients.
In addition, by selecting $\alpha$ from a particular interval, we avoid
vanishing gradients of neural networks when the neural networks reach extreme local minima.

In addition, we provide two techniques for estimating the balancing weights.
First, we propose a validation method using test data and an early stopping method to improve the generalization performance of balancing.
The generalization performance of the weights worsens as the dimensions of the data increase, and the sample size requirements of the weights increase exponentially with the dimensions.
Next, we present a method for measuring  the performance of balancing weights by estimating the $\alpha$-divergence information to check the balance of the distribution,

This study is divided into seven parts. First, we introduce the background of the study.
Second, we review related studies.
Third, we define the terminology and concepts for causal inferences.
Fourth, we present our novel method for estimating balancing weights.
Fifth, we provide techniques  for estimating the weights.
Sixth, we discuss the sample requirements for the weights.
Finally, we conclude this paper.
All the numerical experiments and proofs are described in the appendix.

\section{Related Work}

\paragraph{Balancing weight.}
Many methods have been proposed to estimate the balancing weights.
The following methods are proposed for binary intervention:
inverse propensity weighting (IPW) \cite{rosenbaum1984reducing},
augmented inverse propensity weighting (AIPW) \cite{robins1994estimation},
covariate balancing propensity score (CBPS) \cite{imai2014covariate}
and overlap weighting \cite{li2018balancing}.
The following methods have been proposed for continuous intervention:
generalized propensity score (GPS) \cite{imai2004causal}
and entropy balancing (EB) \cite{hainmueller2012entropy,tubbicke2022entropy}.
Recently, \citeauthor*{lee2022covariate} (\citeyear{lee2022covariate}) proposed balancing weights available for discrete or continuous intervention
but only one of them at a time.

\paragraph{Statistical divergences and density ratio estimation.}
Despite the abundance of classic studies \cite{nguyen2007estimating,sugiyama2012density},
we focused on studies that directly estimate density ratios or optimize statistical divergences using neural networks.
In this review, these studies have beenclassified into four groups.
First is the estimation of KL-divergence or mutual information
\cite{belghazi2018mutual,oord2018representation,poole2019variational};
the second is density ratio estimation \cite{kato2021non};
the third is generative adversarial networks (GANs) \cite{nowozin2016f,uehara2016generative,ganin2016domain,uppal2019nonparametric} (statistical divergences were used as discriminators for GANs); and the
fourth is domain generation \cite{si2009bregman,ganin2016domain,zhang2019bridging,acuna2021f}.
In addition to these application studies, divergences were improved \cite{birrell2022optimizing}.

\section{Terminologies and Definitions}
Here, we briefly introduce the terminology and definitions used in this study.

\paragraph{Notations and Terminologies.}\label{sec_term_def}
Random variables are denoted by capital letters; for example, $A$.
Small letters are used for the values of random variables of the corresponding capital letters;
$a$ is the value of the random variable $A$.
Bold letters $\mathbf{A}$ or $\mathbf{a}$ represent a set of variables or random variable values.
In particular, $\mathbf{V}=\{V_1,\dots,V_n\}$ are used for the observed random variables
and $\mathbf{U}=\{U_1,\dots, U_m\}$ are used as unobserved random variables.
For example, the domain of the variable $A$ is denoted by $\mathcal{X}_A$,
and $\mathcal{X}_{A_1} \times \cdots \times \mathcal{X}_{A_n}$ is denoted by $\mathcal{X}_\mathbf{A}$
for $\mathbf{A}= A_1 \times \cdots \times A_n$.
$\mathbf{V}\cup\mathbf{U}$ are assumed to be semi-Markovian models
and $G=G_{\mathbf{V}\mathbf{U}}$ denotes the causal graph for $\mathbf{V}\cup\mathbf{U}$.
$Pa(\mathbf{A})_G$, $Ch(\mathbf{A})_G$, $An(\mathbf{A})_G$, and $De(\mathbf{A})_G$
represent parents, children, ancestors, and descendants of the observed variables in $G$, respectively,
for $\mathbf{A}\subset\mathbf{V}$.
In this study, $Pa(\mathbf{A})_G$, $Ch(\mathbf{A})_G$,
$An(\mathbf{A})_G$, and $De(\mathbf{A})_G$ do not include $\mathbf{A}$.
$P$ and $Q$ are used as the probability measures on $(\mathbb{R}^{d}, \mathscr{F})$, where $\mathscr{F}$ denotes  the $\sigma$-algebra of
subsets of $\mathbb{R}^{d}$. 
$E_P[\cdot]$ and $E_P[\cdot|\cdot]$ denote expectation and conditional expectation
under the distribution $P$, respectively. For example,
$E_P[\mathbf{X}]=\int_{\mathcal{X}_{\mathbf{X}}}dP$ and
$E_P[\mathbf{Y}|\mathbf{X}]=\int_{\mathcal{X}_{\mathbf{Y}}}dP(\mathbf{Y}|\mathbf{X})$.
$\hat{E}_{P}[\cdot]$ denotes the empirical expectation under $P$; that is, the sample mean of the finite observations drawn from $P$.
$P$ is called \emph{absolute continuous} with respect to $Q$,
$P(A)=0$ whenever $Q(A)=0$ for any $A \in \mathscr{F}$, which is represented as $P \ll Q$.
$\frac{dP}{dQ}$ denotes the Radon--Nikod\'{y}m derivative of $P$ with respect to $Q$ for $P$ and $Q$ with $P \ll Q$. 
In this study, we refer to density ratios as the Radon--Nikod\'{y}m derivatives.
$\mu$ denotes a probability measure on $\mathbb{R}^{d}$ with  $P \ll \mu$ and $Q \ll \mu$.
$\mathbf{X}^{(N)} = \{\mathbf{X}^{1}, \ldots, \mathbf{X}^{N} \}$ denotes $N$ i.i.d. random variables from $\mu$.
$\mathbf{X}_{P}^{(N)} = \allowbreak \{\mathbf{X}^{1}_{\sim P}, \allowbreak \ldots, \allowbreak \mathbf{X}^{N}_{\sim P}\}$ and $\mathbf{X}_{Q}^{(N)} = \allowbreak \{\mathbf{X}^{1}_{\sim Q}, \allowbreak \ldots, \allowbreak \mathbf{X}^{N}_{\sim Q}\}$ denote variables defined as $P(\mathbf{X}^{i}_{\sim P} \le \mathbf{x}) = \allowbreak \mu(\mathbf{X}^{i} \le \mathbf{x})$ and $Q(\mathbf{X}^{i}_{\sim Q} \le \mathbf{x}) = \allowbreak \mu(\mathbf{X}^{i} \le \mathbf{x})$, $\forall \mathbf{x} \in \mathbb{R}^{d}$, for $1\le i \le N$.
We represent $f \lesssim g$ when $\lim \sup_{n \rightarrow \infty}f(n)/g(n) < \infty$ holds.
The notation $f \gtrsim g$ is defined similarly.

\subsection{Definitions}
In this study, we considered the causal effects of joint and multidimensional interventions.
For clarity, we used different notations, “$do$” and  “$\overline{do}$”,
for single-dimensional and multidimensional interventions, respectively.
\footnote{The values of the variables in the parentheses for both symbols can be dropped if not necessary in the context.
    For example, we sometimes represent $do(\mathbf{X}=\mathbf{x})$ or $\overline{do}(\mathbf{X}=\mathbf{x})$
    as $do(\mathbf{X})$ or $\overline{do}(\mathbf{X})$, respectively.}
For a single-dimensional intervention, a $do$ symbol is used, which is the same as Pearl's $do$-calculation.
\begin{definition}[$do$-calculation, \citeauthor{pearl2009causality}(\citeyear{pearl2009causality})]\label{Def_do_calculation}
    For the two given disjoint sets of variables $\mathbf{X}, \mathbf{Y} \subset \mathbf{V}$,
    the causal effect on $\mathbf{Y}$ for intervention in $\mathbf{X}$ with values $\mathbf{x}$,
    denoted by $P(\mathbf{Y}|do(\mathbf{X}=\mathbf{x}))$, is defined as the probability distribution, such that
    \begin{equation}
        P(\mathbf{Y}|do(\mathbf{X}=\mathbf{x})) =
        \sum_{\substack{\mathbf{v}' \in \mathcal{X}_{\mathbf{V}'}\\
                pa_\mathbf{x} \in  \mathcal{X}_{Pa(\mathbf{X})_G}}}
        \frac{
            P(\mathbf{Y}, \mathbf{X}=\mathbf{x},Pa(\mathbf{X})_G=pa_\mathbf{x}, \mathbf{V}'=\mathbf{v}')}{
            P(\mathbf{X}=\mathbf{x}|Pa(\mathbf{X})_G=pa_\mathbf{x})
        }, \label{def_causal_effect_no_condi_eq}
    \end{equation}
    where $\mathbf{V}'=\mathbf{V}\setminus(\mathbf{X}\cup Pa(\mathbf{X})_G \cup \mathbf{Y})$. 
    The causal effect of $\mathbf{X}$ on $\mathbf{Y}$ under the conditions $\mathbf{Z}$
    denoted by $P(\mathbf{Y}=\mathbf{y}|do(\mathbf{X}=\mathbf{x}),\mathbf{Z}=\mathbf{z})$ is defined as
    the probability distribution, such that
    \begin{equation}
        P(\mathbf{Y}=\mathbf{y}|do(\mathbf{X}=\mathbf{x}),\mathbf{Z}) =
        \frac{
            P(\mathbf{Y}=\mathbf{y},\mathbf{Z}|do(\mathbf{X}=\mathbf{x}))}{
            P(\mathbf{Z}|do(\mathbf{X}=\mathbf{x}))
        }. \label{def_causal_effect_on_condi_eq}
    \end{equation}
\end{definition}
Notably, from Definition \ref{Def_do_calculation}, a $do$-calculation for a set of variables
coincides with the simultaneous interventions for each variable:
\begin{equation}
    P(\mathbf{Y}|do(\mathbf{X})) = P(\mathbf{Y}|do(X_1), do(X_2), \ldots, do(X_n)), \label{Eq_single_dimensional_do}
\end{equation}
where $\mathbf{X}=\{X_1,X_2,\dots, X_n\}$.
Here, we refer to each intervention in Eq. (\ref{Eq_single_dimensional_do}) as a “single-dimensional intervention” .

Furthermore, we use the $\overline{do}$ symbol for multidimensional intervention.
Intuitively, a $\overline{do}$ symbol represents the intervention of the variables
that preserves the functional relationship within the variables.
\begin{definition}[$\overline{do}$ symbol]\label{Def_do_var_symbol}
    For disjoint sets of variables $\mathbf{X}_1, \mathbf{X}_2, \dots,\mathbf{X}_n, \mathbf{Y} \subset \mathbf{V}$,
    $\overline{do}$ symbol defines the following probability distribution over them:
    \begin{equation}
        P(\mathbf{Y}, \overline{do}(\mathbf{X_1}),
        \overline{do}(\mathbf{X_2}), \ldots, \overline{do}(\mathbf{X_n}))
        =P(\mathbf{Y}|do(\mathbf{X})) \times P(\mathbf{X_1}) \times
        P(\mathbf{X_2}) \times \cdots \times P(\mathbf{X_n}),
    \end{equation}
    where $\mathbf{X}=\mathbf{X}_1 \cup \mathbf{X}_2 \cup \dots \cup \mathbf{X}_n$.
\end{definition}

$\overline{do}$ symbols are useful, particularly when we consider interventions
in a multivalued discrete variable expressed using one-hot encoding.
In this case, we cannot express the causal effect effectively using $do$ symbols.
For example, let us consider the case of an intervention in the ternary variable $X$,
$\mathcal{X}_X=\{x_1, x_2, x_3\}$ and let $X$
be expressed by $\mathbf{X}'=(X'_1,X'_2,X'_3)$, such that $X'_i=1$ if $X=x_i$
otherwise $X'_i=0$ for $i=1,2,3$.
Then, $P(\mathbf{Y}|do(X=x_3))$ is the same as $P(\mathbf{Y},\overline{do}(\mathbf{X}'=(0,0,1)))$, which differs from $P(\mathbf{Y}|do(\mathbf{X}'=(0,0,1)))$.
We refer to this type of intervention as a “multidimensional intervention” .

Next, we provide  definitions of the $f$-divergence and $f$-divergence information.
\begin{definition}[$f$-divergence]
    The $f$-divergence $D_{f}$ between the two probability measures $P$ and $Q$ with $Q \ll P$  induced by a convex function $f$ satisfying $f(1) = 0$ is defined by
    $D_{f}(Q||P)=E_{P}[f(dQ/dP)]$.
\end{definition}
Many divergences are specific cases obtained by selecting a suitable generator function $f$.
For example, $f(u) = u \log u$ corresponds to the KL-divergence.
In particular, we focus on $\alpha$-divergence, which is expressed as follows:
\begin{equation}
    D_{\alpha}(Q||P)=E_{P}\left[\frac{1}{\alpha(\alpha-1)} \left\{ \left( \frac{dQ}{dP}\right)^{1 - \alpha} - 1 \right\}  \right], \label{Eq_alpha_div_def}
\end{equation}
where $\alpha \in \mathbb{R} \setminus \{0, 1\}$.
From Eq. (\ref{Eq_alpha_div_def}), Hellinger divergence is obtained as $\alpha=1/2$, and $\chi^2$ divergence by $\alpha=-1$.

From $f$-divergence, the $f$-divergence information is defined as the mutual information
if we choose the KL-divergence as the $f$-divergence.
Here, we present a definition of $f$-divergence information for multi-variables.
\begin{definition}[$f$-divergence information]\label{Def_Mutual_information}
    For disjoint sets of variables $\mathbf{X} = \{\mathbf{X}_1, \mathbf{X}_2, \ldots, \mathbf{X}_n\} \subset \mathbf{V}$,
    let $P_{\mathbf{X}}$ be the joint probability measure for $\mathbf{X}$.
    For each $i=1,2,\ldots,n$, $P_{\mathbf{X}_i}=\int_{\mathcal{X}_{\mathbf{X} \setminus \mathbf{X}_i}} dP_{\mathbf{X}}$
    is a measure of the marginal distribution of $P_{\mathbf{X}}$ for $\mathbf{X}_i$.
    The $f$-divergence information for $\mathbf{X}_1, \mathbf{X}_2, \ldots, \mathbf{X}_n$ under $P_{\mathbf{X}}$
    and a convex function $f$ satisfying $f(1) = 0$ is defined as the $f$-divergence between
    $P_{\mathbf{X}}$ and $P_{\mathbf{X}_1} \times P_{\mathbf{X}_2} \times \cdots \times P_{\mathbf{X}_n}$:
    \begin{eqnarray}
        I_{f}(\mathbf{X}_1, \mathbf{X}_2, \ldots, \mathbf{X}_n;P_{\mathbf{X}})
        &=& E_{P_{\mathbf{X}}}
        \left[
        f\left(
        \frac{dP_{\mathbf{X}_1} \times dP_{\mathbf{X}_2} \times \cdots \times dP_{\mathbf{X}_n}}{dP_{\mathbf{X}}}
        \right)
        \right]. \label{Eq_Mutual_information_Def}
    \end{eqnarray}
\end{definition}

\section{Problem Set Up}
Before describing the details of the problem, we provide a notation for the probability distribution, which is the goal of balancing.
Hereafter, $P$ denotes the probability distribution of observational data.
For the given disjoint sets $\mathbf{X}_1, \mathbf{X}_2, \dots,\mathbf{X}_n, \mathbf{Y}, \mathbf{Z} \subset \mathbf{V}$,
let $\widetilde{P}$ be a probability distribution, as follows:
\begin{eqnarray}
    \widetilde{P}
    &=& P(\mathbf{Y}, \overline{do}(\mathbf{X_1}),
    \overline{do}(\mathbf{X_2}), \ldots,
    \overline{do}(\mathbf{X_n}), \mathbf{Z}) \nonumber  \\
    &=& P(\mathbf{Y}|do(\mathbf{X}), \mathbf{Z}) \times P(\mathbf{X_1}) \times P(\mathbf{X_2}) \times  \cdots   \times P(\mathbf{X_n}) \times P(\mathbf{Z}),
    \label{Eq_def_p_tilde}
\end{eqnarray}
where $\mathbf{X}=\mathbf{X}_1 \cup \mathbf{X}_2 \cup \dots \cup \mathbf{X}_n$.
$\widetilde{P}$ is the probability distribution of the counterfactual data
from simultaneous (multidimensional) interventions in $\mathbf{X_1}, \mathbf{X_2}, \ldots, \mathbf{X_n}$
under the condition $\mathbf{Z}$.

\paragraph{Objective.}\label{problemSetUpAndRelatedWork}
The objective of this study is to obtain the balancing weights that transform $P(\mathbf{Y},\mathbf{X},\mathbf{Z})$ into $\widetilde{P}(\mathbf{Y},\mathbf{X},\mathbf{Z})$.
More precisely, given the i.i.d. observational data $\{(\mathbf{x}^i,\mathbf{z}^i)|i=1,2,\ldots,N\}$,
we aim to estimate the weights $\mathit{BW}(\mathbf{X},\mathbf{Z})$, such that
\begin{equation}
    E_{\widetilde{P}}\left[f(\mathbf{X},\mathbf{Z})\right] = E_{P}\left[f(\mathbf{X},\mathbf{Z}) \cdot \mathit{BW}(\mathbf{X},\mathbf{Z}) \right]
    \label{our_objective_weights}
\end{equation}
holds for any measurable function $f$ on $\mathbb{R}^d$.
If we obtain the weights, we estimate the conditional average causal effect (CACE) for
$P(\mathbf{Y},\overline{do}(\mathbf{X_1}), \overline{do}(\mathbf{X_2}), \ldots, \allowbreak
\overline{do}(\mathbf{X_n}), \mathbf{Z})$, that is $E_{\widetilde{P}}[\mathbf{Y}|\mathbf{X},\mathbf{Z}]$,
using state-of-the-art supervised machine learning algorithms, with the weights assigned as the individual weights for each sample.

\paragraph{Assumptions.}\label{problemSetUpAndRelatedWork}
We assumed the following to achieve our objective:
\begin{itemize}
    \item Assumption 1. The causal effect
    $P(\mathbf{Y}|do(\mathbf{X}))$ is identifiable, or equivalently, $\widetilde{P}$ from Eq. (\ref{Eq_def_p_tilde}) can be identified.
    \footnote{The simplest case that satisfies Assumption 1 is that   no confounding exists among the data (\cite{pearl2009causality}, P78, Theorem 3.2.5).}
    \footnote{If certain unobserved data are assumed to exist, the identifiability of the causal effect is determined by the structure of the causal diagram for $P$.
        One criterion for the identifiability of a causal effect is expressed by \cite{shpitser2012identification}.
        The discussion of the identifiability of the causal effect is beyond the scope of this study.}
    \item Assumption 2.
    Let  $\mathbb{P}$ = $P(\mathbf{X}_1,\mathbf{X}_2,\dots, \mathbf{X}_n, \mathbf{Z})$
    and let $\mathbb{Q}$ = $P(\mathbf{X_1}) \times P(\mathbf{X_2}) \times \cdots \times P(\mathbf{X_n}) \times P(\mathbf{Z})$.
    Subsequently, we assume that  $\mathbb{Q} \ll \mathbb{P}$.
\end{itemize}
Assumption 2 is the same as the \emph{overlap} assumption if
we consider this a single-dimensional intervention.
Here, we propose overlapped assumptions for joint and multidimensional interventions.

\section{Estimation of Balancing Weights}\label{Section_Estimation_of_Balancing_Weights}
In this section, we present the way to effectively estimate the probability density ratios by optimizing $f$-divergence.

\paragraph{Density Ratios as Balancing Weights.}
We first note that the density ratios,
which are referred to as the Radon--Nikod\'{y}m derivative in this paper,
are equal to the balancing weight of the target.
For a density ratio of $P$ to $\widetilde{P}$,
that is $\frac{d\widetilde{P}}{dP}$, it holds that
\begin{equation}
    E_{\widetilde{P}}[f] = \int f \cdot  \frac{d\widetilde{P}}{dP} \cdot dP  = E_{P}\left[f \cdot \frac{d\widetilde{P}}{dP}\right], \label{Eq_P_tilde_to_P}
\end{equation}
for any measurable function $f$ in $\mathbb{R}^d$.
Then, Eq. (\ref{our_objective_weights}) and Eq. (\ref{Eq_P_tilde_to_P}) are equivalent.
As an example of the aforementioned density ratio, let $X$ be a binary variable with
$\mathcal{X}_X=\{1, 0\}$ and let $\mathbf{Z}$ be covariates.
Using propensity score $e(\mathbf{z})=P(X=1|\mathbf{Z}=\mathbf{z})$,
we observe that $\frac{d\widetilde{P}}{dP}(X=1,\mathbf{z}) = P(X=1)/e(\mathbf{z})$
and $\frac{d\widetilde{P}}{dP}(X=0,\mathbf{z}) = P(X=0)/(1-e(\mathbf{z}))$.
That is, $\frac{d\widetilde{P}}{dP}$ is the stabilized inverse probability of the treatment weighting \cite{robins2000marginal}.

\subsection{Our Approach}\label{section_our_method} 
Our approach involves obtaining the density ratios as an optimal function for a variational representation of an $f$-divergence.
This approach is based on the fact that the optimal function is connected to density ratios \cite{nguyen2007estimating}.

\paragraph{Variational representation.} Using the Legendre transform of the convex conjugate of a twice differentiable convex function $f$,
$f^*(\psi) = \sup_{r\in \mathbb{R}}\{\psi \cdot r - f(r)\}$,  we obtain a variational representation of $f$-divergence:
\begin{equation}
    D_{f} (Q||P) = \sup_{\phi \ge 0}\{ E_Q[f'(\phi)] - E_P[f^*(f'(\phi)) ]\}, \label{Eq_Variational_representation_f_div_phi}
\end{equation}
where supremum  is considered over all  measurable functions with $ E_Q[f'(\phi)] < \infty$ and $E_P[f^* (f'(\phi))] < \infty$.
The maximum value is achieved at $\phi=dQ/dP$.
\footnote{Eq. (\ref{Eq_Variational_representation_f_div_phi}) holds only for differentiable convex functions.
	For a general statement of a variational representation of $f$-divergence, for example,
	see \citeauthor{nguyen2007estimating}(\citeyear{nguyen2007estimating}).}

We obtained the optimal function for Eq. (\ref{Eq_Variational_representation_f_div_phi})
by replacing $\phi$ in the equation with a neural network model $\phi_{\theta}$ and training it
through back-propagation with a loss function, such that
\begin{equation}
    \mathcal{L}(\theta) = - \left\{ \hat{E}_Q[f'(\phi_{\theta})] - \hat{E}_P[f^* (f'(\phi_{\theta}))] \right\}. \label{Eq_loss_func_f_div}
\end{equation}

\paragraph{Selecting $\alpha$-divergence for Optimization.}
We select $\alpha$-divergence for the following reasons.
First, the sample size requirements for $\alpha$-divergence is independent of its ground truth value;
second, it has unbiased mini-batch gradients;
third, it can avoid a vanishing gradient problem.

The variational representation of $\alpha$-divergence is as follows (Lemma \ref{lemma_alpha_div_resp} in Appendix \ref{section_appendix_poofs}):
\begin{eqnarray}
    D_{\alpha} (Q||P) = \sup_{\phi \ge 0} \left\{
    \frac{1}{\alpha(1-\alpha)}
    - \frac{1}{\alpha} E_{Q} \left[\phi^{-\alpha}  \right]  - \frac{1}{1 - \alpha} E_{P} \left[\phi^{1 - \alpha} \right] \right\}. \label{Eq_alpha_variation}
\end{eqnarray}

\paragraph{Sample size requirements for $\alpha$-divergence.}
The $\alpha$-divergence has an estimator with sample complexity $O(1)$ (Corollary 1 in \citeauthor{birrell2022optimizing}, \citeyear{birrell2022optimizing}, P19;
Corollary \ref{col_alpha_half_sample_requirement} in Appendix \ref{section_appendix_poofs}).
Conversely, the sample complexity of KL-divergence is $O(e^{KL(Q||P)})$ \cite{mcallester2020formal,song2019understanding}:

\begin{equation}
    \lim_{N\rightarrow \infty} \frac{N \cdot \mathrm{Var} \Big[  \widehat{KL^N} (Q||P) \Big]}{KL(Q||P)^2} \ge \frac{e^{KL(Q||P)}-1}{KL(Q||P)^2},
\end{equation}
where $\widehat{KL^N}(Q||P)$ is the KL-divergence estimator for sample size N using a variational representation of the divergence,
and $KL(Q||P)$ is the ground truth value.

\paragraph{Unbiasedness for mini-batch gradients.}
$\phi$ in Eq. (\ref{Eq_alpha_variation}) can be expressed in a Gibbs density form (Proposition \ref{proposition_alpha_div_resp_in_gibbs_dinsity_form} in Appendix \ref{section_appendix_poofs}).
Then, we observe that
\begin{equation}
    D_{\alpha} (Q||P) = \sup_{T} \left\{
    \frac{1}{\alpha(1-\alpha)}   - \frac{1}{\alpha} E_{Q} \left[e^{\alpha \cdot T} \right] -  \frac{1}{1- \alpha} E_{P} \left[ e^{(\alpha - 1) \cdot T} \right]  \right\}, \label{Eq_alpha_variation_in_gibbs_form}
\end{equation}
where supremum  is considered over all  measurable functions $T:\mathbb{R}^d \rightarrow \mathbb{R}$
with $E_P[e^{(\alpha - 1) \cdot T}] < \infty$ and $E_Q[e^{\alpha \cdot T}] < \infty$.

From this equation, we obtain our loss function, which has unbiasedness for mini-batch gradients (Proposition \ref{prop_gaddient_unbaised} in Appendix \ref{section_appendix_poofs}), as follows :
\begin{equation}
    \mathcal{L}_{\alpha}(\theta) = \frac{1}{\alpha} \hat{E}_Q\left[e^{\alpha \cdot T_{\theta}}\right] +
    \frac{1}{1- \alpha}  \hat{E}_P\left[e^{(\alpha - 1) \cdot T_{\theta}}\right]. \label{Eq_loss_func_alpha_div}
\end{equation}

\paragraph{Advantage in vanishing gradients problem.}
By setting $\alpha$ within $(0, 1)$, we can avoid vanishing gradients of neural networks when they reach the extreme local minima.
The vanishing-gradient problem for optimizing divergence is known in GANs \cite{arjovsky2017towards}.
Now, we consider the case where the probability ratio $e^{T_{\theta}(\mathbf{x})}$ in Eq. (\ref{Eq_loss_func_alpha_div}) is nearly zero or large for some point $\mathbf{x}$,
corresponding to cases in which the probabilities for $P$ or $Q$ at some points are much smaller than those for the other.

To show the relation between $e^{T_{\theta}(\mathbf{x})}$ and the learning of the neural networks,
we obtain gradient of Eq. (\ref{Eq_loss_func_alpha_div}):
\begin{equation}
    \nabla_{\theta}	\mathcal{L}_{\alpha} (\theta) = \hat{E}_{Q} \Big[ \nabla_{\theta} T_{\theta} \cdot e^{\alpha \cdot T_{\theta}} \Big]
    - \hat{E}_{P} \Big[ \nabla_{\theta} T_{\theta} \cdot e^{(\alpha - 1) \cdot T_{\theta}} \Big]. \label{Eq_delta_loss_alpha}
\end{equation}

The behavior of $\nabla_{\theta}\mathcal{L}_{\alpha}(\theta)$ when $E_{Q}[e^{T_{\theta}}] \rightarrow 0$
or $E_{Q}[e^{T_{\theta}}] \allowbreak \rightarrow \allowbreak \infty$, under some regular conditions for $T_{\theta}$ and an assumption that $P \ll Q$,
can be summarized as follows: Let $E[\ \cdot \ ]$ denote $E_{P}[E_{Q}[\ \cdot \ ]]$, then

$\alpha > 1$:
\hspace{1ex} $E[\nabla_{\theta} \mathcal{L}_{\alpha}(\theta)] \rightarrow  \vec{0}$ (as  \ $E_{Q}[e^{T_{\theta}}] \rightarrow 0$), \ \  and \ \
$E[\nabla_{\theta} \mathcal{L}_{\alpha}(\theta)] \rightarrow \vec{\infty} - \vec{\infty} $ (as $E_{Q}[e^{T_{\theta}}] \rightarrow \infty$).

$\alpha < 0$:
\hspace{1ex} $E[\nabla_{\theta} \mathcal{L}_{\alpha}(\theta)] \rightarrow \vec{0}$ (as \ $E_{Q}[e^{T_{\theta}}] \rightarrow \infty$), \ and \
$E[\nabla_{\theta} \mathcal{L}_{\alpha}(\theta)] \rightarrow \vec{\infty} - \vec{\infty} $ (as  $E_{Q}[e^{T_{\theta}}] \rightarrow 0$).

$0 < \alpha < 1$:
$E[\nabla_{\theta} \mathcal{L}_{\alpha}(\theta)] \rightarrow - \vec{\infty}$ (as $E_{Q}[e^{T_{\theta}}] \rightarrow 0$), \ and \
$E[\nabla_{\theta} \mathcal{L}_{\alpha}(\theta)] \rightarrow \vec{\infty}$ (as $E_{Q}[e^{T_{\theta}}]\rightarrow \infty$).

Notably, $E_{Q}[e^{T_{\theta}}] \rightarrow 0 \Leftrightarrow E_{P}[e^{T_{\theta}}] \rightarrow 0$ and
$E_{Q}[e^{T_{\theta}}] \rightarrow \infty \Leftrightarrow E_{P}[e^{T_{\theta}}] \rightarrow  \infty$, because $Q \ll P$ and $P \ll Q$.

For $\alpha > 1$ and $\alpha < 0$, cases exist where $E[\nabla_{\theta} \mathcal{L}_{\alpha}(\theta)] \allowbreak \rightarrow \vec{0}$.
This implies the possibility that the neural networks reach extreme local minima such that their estimations for density ratios are $0$ or $\infty$.
However, this problem can be avoided by selecting $\alpha$ from interval $(0, 1)$.
We note that the selecting of $\alpha$ does not cause instability in numerical calculations for cases where $E[\nabla_{\theta} \mathcal{L}_{\alpha}(\theta)] \rightarrow \vec{\infty}  - \vec{\infty} $.
In Appendix \ref{subsection_neumerical_experiments_convegence_various_alphas}, we present numerical experimental results for different values of $\alpha$.

\section{Method} \label{Section_Method}
In this section, we first present the main theorem that summarizes the new balancing weight method proposed herein.
Next, we present the balancing weight method.

\subsection{Main Theorem}

Here, we present the main theorem that summarizes the new balancing weight method proposed herein.
\begin{theorem}\label{Theorem_main_result}
    Given disjoint sets of
    $\mathbf{X} = \{\mathbf{X}_1, \allowbreak \mathbf{X}_2,\allowbreak \dots, \allowbreak\mathbf{X}_n \}, \allowbreak
    \mathbf{Y},\allowbreak \mathbf{Z} \subset\allowbreak  \mathbf{V}$  satisfying
    \begin{equation}
        \mathbf{X}=\{\mathbf{X}_1,\mathbf{X}_2,\dots, \mathbf{X}_n\} \subset An(\mathbf{Y})_G \hspace{3ex} \text{and} \hspace{3ex} \mathbf{Z} \cap De(\mathbf{Y})_G=\phi. \label{Assumpion_Variables}
    \end{equation}
    Let  $\mathbb{P} = P(\mathbf{X}_1,\mathbf{X}_2,\dots, \mathbf{X}_n, \mathbf{Z})$ and
    $\mathbb{Q} = P(\mathbf{X_1}) \times P(\mathbf{X_2}) \times \cdots \times P(\mathbf{X_n}) \times P(\mathbf{Z})$,
    and $\widetilde{P} =
    P(\mathbf{Y}|do(\mathbf{X}), \mathbf{Z}) \times P(\mathbf{X_1}) \times P(\mathbf{X_2}) \times \cdots
    \times P(\mathbf{X_n}) \times P(\mathbf{Z})$.
    We assume that $P$ satisfies Assumptions 1 and 2 in the aforementioned setting,
    and it holds that $E_{\mathbb{P}}\left[ \left( d\mathbb{Q}/d\mathbb{P}\right)^{1 - \alpha} \right] < \infty$
    for some $0 < \alpha < 1$,
    then, for the optimal function $T^*$, such that
    \begin{eqnarray}
        T^*(\mathbf{X}_1,\mathbf{X}_2,\dots, \mathbf{X}_n, \mathbf{Z}) =
        \arg \inf_{T \in \mathcal{T}^{\alpha}} \left\{ \frac{1}{\alpha} E_{\mathbb{Q}}\left[e^{\alpha \cdot T}\right]
        + \frac{1}{1 - \alpha } E_{\mathbb{P}}\left[e^{(\alpha - 1) \cdot T }\right] \right\} ,  \label{Eq_theorem_optimize}
    \end{eqnarray}
    it holds that
    \begin{equation}
        \frac{d \widetilde{P}}{dP} = e^{-T^*(\mathbf{X}_1,\mathbf{X}_2,\dots, \mathbf{X}_n, \mathbf{Z})}.
        \label{Eq_theorem_gibbs_density_estimate}
    \end{equation}
    Here, $\mathcal{T}^{\alpha}$ denotes the set of all non-constant functions $T(\mathbf{x}):\mathbb{R}^d \rightarrow \mathbb{R}$
    with $E_{\mathbb{P}}[e^{(\alpha -1)\cdot T(\mathbf{X})}] < \infty$.
\end{theorem}
\begin{proof}
    See Appendix \ref{section_appendix_poofs}.
\end{proof}
Here, we mention that the assumption (\ref{Assumpion_Variables}) is necessary
for the Eq. (\ref{Eq_theorem_gibbs_density_estimate}) to hold, which is derived from our Theorem \ref{theorem_intervention_eq} in Appendix \ref{section_appendix_poofs}.

\subsection{Balancing Weight Method}
We present the implementation of training a neural balancing weights (NBW) model in Algorithm \ref{algo_train}.
It is important to consider the stopping time $K$ for neural network model $T_{\theta_K}$ in Algorithm \ref{algo_train}, which is discussed in the next section.
To obtain the sample mean under $\mathbb{Q}$, that is, the estimator for $ E_{\mathbb{Q}}\left[e^{\alpha \cdot T_{\theta}}\right]$ in Eq. (\ref{Eq_theorem_optimize}), a shuffling operation can be used for the samples.
Now, we define neural balancing weights (NBW).
\footnote{We distinguish the notation of $\mathit{BW}(\cdot)$ by the expression of the variables in the parentheses.
    For example, for three variables $X_1, X_2, X_3 \subset \mathbf{V}$,
    let $\mathbf{X} = \{X_1, X_2 \}$. Then, $\mathit{BW}(\mathbf{X}, X_3;T_{\theta})$ is used to indicate
    the balancing weights for $dP(X_1, X_2) \times dP(X_3) / dP(X_1, X_2, X_3)$.
    Conversely, $\mathit{BW}(X_1, X_2, X_3;T_{\theta})$ denotes the balancing weights for
    $dP(X_1) \times dP(X_2) \times dP(X_3) / dP(X_1, X_2, X_3)$.}
\footnote{However, we drop the variables in the parentheses and write
    $\mathit{BW}(\mathbf{X}_1,\mathbf{X}_2,\dots, \mathbf{X}_n, \mathbf{Z};T_{\theta})$ as $\mathit{BW}_{\theta}$
    if not necessary in the context.}
\begin{definition}[Neural Balancing Weights]
    Let $T_{\theta_K}$ be a neural networks obtained from Algorithm \ref{algo_train}.
    Then, the NBW of $T_{\theta_K}$, expressed as $\mathit{BW}(\mathbf{X}_1,\mathbf{X}_2,\dots, \mathbf{X}_n, \mathbf{Z};T_{\theta_K})$, are defined as
    \begin{equation}
        \mathit{BW}(\mathbf{X}_1,\mathbf{X}_2,\dots, \mathbf{X}_n, \mathbf{Z};T_{\theta_K}) = \frac{1}{Z}e^{-T_{\theta_K}(\mathbf{X}_1,\mathbf{X}_2,\dots, \mathbf{X}_n, \mathbf{Z})}, \label{Eq_def_neural_gibbs_density_balancing_weights}
    \end{equation}
    where $Z=\hat{E}_{\mathbb{P}}\left[e^{-T_{\theta_K}(\mathbf{X}_1,\mathbf{X}_2,\dots, \mathbf{X}_n, \mathbf{Z})} \right]$.
\end{definition}

We estimate $E_{\widetilde{P}}[\mathbf{Y}|\mathbf{X},\mathbf{Z}]$, that is
the CACE for $P(\mathbf{Y}, \overline{do}(\mathbf{X_1}),\allowbreak \overline{do}(\mathbf{X_2}),\allowbreak
\ldots, \allowbreak \overline{do}(\mathbf{X_n}), \allowbreak \mathbf{Z})$,
using $\mathit{BW}(\mathbf{X}_1,\allowbreak \mathbf{X}_2, \allowbreak \dots,
\allowbreak \mathbf{X}_n,\allowbreak \mathbf{Z};\allowbreak T_{\theta_K})$
as the sample weights of the supervised algorithm:
\begin{equation}
    \widehat{E}_{\widetilde{P}}\left[\mathbf{Y}|\mathbf{X},\mathbf{Z}\right] =
    \widehat{E}_{P}\left[\mathbf{Y} \cdot \mathit{BW}_{\theta_K} \ |\ \mathbf{X},\mathbf{Z}\right].
\end{equation}
Here, $\widehat{E}_{P}$ corresponds to the model of a supervised machine learning algorithm.
As an example, we demonstrate a back-propagation algorithm using balancing weights for the
mean squared error (MSE) loss in Algorithm \ref{algo_cbp} in Appendix \ref{Subsection_backporp_algo}.

\begin{algorithm}[tb]
    \caption{Training a Neural Balancing Weight model}\label{algo_train}
    \begin{multicols}{2}
        \begin{algorithmic}
            \Require  Train Data $(\mathbf{x}_1, \{(\mathbf{x}_1^i, \ldots,  \mathbf{x}_n^i, \allowbreak \mathbf{z}^i)\}_{i=1}^N$
            \Ensure  A Neural Balancing Weight Model $T_{\theta_K}$
            \State   $\sigma_{1}^{\mathbf{x}} \leftarrow \texttt{SHUFFLE}(\{1:N\})$
            \State   \qquad \qquad $\vdots$
            \State   $\sigma_{n}^{\mathbf{x}} \leftarrow \texttt{SHUFFLE}(\{1:N\})$
            \State   $\sigma^{\mathbf{z}} \leftarrow \texttt{SHUFFLE}(\{1:N\})$
        \end{algorithmic}
        \columnbreak
        \begin{algorithmic}
            \For{$t=1$ to $K$}
            \State
            $\hat{E}_{\mathbb{P}} \leftarrow \frac{1}{N} \Sigma_{i=1}^N e^{(\alpha - 1) \cdot T_{\theta_t}(\mathbf{x}_1^i, \ldots, \mathbf{x}_n^i, \mathbf{z}^i)}$ \label{algo_line_Ep_sum_estimate_upper}
            \State
            $\hat{E}_{\mathbb{Q}} \leftarrow \frac{1}{N}
            \Sigma_{i=1}^N e^{\alpha \cdot T_{\theta_t}(\mathbf{x}_1^{\sigma_{1}^{\mathbf{x}}(i)},
                \ldots,
                \mathbf{x}_n^{\sigma_{n}^{\mathbf{x}}(i)},
                \mathbf{z}^{\sigma^{\mathbf{z}}(i)})}$ 					  \label{algo_line_Ep_sum_estimate_lower}
            \State $\mathcal{L}_{\alpha}({\theta_t}) \leftarrow  \hat{E}_{\mathbb{Q}}/\alpha + \hat{E}_{\mathbb{P}}/(1 -\alpha)$ 
            \State $\theta_{t+1} \leftarrow  \theta_t  -  \nabla_{\theta_t} \mathcal{L}_{\alpha}({\theta_t})$ 
            \EndFor
        \end{algorithmic}
    \end{multicols}
\end{algorithm}

\section{Techniques for NBW} \label{Section_techniques_estimating_NGD_balancing_weights}
We propose two techniques for estimating balancing weights:
($\mathrm{i}$) improves generalization performance of the balancing weights.
($\mathrm{ii}$) measures the performance of the balancing weights by estimating the $\alpha$-divergence information.

\subsection{Improving the Generalization Performance of the Balancing Weights} \label{Subsetion_Improving_The_Generalization_Performance}
In this section, we first present an overfitting problem for balancing distributions.
We then present two methods for improving the generalization performance of the weights: a validation method using test data and an early stopping method.
Herein, let $T_{\theta_t}$ denote an NBW model at step $t$ in Algorithm \ref{algo_train}. Let $\hat{\mathbf{X}}_{{Q}}^{(N)}(t) = e^{-T_{\theta_t}} \cdot \mathbf{X}_{{P}}^{(N)}$,
that is, the data balanced by the weights of $e^{-T_{\theta_t}}$.
Subsequently, let  $\hat{{Q}}^{(N)}_t$ and $\hat{{Q}}^{(N)}$ denote the probability distributions of $\hat{\mathbf{X}}_{{Q}}^{(N)}(t)$ and $\hat{\mathbf{X}}_{{Q}}^{(N)}$, respectively, which correspond to the estimated and true distributions for balancing.

\paragraph{An overfitting problem for balancing distributions.}
From Corollary \ref{corollary_desitination_alpha_div_est} in Appendix \ref{section_appendix_poofs}, we observe $\hat{\mathbf{X}}_{{Q}}^{(N)}(t) \xrightarrow{\ \  \text{d} \ \ } \mathbf{X}_{{Q}}^{(N)}$ as $t \rightarrow \infty$. Then, Theorem 1 in \cite{weed2019sharp} shows that
\begin{equation}
    \lim_{t \rightarrow \infty}  W_1({Q} ,\hat{{Q}}^{(N)}_t) =  W_1({Q}, \hat{{Q}}^{(N)}) \gtrsim N^{-1/(d-\delta)} \quad (\, \forall \delta > 0 \,), \label{Eq_the_final_destination_balancing_has_COD}
\end{equation}
where $W_1$ is the Wasserstein distance of order $1$ and $d$ is the lower Wasserstein dimension defined in \cite{weed2019sharp}.
Eq. (\ref{Eq_the_final_destination_balancing_has_COD}) implies that, for balancing finite data, the destination of the balanced distribution is an empirical distribution, and the generalization performance of balancing worsens exponentially when the dimension of the data is larger. In view of optimizations of GANs, \cite{yang2022generalization} referred to this phenomenon the “momorization” and proposed an early stopping method.

\paragraph{Validation method using test data.}
We can use a validation method using test data.
Because $\hat{{Q}}^{(N)}$ and $\hat{{P}}^{(N)}$ are empirical probability distributions, 
we observe that $d\hat{{Q}}^{(N)}/d\hat{{P}}^{(N)}(x) = \allowbreak dQ/dP(x)$ if $x \in \mathbf{X}^{(N)}$, otherwise $d\hat{{Q}}^{(N)}/d\hat{{P}}^{(N)}(x) = 0$
(Proposition \ref{proposition_representation_of_density_ratios_between_emprical_distributions} in Appendix \ref{section_appendix_poofs}).
Then, the optimal function of Eq. (\ref{Eq_loss_func_alpha_div}) for both distributions, that is $T_*^{(N)} = \allowbreak - log(d\hat{{Q}}^{(N)}/d\hat{{P}}^{(N)})$, is infinite except for the observations, and the loss of the $T_*^{(N)}$ is  infinite for data independent of the observations.
This implies that the loss of $T_t^{(N)}$ for the test data turns to increase from the middle of the training period, and we can determine the training step at which the generalization performance of the weights begins to worsen. In Section \ref{subsection_neumerical_experiments_convegence_dims_data} in Appendix \ref{section_neumerical_experiments}, we provide numerical experimental results to confirm the relationship between dimensions of data ($d$) and steps in training ($K$).

\paragraph{Early stopping method.}
In addition, we present an early stopping method for estimating the balancing weights as follows, which is inspired by the method developed in \cite{yang2022generalization} (Corollary \ref{corollary_early_stopping_1} in Appendix \ref{section_appendix_poofs}): for some $\delta > 0$, let
\begin{equation}
    K_0 = C\cdot N^{2/(d+\delta)}, \label{Eq_the_early_stop_def}
\end{equation}
where $C > 0$ is constant. Then, we have $W_1({Q}, \hat{{Q}}^{(N)}_{K_0}) \lesssim N^{-1/(d+\delta)}$.
Unfortunately, the curse of dimensionality remains in the proposed method. This will be discussed in the next section.

\subsection{Measuring the Performance of the Balancing Weights}
Let us assume that we obtain an NBW model $T_{\theta_0}$ and let $\mathit{BW}_{\theta_0} = \mathit{BW}(\mathbf{X}_1,\mathbf{X}_2,\dots, \mathbf{X}_n, \mathbf{Z};T_{\theta_0})$
be the balancing weights of $T_{\theta_0}$.
If $\mathit{BW}_{\theta_0}$ successfully estimates $\frac{d{Q}}{d{P}}$, then
the $\alpha$-divergence between ${Q}$ and ${P}_0$ will be nearly zero.
Conversely, if $\mathit{BW}_{\theta_0}$ fails to estimate $\frac{d{Q}}{d{P}}$,
the $\alpha$-divergence between ${Q}$ and ${P}_0$ is significantly different from zero.
This implies that we can measure the performance of the balancing weights using the $\alpha$-divergence information for ${P}_0$.

Next, we present the definition of an $\alpha$-divergence information estimator using neural networks.
\begin{definition}[Neural $\alpha$-divergence Information Estimator]\label{Def_mutula_information_estimotor_via_NGD}
    For disjoint sets of variables $\mathbf{X}_1,  \allowbreak \mathbf{X}_2, \allowbreak \ldots,
    \allowbreak \mathbf{X}_n \allowbreak \subset \allowbreak \mathbf{V}$,
    the neural $\alpha$-divergence information estimator for ${P}$ 
    is defined as
    \begin{eqnarray}
        \widehat{I}_{\alpha} (\mathbf{X}_1, \mathbf{X}_2, \ldots, \mathbf{X}_n;T_{\theta_*}) = \frac{1}{\alpha(1-\alpha)}  - \inf_{
            \theta \in \Theta} \left\{  \frac{1}{\alpha}\hat{E}_{{Q}} \left[e^{\alpha \cdot T_{\theta}}\right]
        + \frac{1}{1 - \alpha} \hat{E}_{{P}} \left[e^{(\alpha - 1) \cdot T_{\theta}}\right] \right\}.   \label{def_MIestimate_Eq_MIestimator}
    \end{eqnarray}
\end{definition}

To measure the performance of balancing the weights from the NBW model,
we estimate the $\alpha$-divergence information for  balanced distribution from the weights.
That is, we use the sample mean under a balanced distribution,
despite the sample mean under ${P}$ for Eq. (\ref{def_MIestimate_Eq_MIestimator}).
For example, we assume that we have certain weights $\mathit{BW'} = \{bw^i:i=1,2,...,N\}$,
where $bw^i$ denotes the weight of sample $i$ of $N$.
The balanced distribution from the weights is
\begin{equation}
    d {P}' = \mathit{BW}' \cdot d{P}.
\end{equation}
The $\alpha$-divergence information for ${P}'$ is estimated by replacing ${P}$ with ${P}'$ for
Eq. (\ref{def_MIestimate_Eq_MIestimator}) in the following manner: despite the sample mean $\hat{E}_{{P}}[e^{(\alpha-1) \cdot T_{\theta}}]$ for these equations,
we use the weighted sample mean, such that
\begin{equation}
    \hat{E}_{{P}'}[e^{(\alpha-1) \cdot T_{\theta}}] = \frac{1}{N} \Sigma_{i=1}^N bw^i \cdot e^{(\alpha - 1) \cdot T_{\theta}(\mathbf{x}_1^i, \mathbf{x}_2^i, \ldots, \mathbf{x}_n^i, \mathbf{z}^i)}. \label{Eq_sample_mean_p0}
\end{equation}
Details on the implementation for measuring the performance of balancing weights from an NBW model are provided in Algorithm \ref{algo_measure_paformance},
which includes the validation method for the overfitting problem in Section \ref{Subsetion_Improving_The_Generalization_Performance}.

\begin{algorithm}[tb]
    \begin{multicols}{2}
        \caption{Algorithm for checking the balance}\label{algo_measure_paformance}
        \begin{algorithmic}
            \Require Train Data $\{(\mathbf{x}_1^i,  \ldots, \mathbf{x}_n^i, \mathbf{z}^i)\}_{i=1}^N$,
            Test Data $\{(\widetilde{\mathbf{x}}_1^i,  \ldots, \widetilde{\mathbf{x}}_n^i, \widetilde{\mathbf{z}}^i)\}_{i=1}^N$,
            A Neural Balancing Weight Model $T_{\theta}$
            \Ensure The estimated $\alpha$-divergence information $\widehat{I}_{\alpha}$ for the balanced distribution with the balancing weights from $T_{\theta}$
            \State
            \State   $\sigma_{1}^{\mathbf{x}} \leftarrow \texttt{SHUFFLE}(\{1:N\})$
            \State   \qquad \qquad \qquad $\vdots$
            \State   $\sigma_{n}^{\mathbf{x}} \leftarrow \texttt{SHUFFLE}(\{1:N\})$
            \State   $\sigma^{\mathbf{z}} \leftarrow \texttt{SHUFFLE}(\{1:N\})$
            \State $\{bw^i\}_i^N \leftarrow
            \frac{ e^{-T_{\theta}(\mathbf{x}_1^i, \mathbf{x}_2^i, \ldots, \mathbf{x}_n^i, \mathbf{z}^i)}}
            { \sum \{e^{-T_{\theta}(\mathbf{x}_1^i, \mathbf{x}_2^i, \ldots, \mathbf{x}_n^i, \mathbf{z}^i)}\} }$
            \State  $\widehat{\mathbf{I}}_{\alpha} \leftarrow \{\}$
            \columnbreak
            \For{$t=1$ to $K$}
            \State         $\hat{E}_{\mathbb{P}_0} \leftarrow \frac{1}{N}
            \Sigma_{i=1}^N e^{(\alpha - 1) \cdot T_{\psi}(\mathbf{x}_1^i, \ldots, \mathbf{x}_n^i, \mathbf{z}^i)}\cdot bw^i$
            \State       $\hat{E}_{\mathbb{Q}} \leftarrow \frac{1}{N}
            \Sigma_{i=1}^N e^{\alpha \cdot T_{\psi}(\mathbf{x}_1^{\sigma_{1}^{\mathbf{x}}(i)},
                \ldots, \mathbf{x}_n^{\sigma_{n}^{\mathbf{x}}(i)}, \mathbf{z}^{\sigma^{\mathbf{z}}(i)})}$ \label{algo_line_Ep_sum_estimate_lower}
            \State $\mathcal{L}_{\alpha}({\psi}) \leftarrow   \hat{E}_{\mathbb{Q}}/\alpha +  \hat{E}_{\mathbb{P}_0}/(1 -\alpha)$
            \State $\psi \leftarrow  \psi  -  \nabla_{\psi} \mathcal{L}_{\alpha}({\psi})$
            \State $\hat{E}_{\mathbb{P}_0}^{te}  \leftarrow \frac{1}{N}
            \Sigma_{i=1}^N e^{(\alpha - 1) \cdot T_{\psi}(\widetilde{\mathbf{x}}_1^i, \ldots,
                \widetilde{\mathbf{x}}_n^i, \widetilde{\mathbf{z}}^i)}\cdot bw^i$
            \State $\hat{E}_{\mathbb{Q}}^{te}  \leftarrow \frac{1}{N}
            \Sigma_{i=1}^N e^{\alpha \cdot T_{\psi}(\widetilde{\mathbf{x}}_1^{\sigma_{1}^{\mathbf{x}}(i)},
                \ldots,
                \widetilde{\mathbf{x}}_n^{\sigma_{n}^{\mathbf{x}}(i)},
                \widetilde{\mathbf{z}}^{\sigma^{\mathbf{z}}(i)})}$
            \State $\widehat{I}_{\alpha}^t \leftarrow 1/\{\alpha\cdot(1-\alpha)\}$ \\
            $\qquad \qquad \qquad - \ \hat{E}_{\mathbb{Q}}^{te} /\alpha \,  - \,  \hat{E}_{\mathbb{P}_0}^{te} /(1 -\alpha)$
            \State $\widehat{\mathbf{I}}_{\alpha}  \leftarrow \widehat{\mathbf{I}}_{\alpha}  \cup \{ \widehat{I}_{\alpha}^t \}$
            \EndFor
            \State $\widehat{I}_{\alpha}  \leftarrow \max_t \widehat{\mathbf{I}}_{\alpha}$
        \end{algorithmic}
    \end{multicols}
\end{algorithm}

\section{Limitations: Sample Size Requirements.} \label{Section_Sample_Size_Requirements}
In Section \ref{Subsetion_Improving_The_Generalization_Performance}, we noted that our method has a curse of dimensionality.
The sample size requirement of the proposed method is $N > \left( \frac{1}{\varepsilon}\right)^{d+\delta}$ for $W_1({Q} ,\hat{{Q}}^{(N)}_{K_0}) < \varepsilon$  (Corollary \ref{corollary_early_stopping_2} in Appendix \ref{section_appendix_poofs}).
However, the curse of dimensionality is an essential problem when balancing multivariate data owing to the following factors.
Because the optimal balancing weights defined as Eq. (\ref{our_objective_weights}) for (finite) observational data are the density ratios of the empirical distributions, the distribution of the data balanced by them is the empirical distribution.
Subsequently, owing to the balancing of the weights, the curse of dimensionality of the empirical distribution occurs, which is the same as that described in Section \ref{Subsetion_Improving_The_Generalization_Performance}.
Therefore, to achieve high generalization performance, we need to obtain weights that differ from the ideal density ratio between the source and target of the empirical distribution.
Further research is required to address this problem.
In Appendix \ref{subsection_neumerical_experiments_causal_effect}, we present the numerical examination results in which the causal effects of joint and multidimensional interventions were estimated with different sample sizes.

\section{Conclusion}
We propose generalized balancing weights to estimate the causal effects of an arbitrary mixture of discrete and continuous interventions.
Three methods for training the weights were provided: an optimization method to learn the weights, a method to improve the generalization performance of the balancing weights, and a method to measure the performance of the weights.
We showed the sample size requirements for the weights and then discussed the curse of dimensionality that occurs as a general problem when balancing multidimensional data.
Although the curse of dimensionality remains in our method, we believe that this study provides a basic approach for estimating the balancing weights of multidimensional data using variational $f$-divergence.

\bibliography{NeurIPS2023_reference_papers}

\begin{thebibliography}{37}
\providecommand{\natexlab}[1]{#1}
\providecommand{\url}[1]{\texttt{#1}}
\expandafter\ifx\csname urlstyle\endcsname\relax
  \providecommand{\doi}[1]{doi: #1}\else
  \providecommand{\doi}{doi: \begingroup \urlstyle{rm}\Url}\fi

\bibitem[Acuna et~al.(2021)Acuna, Zhang, Law, and Fidler]{acuna2021f}
David Acuna, Guojun Zhang, Marc~T Law, and Sanja Fidler.
\newblock f-domain adversarial learning: Theory and algorithms.
\newblock In \emph{International Conference on Machine Learning}, pages 66--75.
  PMLR, 2021.

\bibitem[Arjovsky and Bottou(2017)]{arjovsky2017towards}
Martin Arjovsky and L{\'e}on Bottou.
\newblock Towards principled methods for training generative adversarial
  networks.
\newblock \emph{arXiv preprint arXiv:1701.04862}, 2017.

\bibitem[Belghazi et~al.(2018)Belghazi, Baratin, Rajeshwar, Ozair, Bengio,
  Courville, and Hjelm]{belghazi2018mutual}
Mohamed~Ishmael Belghazi, Aristide Baratin, Sai Rajeshwar, Sherjil Ozair,
  Yoshua Bengio, Aaron Courville, and Devon Hjelm.
\newblock Mutual information neural estimation.
\newblock In \emph{International conference on machine learning}, pages
  531--540. PMLR, 2018.

\bibitem[Bellemare et~al.(2017)Bellemare, Danihelka, Dabney, Mohamed,
  Lakshminarayanan, Hoyer, and Munos]{bellemare2017cramer}
Marc~G Bellemare, Ivo Danihelka, Will Dabney, Shakir Mohamed, Balaji
  Lakshminarayanan, Stephan Hoyer, and R{\'e}mi Munos.
\newblock The cramer distance as a solution to biased wasserstein gradients.
\newblock \emph{arXiv preprint arXiv:1705.10743}, 2017.

\bibitem[Birrell et~al.(2022)Birrell, Katsoulakis, and
  Pantazis]{birrell2022optimizing}
Jeremiah Birrell, Markos~A Katsoulakis, and Yannis Pantazis.
\newblock Optimizing variational representations of divergences and
  accelerating their statistical estimation.
\newblock \emph{IEEE Transactions on Information Theory}, 2022.

\bibitem[Ganin et~al.(2016)Ganin, Ustinova, Ajakan, Germain, Larochelle,
  Laviolette, Marchand, and Lempitsky]{ganin2016domain}
Yaroslav Ganin, Evgeniya Ustinova, Hana Ajakan, Pascal Germain, Hugo
  Larochelle, Fran{\c{c}}ois Laviolette, Mario Marchand, and Victor Lempitsky.
\newblock Domain-adversarial training of neural networks.
\newblock \emph{The journal of machine learning research}, 17\penalty0
  (1):\penalty0 2096--2030, 2016.

\bibitem[Gilardoni(2010)]{gilardoni2010pinsker}
Gustavo~L Gilardoni.
\newblock On pinsker's and vajda's type inequalities for csisz{\'a}r's $ f
  $-divergences.
\newblock \emph{IEEE Transactions on Information Theory}, 56\penalty0
  (11):\penalty0 5377--5386, 2010.

\bibitem[Hainmueller(2012)]{hainmueller2012entropy}
Jens Hainmueller.
\newblock Entropy balancing for causal effects: A multivariate reweighting
  method to produce balanced samples in observational studies.
\newblock \emph{Political analysis}, 20\penalty0 (1):\penalty0 25--46, 2012.

\bibitem[Imai and Ratkovic(2014)]{imai2014covariate}
Kosuke Imai and Marc Ratkovic.
\newblock Covariate balancing propensity score.
\newblock \emph{Journal of the Royal Statistical Society: Series B (Statistical
  Methodology)}, 76\penalty0 (1):\penalty0 243--263, 2014.

\bibitem[Imai and Van~Dyk(2004)]{imai2004causal}
Kosuke Imai and David~A Van~Dyk.
\newblock Causal inference with general treatment regimes: Generalizing the
  propensity score.
\newblock \emph{Journal of the American Statistical Association}, 99\penalty0
  (467):\penalty0 854--866, 2004.

\bibitem[Kato and Teshima(2021)]{kato2021non}
Masahiro Kato and Takeshi Teshima.
\newblock Non-negative bregman divergence minimization for deep direct density
  ratio estimation.
\newblock In \emph{International Conference on Machine Learning}, pages
  5320--5333. PMLR, 2021.

\bibitem[Lacoste-Julien et~al.(2012)Lacoste-Julien, Schmidt, and
  Bach]{lacoste2012simpler}
Simon Lacoste-Julien, Mark Schmidt, and Francis Bach.
\newblock A simpler approach to obtaining an o (1/t) convergence rate for the
  projected stochastic subgradient method.
\newblock \emph{arXiv preprint arXiv:1212.2002}, 2012.

\bibitem[Lee et~al.(2022)Lee, Ma, and de~Luna]{lee2022covariate}
Seong-ho Lee, Yanyuan Ma, and Xavier de~Luna.
\newblock Covariate balancing for causal inference on categorical and
  continuous treatments.
\newblock \emph{Econometrics and Statistics}, 2022.

\bibitem[Li et~al.(2018)Li, Morgan, and Zaslavsky]{li2018balancing}
Fan Li, Kari~Lock Morgan, and Alan~M Zaslavsky.
\newblock Balancing covariates via propensity score weighting.
\newblock \emph{Journal of the American Statistical Association}, 113\penalty0
  (521):\penalty0 390--400, 2018.

\bibitem[McAllester and Stratos(2020)]{mcallester2020formal}
David McAllester and Karl Stratos.
\newblock Formal limitations on the measurement of mutual information.
\newblock In \emph{International Conference on Artificial Intelligence and
  Statistics}, pages 875--884. PMLR, 2020.

\bibitem[Nguyen et~al.(2007)Nguyen, Wainwright, and
  Jordan]{nguyen2007estimating}
XuanLong Nguyen, Martin~J Wainwright, and Michael Jordan.
\newblock Estimating divergence functionals and the likelihood ratio by
  penalized convex risk minimization.
\newblock \emph{Advances in neural information processing systems}, 20, 2007.

\bibitem[Nguyen et~al.(2010)Nguyen, Wainwright, and
  Jordan]{nguyen2010estimating}
XuanLong Nguyen, Martin~J Wainwright, and Michael~I Jordan.
\newblock Estimating divergence functionals and the likelihood ratio by convex
  risk minimization.
\newblock \emph{IEEE Transactions on Information Theory}, 56\penalty0
  (11):\penalty0 5847--5861, 2010.

\bibitem[Nowozin et~al.(2016)Nowozin, Cseke, and Tomioka]{nowozin2016f}
Sebastian Nowozin, Botond Cseke, and Ryota Tomioka.
\newblock f-gan: Training generative neural samplers using variational
  divergence minimization.
\newblock \emph{Advances in neural information processing systems}, 29, 2016.

\bibitem[Oord et~al.(2018)Oord, Li, and Vinyals]{oord2018representation}
Aaron van~den Oord, Yazhe Li, and Oriol Vinyals.
\newblock Representation learning with contrastive predictive coding.
\newblock \emph{arXiv preprint arXiv:1807.03748}, 2018.

\bibitem[Pearl(1995)]{pearl1995causal}
Judea Pearl.
\newblock Causal diagrams for empirical research.
\newblock \emph{Biometrika}, 82\penalty0 (4):\penalty0 669--688, 1995.

\bibitem[Pearl(2009)]{pearl2009causality}
Judea Pearl.
\newblock Causality: Models, reasoning and inference, 2009.

\bibitem[Poole et~al.(2019)Poole, Ozair, Van Den~Oord, Alemi, and
  Tucker]{poole2019variational}
Ben Poole, Sherjil Ozair, Aaron Van Den~Oord, Alex Alemi, and George Tucker.
\newblock On variational bounds of mutual information.
\newblock In \emph{International Conference on Machine Learning}, pages
  5171--5180. PMLR, 2019.

\bibitem[Robins et~al.(1994)Robins, Rotnitzky, and Zhao]{robins1994estimation}
James~M Robins, Andrea Rotnitzky, and Lue~Ping Zhao.
\newblock Estimation of regression coefficients when some regressors are not
  always observed.
\newblock \emph{Journal of the American statistical Association}, 89\penalty0
  (427):\penalty0 846--866, 1994.

\bibitem[Robins et~al.(2000)Robins, Hernan, and Brumback]{robins2000marginal}
James~M Robins, Miguel~Angel Hernan, and Babette Brumback.
\newblock Marginal structural models and causal inference in epidemiology,
  2000.

\bibitem[Rosenbaum and Rubin(1984)]{rosenbaum1984reducing}
Paul~R Rosenbaum and Donald~B Rubin.
\newblock Reducing bias in observational studies using subclassification on the
  propensity score.
\newblock \emph{Journal of the American statistical Association}, 79\penalty0
  (387):\penalty0 516--524, 1984.

\bibitem[Shiryaev(1995)]{shiryaev1995probability}
Albert~Nikolaevich Shiryaev.
\newblock Probability, 1995.

\bibitem[Shpitser and Pearl(2012)]{shpitser2012identification}
Ilya Shpitser and Judea Pearl.
\newblock Identification of conditional interventional distributions.
\newblock \emph{arXiv preprint arXiv:1206.6876}, 2012.

\bibitem[Si et~al.(2009)Si, Tao, and Geng]{si2009bregman}
Si~Si, Dacheng Tao, and Bo~Geng.
\newblock Bregman divergence-based regularization for transfer subspace
  learning.
\newblock \emph{IEEE Transactions on Knowledge and Data Engineering},
  22\penalty0 (7):\penalty0 929--942, 2009.

\bibitem[Song and Ermon(2019)]{song2019understanding}
Jiaming Song and Stefano Ermon.
\newblock Understanding the limitations of variational mutual information
  estimators.
\newblock \emph{arXiv preprint arXiv:1910.06222}, 2019.

\bibitem[Sugiyama et~al.(2012)Sugiyama, Suzuki, and
  Kanamori]{sugiyama2012density}
Masashi Sugiyama, Taiji Suzuki, and Takafumi Kanamori.
\newblock Density-ratio matching under the bregman divergence: a unified
  framework of density-ratio estimation.
\newblock \emph{Annals of the Institute of Statistical Mathematics},
  64\penalty0 (5):\penalty0 1009--1044, 2012.

\bibitem[T{\"u}bbicke(2022)]{tubbicke2022entropy}
Stefan T{\"u}bbicke.
\newblock Entropy balancing for continuous treatments.
\newblock \emph{Journal of Econometric Methods}, 11\penalty0 (1):\penalty0
  71--89, 2022.

\bibitem[Uehara et~al.(2016)Uehara, Sato, Suzuki, Nakayama, and
  Matsuo]{uehara2016generative}
Masatoshi Uehara, Issei Sato, Masahiro Suzuki, Kotaro Nakayama, and Yutaka
  Matsuo.
\newblock Generative adversarial nets from a density ratio estimation
  perspective.
\newblock \emph{arXiv preprint arXiv:1610.02920}, 2016.

\bibitem[Uppal et~al.(2019)Uppal, Singh, and
  P{\'o}czos]{uppal2019nonparametric}
Ananya Uppal, Shashank Singh, and Barnab{\'a}s P{\'o}czos.
\newblock Nonparametric density estimation \& convergence rates for gans under
  besov ipm losses.
\newblock \emph{Advances in neural information processing systems}, 32, 2019.

\bibitem[Vegetabile et~al.(2021)Vegetabile, Griffin, Coffman, Cefalu, Robbins,
  and McCaffrey]{vegetabile2021nonparametric}
Brian~G Vegetabile, Beth~Ann Griffin, Donna~L Coffman, Matthew Cefalu,
  Michael~W Robbins, and Daniel~F McCaffrey.
\newblock Nonparametric estimation of population average dose-response curves
  using entropy balancing weights for continuous exposures.
\newblock \emph{Health Services and Outcomes Research Methodology}, 21\penalty0
  (1):\penalty0 69--110, 2021.

\bibitem[Weed and Bach(2019)]{weed2019sharp}
Jonathan Weed and Francis Bach.
\newblock Sharp asymptotic and finite-sample rates of convergence of empirical
  measures in wasserstein distance.
\newblock 2019.

\bibitem[Yang and Weinan(2022)]{yang2022generalization}
Hongkang Yang and E~Weinan.
\newblock Generalization and memorization: The bias potential model.
\newblock In \emph{Mathematical and Scientific Machine Learning}, pages
  1013--1043. PMLR, 2022.

\bibitem[Zhang et~al.(2019)Zhang, Liu, Long, and Jordan]{zhang2019bridging}
Yuchen Zhang, Tianle Liu, Mingsheng Long, and Michael Jordan.
\newblock Bridging theory and algorithm for domain adaptation.
\newblock In \emph{International Conference on Machine Learning}, pages
  7404--7413. PMLR, 2019.

\end{thebibliography}

\newpage

\appendix
\section{Organization of the supplementary}
The organization of this supplementary document is as follows:
In Section \ref{section_supplementary_notations},  we lists definitions of the notations used in the proofs;
Section \ref{section_appendix_poofs} presents the theorems and propositions referenced in this study and their proofs;
In Section \ref{section_neumerical_experiments}, the numerical experimental results related to the discussion in this study are provided;
In Section \ref{Subsection_backporp_algo}, the backpropagation algorithm referred to in Section \ref{Section_Method} is presented.
In addition to this material, the code used in the numerical experiments is also submitted as supplementary material.

\section{Notations} \label{section_supplementary_notations}
Table \ref{Table_Notations_are_definitions_supplementary} lists the notations and definitions used in the proofs of Section \ref{section_appendix_poofs}.

\begin{table}[p]
    \caption{Notations and definitions used in the proofs}
    \label{Table_Notations_are_definitions_supplementary}
    \centering
    \begin{tabular}{l||l}
        \toprule            
        Notations & Definitions, Meanings\\
        \midrule              
        $\mathbf{1}(\cdot)$ & A propositional function: $\mathbf{1}(\textit{cond})=1$ if $\textit{cond}$ is true, and \\
        & $\mathbf{1}(\textit{cond})=0$ otherwise. \\
        $id_A$ & The identity function of a set $A$: $id_A(x) = 1$ if $x \in A$, and \\
        & $id_A(x) = 0$ otherwise.\\
        $\|\cdot\|$ & the Euclidean norm.  \\
        $O_{a}(x)$ & A term such that $\lim_{x\rightarrow 0} O(x)/x = C_a < \infty$, where $C_a$ is a scalar value\\
        & determined by $a$.\\
        $O(x)$ & A term such that $\lim_{x\rightarrow 0} O(x)/x  < C$, where $C$ is constant.\\
        $f \lesssim g$& A relationship between two functions $f$ and $g$ such that \\
        & $\lim \sup_{n \rightarrow \infty}f(n)/g(n) < \infty$.\\
        $P \ll Q$ &  $P$ is absolutely continuous with respect to $Q$. \\
        $P$, $Q$ &  A pair of probability measures with $P \ll Q$ and $Q \ll P$. \\
        $\mu$ & A probability measure  with $P \ll \mu$ and $Q \ll \mu$.\\
        $\frac{dP}{dQ}$ & The Radon--Nikod\'{y}m derivative of $P$ with respect to $Q$.\\
        & When $\frac{dQ}{d\mu}(\mathbf{x})=0$, this is defined as $\frac{dP}{dQ}(\mathbf{x})=0$. \\
        $\mathbf{X}$ & A random variable with a probability distribution $\mu$. \\
        $\mathbf{X}_{\sim P}$ & A random variable obtained from $\mathbf{X}$ by changing the probability \\
        & distributions from $\mu$ to $P$:
        $P(\mathbf{X}_{\sim P} \le \mathbf{x}) = \allowbreak \mu(\mathbf{X} \le \mathbf{x})$, $\forall  \mathbf{x} \in \mathbb{R}^d$. \\
        & Intuitively, an observed value of $\mathbf{X}$ in $P$.\\
        $\mathbf{X}_{\sim Q}$ & A random variable obtained from $\mathbf{X}$ by changing the probability \\
        & distributions from $\mu$ to $Q$:
        $Q(\mathbf{X}_{\sim Q} \le \mathbf{x}) = \allowbreak \mu(\mathbf{X} \le \mathbf{x})$, $\forall  \mathbf{x} \in \mathbb{R}^d$. \\
        & Intuitively, an observed value of $\mathbf{X}$ in $Q$.\\
        $\mathbf{X}^{(N)}$& $N$ i.i.d. random variables from $\mu$: $\mathbf{X}^{(N)} = \{\mathbf{X}^{1}, \mathbf{X}^{2}, \ldots, \mathbf{X}^{N}\}$,
        $\mathbf{X}^{i} \overset{\mathrm{iid}}{\sim} \mu$.  \\
        $\mathbf{X}_{P}^{(N)}$ & Random variables obtained from $\mathbf{X}^{(N)}$ by changing the probability \\
        & distributions from $\mu$ to $P$:
        $\mathbf{X}_{P}^{(N)} = \allowbreak \{\mathbf{X}^{1}_{\sim P} , \allowbreak \mathbf{X}^{2}_{\sim P}, \ldots,  \mathbf{X}^{N}_{\sim P}\}$, \\
        &$P(\mathbf{X}_{\sim P}^i \le \mathbf{x}) = \allowbreak  \mu(\mathbf{X^i} \le \mathbf{x})$, $\forall  \mathbf{x} \in \mathbb{R}^d$, ($1 \le i \le N$). \\
        &Intuitively, observed values of $\mathbf{X}^{(N)}$ in $P$.\\
        $\mathbf{X}_{Q}^{(N)}$ & Random variables obtained from $\mathbf{X}^{(N)}$ by changing the probability \\
        & distributions from $\mu$ to $Q$:
        $\mathbf{X}_{Q}^{(N)} = \allowbreak \{\mathbf{X}^{1}_{\sim Q} , \allowbreak \mathbf{X}^{2}_{\sim Q}, \ldots, \mathbf{X}^{N}_{\sim Q}\}$, \\
        &$Q(\mathbf{X}_{\sim Q}^i \le \mathbf{x}) = \allowbreak \mu(\mathbf{X^i} \le \mathbf{x})$, $\forall  \mathbf{x} \in \mathbb{R}^d$, ($1 \le i \le N$). \\
        &Intuitively, observed values of $\mathbf{X}^{(N)}$ in $Q$.\\
        $\hat{P}^{(N)}$ & The (empirical) distributions of $\mathbf{X}_{P}^{(N)}$: $\hat{P}^{(N)}(x) = \frac{1}{N} \sum_{i} \mathbf{1}(\mathbf{X}^{i}_{\sim P}=x)$. \\
        $\hat{Q}^{(N)}$ & The (empirical) distributions of $\mathbf{X}_{Q}^{(N)}$: $\hat{Q}^{(N)}(x) = \frac{1}{N} \sum_{i} \mathbf{1}(\mathbf{X}^{i}_{\sim Q}=x)$. \\
        $\mathcal{T}^{\alpha}$ & The set of functions defined in Theorem \ref{Theorem_main_result}.\\
        $\mathbf{U}=\{U_1,\dots, U_m\}$ & Unobserved random variables.\\
        $\mathbf{V}=\{V_1,\dots,V_n\}$ & Observed random variables.\\
        $\mathcal{X}_\mathbf{A}$ & The domain of variables $\mathbf{A}$. \\
        $G=G_{\mathbf{V}\mathbf{U}}$ & The causal graph for $\mathbf{V}\cup\mathbf{U}$.\\
        $Pa(\mathbf{A})_G$ & All the parents of the observed variables in $G$ for for $\mathbf{A}\subset\mathbf{V}$. \\
        $Ch(\mathbf{A})_G$ & All the children of the observed variables in $G$ for for $\mathbf{A}\subset\mathbf{V}$. \\
        $An(\mathbf{A})_G$ & All the ancestors of the observed variables in $G$ for for $\mathbf{A}\subset\mathbf{V}$. \\
        $De(\mathbf{A})_G$ & All the descendants of the observed variables in $G$ for for $\mathbf{A}\subset\mathbf{V}$. \\
        $W_p$&  The Wasserstein distance of order $p$. \\
        \bottomrule          
    \end{tabular}
\end{table}

\section{Proofs}\label{section_appendix_poofs}

\subsection{Proofs for Section \ref{Section_Estimation_of_Balancing_Weights}}
In this Section, we provide propositions and their proofs, as referred to in Section \ref{Section_Estimation_of_Balancing_Weights}.

\begin{lemma} \label{lemma_alpha_div_resp}
    A variational representation of $\alpha$-divergence is given as
    \begin{eqnarray}
        D_{\alpha} (Q||P) &=& \sup_{\phi \ge 0} \left\{
        \frac{1}{\alpha(1-\alpha)} - \frac{1}{\alpha} E_{Q} \left[\phi^{-\alpha} \right]  - \frac{1}{1 - \alpha} E_{P} \left[\phi^{1 - \alpha} \right]  \right\}, \label{lemma_Eq_alpha_variation}
    \end{eqnarray}
    where supremum  is considered over all  measurable functions with $ E_P[\phi^{1 - \alpha}] < \infty$ and $E_Q[\phi^{-\alpha}] < \infty$.
    The maximum value is achieved at $\phi=dQ/dP$.
\end{lemma}

\begin{proof}[proof of Lemma \ref{lemma_alpha_div_resp}]
    Let $f_{\alpha}(t) = \{t^{1-\alpha} - (1-\alpha)\cdot t - \alpha\}/\{\alpha  (\alpha-1)\}$ for $\alpha \neq 0, 1$, then
    \begin{eqnarray}
        E_{P} \left[f_{\alpha} \left( \frac{dQ}{dP} \right) \right]
        &=& 	E_{P} \left[ \frac{1}{\alpha(\alpha-1)} \left( \frac{dQ}{dP} \right)^{1-\alpha}
        + \frac{1}{\alpha} \left( \frac{dQ}{dP} \right)
        + \frac{1}{1 - \alpha} \right] \nonumber\\
        &=& \frac{1}{\alpha(\alpha-1)} E_{P} \left[  \left( \frac{dQ}{dP} \right)^{1-\alpha} \right]
        + \frac{1}{\alpha} +  \frac{1}{1 - \alpha}  \nonumber\\
        &=& 	D_{\alpha} (Q||P).
    \end{eqnarray}
    Note that, the Legendre transform for $g_{\alpha}(x) =\mathbf{x}^{1-\alpha}/(1 - \alpha)$ is obatined as
    \begin{equation}
        g_{\alpha}^* (x)=  \frac{\alpha}{\alpha - 1}  x^{1- \frac{1}{\alpha}}. \label{Eq_conjugate_alpha}
    \end{equation}
    In addition, note that, for the Legendre transforms of any fuction $h(x)$, it hold that
    \begin{equation}
        \{C \cdot h(x)\}^*= C \cdot h^*\left(\frac{x}{C}\right) \quad \text{and} \quad \{h(x) + C\cdot t + D\}^*= h^*(x - C) - D. \label{Eq_formula_conjugate}
    \end{equation}
    Here, $A^*$ denotes the the Legendre transform of $A$.

    From (\ref{Eq_conjugate_alpha}) and (\ref{Eq_formula_conjugate}), we have
    \begin{eqnarray}
        f_{\alpha}^*(t)
        &=& \left\{ \frac{1}{(-\alpha)} g_{\alpha}(t) +  \frac{1}{\alpha} t + \frac{1}{1 - \alpha}  \right\}^* \nonumber \\
        &=& \frac{1}{(-\alpha)} g_{\alpha}^*  \left( -\alpha \cdot \left\{ t- \frac{1}{\alpha}\right\} \right)		- \frac{1}{1 - \alpha} \nonumber \\
        &=& - \frac{1}{\alpha} g_{\alpha}^*  \left(1  - \alpha t\right) + \frac{1}{\alpha - 1} \nonumber \\
        &=& - \frac{1}{\alpha} \left\{ \frac{\alpha}{\alpha - 1}  \left(1  - \alpha t\right)^{1- \frac{1}{\alpha}} \right\}   + \frac{1}{\alpha - 1} \nonumber \\
        &=& \frac{1}{1 - \alpha}  \left(1  - \alpha t\right)^{1- \frac{1}{\alpha}}  + \frac{1}{\alpha - 1}. \label{Eq_cojugate_f_alpha}
    \end{eqnarray}

    By differentiating $f_{\alpha}(t)$, we obtain
    \begin{equation}
        f_{\alpha}'(t) = -  \frac{1}{\alpha} t^{-\alpha} + \frac{1}{\alpha}. \label{Eq_derivative_f_alpha}
    \end{equation}
    Thus, we have
    \begin{eqnarray}
          E_Q \left[ f_{\alpha}'(\phi)  \right] =  E_Q\left[  - \frac{1}{\alpha} \phi^{-\alpha} + \frac{1}{\alpha} \right]. \label{Eq_Exp_f_derivative}
    \end{eqnarray}

    From (\ref{Eq_cojugate_f_alpha}) and (\ref{Eq_derivative_f_alpha}), we obtain
    \begin{eqnarray}
        E_P \left[ f_{\alpha}^*(f_{\alpha}'(\phi))  \right]  &=& E_P\left[
        \frac{1}{1 - \alpha}  \left\{ 1 -  \alpha \cdot \left( - \frac{1}{\alpha} \phi^{-\alpha}
        + \frac{1}{\alpha} \right)  \right\}^{1- \frac{1}{\alpha}} + \frac{1}{\alpha - 1} \right] \nonumber \\
        &=& E_P \left[ \frac{1}{1 - \alpha}  \phi^{1 - \alpha} + \frac{1}{\alpha - 1} \right]. \label{Eq_Exp_f_asta_f_derivative}
    \end{eqnarray}
    In additionm, from (\ref{Eq_Exp_f_derivative}) and (\ref{Eq_Exp_f_asta_f_derivative}), we see
    for both $ E_P[\phi^{1 - \alpha}] < \infty$ and $E_Q[\phi^{-\alpha}] < \infty$  to hold
    is equivalent for both  $E_P \left[ f_{\alpha}^*(f_{\alpha}'(\phi))  \right] < \infty$ and  $E_Q \left[ f_{\alpha}'(\phi)  \right] < \infty$ to hold.

    Finally, substituting (\ref{Eq_Exp_f_derivative}) and (\ref{Eq_Exp_f_asta_f_derivative}) for (\ref{Eq_Variational_representation_f_div_phi}), we have
    \begin{eqnarray}
        D_{\alpha} (Q||P) &=& \sup_{\phi \ge 0}\{ E_Q[f_{\alpha}'(\phi)] - E_P[f_{\alpha}^*(f_{\alpha}'(\phi)) ]\} \nonumber \\
        &=& \sup_{\phi \ge 0}
        \left\{
        E_Q \left[  - \frac{1}{\alpha} \phi^{-\alpha} + \frac{1}{\alpha} \right]
        - E_P \left[ \frac{1}{1 - \alpha}  \phi^{1 - \alpha} + \frac{1}{\alpha - 1} \right]
        \right\} \nonumber \\
        &=& \sup_{\phi \ge 0} \left\{
        \frac{1}{\alpha(1-\alpha)} 	- \frac{1}{\alpha} E_{Q} \left[\phi^{-\alpha} \right]
        - \frac{1}{1 - \alpha} E_{P} \left[\phi^{1 - \alpha} \right ]
        \right\}.	 \nonumber
    \end{eqnarray}

    This completes the proof.
\end{proof}

\begin{proposition} \label{proposition_alpha_div_resp_in_gibbs_dinsity_form}
    $\alpha$-divergence can be written as follows:
    \begin{equation}
        D_{\alpha} (Q||P) = \sup_{T:\mathbb{R}^d \rightarrow \mathbb{R}} \left\{
        \frac{1}{\alpha(1-\alpha)} - \frac{1}{\alpha} E_{Q} \left[e^{\alpha \cdot T} \right]
        - \frac{1}{1- \alpha} E_{P} \left[ e^{(\alpha - 1) \cdot T} \right] 	\right\}, \label{Lemma_Eq_loss_func_alpha_div}
    \end{equation}
    where supremum  is considered over all  measurable functions $T:\mathbb{R}^d \rightarrow \mathbb{R}$ with
    $E_P[e^{(\alpha - 1) \cdot T}] < \infty$ and $E_{Q} [e^{\alpha \cdot T}] < \infty$.
    The equality holds for $T^*$ satisfying
    \begin{equation}
        \frac{dQ}{dP} = e^{-T^*}.  \label{restated_Eq_gibbs_densty_minus}
    \end{equation}
\end{proposition}

\begin{proof}[proof of Proposition \ref{proposition_alpha_div_resp_in_gibbs_dinsity_form}]
    Substituting $e^{- T}$ for $\phi$ in (\ref{lemma_Eq_alpha_variation}), we have
    \begin{eqnarray}
        D_{\alpha} (Q||P) &=& \sup_{\phi \ge 0} \left\{
        \frac{1}{\alpha(1-\alpha)} - \frac{1}{\alpha} E_{Q} \left[\phi^{-\alpha} \right]  - \frac{1}{1 - \alpha} E_{P} \left[\phi^{1 - \alpha} \right]  \right\} \nonumber \\
        &=& \sup_{T:\mathbb{R}^d \rightarrow \mathbb{R}} \left\{
        \frac{1}{\alpha(1-\alpha)}  - \frac{1}{\alpha} E_{Q} \left[\left\{ e^{-T} \right\}^{-\alpha} \right]
        - \frac{1}{1 - \alpha} E_{P} \left[\left\{ e^{-T}\right\}^{1 - \alpha} \right]
        \right\} \nonumber \\
        &=& \sup_{T:\mathbb{R}^d \rightarrow \mathbb{R}} \left\{
        \frac{1}{\alpha(1-\alpha)} 	- \frac{1}{\alpha} E_{Q} \left[e^{\alpha \cdot T} \right]
        - \frac{1}{1- \alpha} E_{P} \left[ e^{(\alpha - 1) \cdot T} \right]
        \right\}. \label{Proof_lemma_opti_in_gibbs_Eq_state}
    \end{eqnarray}
    Finally, from Lemma \ref{lemma_alpha_div_resp}, the equality for (\ref{Proof_lemma_opti_in_gibbs_Eq_state}) holds if and only if
    \begin{equation}
        \frac{dQ}{dP} = e^{-T^*}.
    \end{equation}

    This completes the proof.
\end{proof}

\begin{proposition} \label{proposition_loss_optimal_T_is_log_density_ratio}
    For $T \in \mathcal{T}^{\alpha}$, let
    \begin{eqnarray}
         l_{\alpha} \left(\mathbf{X}_{\sim Q}, \mathbf{X}_{\sim P} ; T\right) &=&  \frac{1}{\alpha}
             e^{\alpha \cdot T(\mathbf{X}_{\sim Q} )} +  \frac{1}{1- \alpha} e^{(\alpha - 1) \cdot T(\mathbf{X}_{\sim P})}. \label{Eq_proposition_loss_optimal_T_is_log_density_ratio}
    \end{eqnarray}
    Then the optimal function $T^*$ for $\inf_{T:\mathbb{R}^d \rightarrow \mathbb{R} } l_{\alpha} \left( \mathbf{X}_{\sim Q}, \mathbf{X}_{\sim P} ; T \right)$ is obtained as $T^*= - \log  dQ/dP$, $\mu$-almost everywhere.
\end{proposition}

\begin{proof}[proof of Proposition \ref{proposition_loss_optimal_T_is_log_density_ratio}]
From the definition for $\mathbf{X}_{\sim Q}$ and $\mathbf{X}_{\sim P}$, we see
\begin{equation}
e^{\alpha \cdot T(\mathbf{X}_{\sim Q}=\mathbf{x})} = e^{\alpha \cdot T(\mathbf{X}=\mathbf{x})}  \cdot  \frac{dQ}{d\mu}(\mathbf{x}), \nonumber
\end{equation}
and
\begin{equation}
e^{(\alpha - 1) \cdot T(\mathbf{X}_{\sim P}=\mathbf{x})} =  e^{(\alpha - 1) \cdot T(\mathbf{X}=\mathbf{x})} \cdot  \frac{dP}{d\mu}(\mathbf{x}). \nonumber
\end{equation}
Subsequently, we obtain
\begin{equation}
    l_{\alpha}(\mathbf{X}_{\sim Q}=\mathbf{x}, \mathbf{X}_{\sim P}=\mathbf{x}; T) =  \frac{1}{\alpha}
    \cdot e^{\alpha \cdot T(\mathbf{X}=\mathbf{x})}  \cdot \frac{dQ}{d\mu} (\mathbf{x})
    +  \frac{1}{1 - \alpha} \cdot e^{(\alpha - 1) \cdot T(\mathbf{X}=\mathbf{x})} \cdot \frac{dP}{d\mu} (\mathbf{x}). \nonumber
\end{equation}

Note that, from Jensen's inequality, it holds that
\begin{equation}
   \log(p \cdot  X + q \cdot Y) \ge  p \cdot \log (X) + p \cdot \log (Y), \label{Eq_Jensen_s_inequality}
\end{equation}
for $X, Y > 0$ and $p, q>0$ with $p+q=1$, and the equality holds when $X=Y$.

From this equation, by letting $X= e^{\alpha \cdot T(\mathbf{X}=\mathbf{x})}  \cdot  \frac{dQ}{d\mu} (\mathbf{x})$,
$Y =  e^{(\alpha - 1) \cdot T(\mathbf{X}=\mathbf{x})} \cdot \frac{dP}{d\mu} (\mathbf{x})$,
$p=1 - \alpha$ and $q=\alpha$, we observe that
\begin{equation}
    \log (p \cdot  X + q \cdot Y) = \log \left( \frac{1}{\alpha  \cdot (1- \alpha)}  \cdot l_{\alpha}(\mathbf{X}_{\sim Q}=\mathbf{x}, \mathbf{X}_{\sim P}=\mathbf{x}; T) \right), \nonumber
\end{equation}

and $ \log \left( \frac{1}{\alpha  \cdot (1- \alpha)}  \cdot l_{\alpha}(\mathbf{X}_{\sim Q}=\mathbf{x}, \mathbf{X}_{\sim P}=\mathbf{x} ; T) \right)$ is minimized when
$e^{\alpha \cdot T(\mathbf{X}=\mathbf{x})}  \cdot  \frac{dQ}{d\mu} (\mathbf{x}) =  e^{(\alpha - 1) \cdot T(\mathbf{X}=\mathbf{x})} \cdot \frac{dP}{d\mu} (\mathbf{x})$,
$\mu$-almost everywhere.

Then, we see that $\inf_{T:\mathbb{R}^d \rightarrow \mathbb{R} } l_{\alpha}(\mathbf{X}_{\sim Q}= \mathbf{x}, \mathbf{X}_{\sim P}= \mathbf{x} ; T)$ is achieved at
$e^{- T^{*}(\mathbf{x})} = \frac{dQ}{dP}$, $\mu$-almost everywhere.
Hence, we have $T^*= - \log dQ/dP $, $\mu$-almost everywhere.

\end{proof}

\begin{proposition} \label{proposition_loss_T_is_optimal_when_constant_is_zero}
    For $T \in \mathcal{T}^{\alpha}$, let  $T^{+k} = T + k$.
    Then the optimal function  $T^*$ for $\inf_{k \in \mathbb{R}} l_{\alpha} \left( \mathbf{X}_{\sim Q}, \mathbf{X}_{\sim P} ; T^{+k} \right)$ is satisfying that $E_{P} \left[ e^{- T^*} \right] = 1$,
    where $l_{\alpha} \left(\mathbf{X}_{\sim Q}, \mathbf{X}_{\sim P} ; T\right)$ is defined as (\ref{Eq_proposition_loss_optimal_T_is_log_density_ratio}).
\end{proposition}

\begin{proof}[proof of Proposition \ref{proposition_loss_T_is_optimal_when_constant_is_zero}]
    From the definition for $\mathbf{X}_{\sim Q}$ and $\mathbf{X}_{\sim P}$, we see
    \begin{equation}
        e^{\alpha \cdot T^{+k}(\mathbf{X}_{\sim Q}=\mathbf{x})} = e^{\alpha \cdot T^{+k}(\mathbf{X}=\mathbf{x})}  \cdot  \frac{dQ}{d\mu}(\mathbf{x})
         = e^{\alpha \cdot k} \cdot e^{\alpha \cdot T(\mathbf{X}=\mathbf{x})} \cdot  \frac{dQ}{d\mu}(\mathbf{x}), \nonumber
    \end{equation}
    and
    \begin{equation}
        e^{(\alpha - 1) \cdot T(\mathbf{X}_{\sim P}=\mathbf{x})} =  e^{(\alpha - 1) \cdot T^{+k}(\mathbf{X}=\mathbf{x})} \cdot \frac{dP}{d\mu}(\mathbf{x})
        = e^{(\alpha - 1) \cdot k} \cdot e^{(\alpha - 1)\cdot T(\mathbf{X}=\mathbf{x})} \cdot \frac{dP}{d\mu}(\mathbf{x}). \nonumber
    \end{equation}
    Subsequently, we obtain
    \begin{equation}
        l_{\alpha}(\mathbf{X}_{\sim Q}=\mathbf{x}, \mathbf{X}_{\sim P}=\mathbf{x}; T) =
        e^{\alpha \cdot k} \cdot e^{\alpha \cdot T(\mathbf{X}=\mathbf{x})} \cdot \frac{dQ}{d\mu}(\mathbf{x})
        + e^{\alpha \cdot k} \cdot e^{\alpha \cdot T(\mathbf{X}=\mathbf{x})} \cdot \frac{dQ}{d\mu}(\mathbf{x}). \nonumber
    \end{equation}

    For Jensen's inequality (\ref{Eq_Jensen_s_inequality}), let
    $X = e^{\alpha \cdot k} \cdot e^{\alpha \cdot T(\mathbf{X}=\mathbf{x})} \cdot  \frac{dQ}{d\mu}(\mathbf{x})$,
    $Y = e^{(\alpha - 1) \cdot k} \cdot e^{(\alpha - 1)\cdot T(\mathbf{X}=\mathbf{x})} \cdot \frac{dP}{d\mu}(\mathbf{x})$,
    $p=1 - \alpha$ and $q=\alpha$.
    Then, $l_{\alpha} \left(\mathbf{X}_{\sim Q}, \mathbf{X}_{\sim P} ;T^{+k}\right)$ is minimized when
    $e^{\alpha \cdot k_*} \cdot e^{\alpha \cdot T(\mathbf{X}=\mathbf{x})} \cdot  \frac{dQ}{d\mu}(\mathbf{x}) = e^{(\alpha - 1) \cdot k_*} \cdot e^{(\alpha - 1)\cdot T(\mathbf{X}=\mathbf{x})} \cdot \frac{dP}{d\mu}(\mathbf{x})$,
    $\mu$-almost everywhere.

    Then, we see
    \begin{equation}
         e^{- k_*} \cdot \frac{dQ}{d\mu}(\mathbf{x}) =  e^{- T(\mathbf{x})} \cdot \frac{dP}{d\mu}(\mathbf{x}). \nonumber
    \end{equation}
    By integrating both sides of the above equality over $\mathbb{R}^d$ with $\mu$, we obtain
    \begin{equation}
        e^{- k_*} =  E_P\left[ e^{- T}  \right]. \nonumber
    \end{equation}
    From this, we have
    \begin{equation}
     e^{k_*} \cdot E_P\left[ e^{- T}  \right] = 1. \nonumber
    \end{equation}

    Now, since $T_* = T + k_*$, we see
    \begin{equation}
         E_P\left[ e^{- T_*}  \right] = E_P\left[ e^{- T + k_*} \right] = e^{k_*} \cdot E_P\left[ e^{- T} \right] = 1. \nonumber
    \end{equation}

    This completes the proof.
\end{proof}

\begin{proposition}\label{proposition_loss_is_mu_storongly_convex}
    For a fixed point $\mathbf{x}_0 \in \mathbb{R}^d$, suppose that $\frac{dQ}{d\mu}(\mathbf{x}_0) > 0$ and $\frac{dP}{d\mu}(\mathbf{x}_0) > 0$.
    For a constant $L > 0$, let $I_{L}$ denote an interval $[-L - \log \frac{dQ}{dP}(\mathbf{x}_0), L - \log \frac{dQ}{dP}(\mathbf{x}_0)]$.
    Subsequently, let $f: I_{L} \rightarrow \mathbb{R}$ be a function as follows:
    \begin{equation}
        f(t)  =  \frac{1}{\alpha} \cdot e^{\alpha \cdot t} \cdot \frac{dQ}{d\mu}(\mathbf{x}_0) +  \frac{1}{1- \alpha} \cdot e^{(\alpha - 1) \cdot t} \cdot \frac{dP}{d\mu}(\mathbf{x}_0).
    \end{equation}
    Let $\frac{\lambda}{2} = e^{-  L} \cdot \left\{\frac{dQ}{d\mu}(\mathbf{x}_0) \right\}^{1 - \alpha} \left\{\frac{dP}{d\mu}(\mathbf{x}_0) \right\}^{\alpha}$.
    Then, $f(t)$ satisfies $f''(t) \cdot t^2 \geq \frac{\lambda}{2} \cdot t^2 $ for all $t \in I_{L}$, that is, $f(t)$ is $\lambda$-strongly convex.
    In addition, $\left|f'(t) \right| \leq 2 \cdot e^{L} \cdot \left\{\frac{dQ}{d\mu}(\mathbf{x}_0) \right\}^{1 - \alpha} \left\{\frac{dP}{d\mu}(\mathbf{x}_0) \right\}^{\alpha}$ holds for all $t \in I_{L}$,
    and $f(t)$ is minimized at $t_*= - \log \frac{dQ}{dP}(\mathbf{x}_0)$.
\end{proposition}

\begin{proof}[proof of Proposition \ref{proposition_loss_is_mu_storongly_convex}]
    By repeating the derivative of $f(t)$, we obtain
    \begin{equation}
        f'(t)  =  e^{\alpha \cdot t} \cdot \frac{dQ}{d\mu}(\mathbf{x}) - e^{(\alpha - 1) \cdot t} \cdot \frac{dP}{d\mu}(\mathbf{x}), \label{proof_proposition_loss_is_mu_storongly_convex_eq_0}
    \end{equation}
    and
    \begin{equation}
    f''(t)  = \alpha \cdot e^{\alpha \cdot t} \cdot \frac{dQ}{d\mu}(\mathbf{x}) + (1-\alpha) \cdot e^{(\alpha - 1) \cdot t} \cdot \frac{dP}{d\mu}(\mathbf{x}).
    \label{proof_proposition_loss_is_mu_storongly_convex_eq_4}
    \end{equation}

    First, we see that $f''(t)\geq \frac{\lambda}{2}$ holds for all $t \in I_{L}$.
    From (\ref{proof_proposition_loss_is_mu_storongly_convex_eq_4}), we have
    \begin{eqnarray}
        f''(t) &\geq& \alpha \cdot e^{\alpha \cdot (t_* - L)} \cdot \frac{dQ}{d\mu}(\mathbf{x}) + (1-\alpha) \cdot e^{(\alpha - 1) \cdot (t_* + L)} \cdot \frac{dP}{d\mu}(\mathbf{x}) \nonumber\\
        &\geq& \alpha \cdot e^{- \alpha \cdot L} \cdot e^{\alpha \cdot t_*} \cdot \frac{dQ}{d\mu}(\mathbf{x})
               + (1-\alpha) \cdot e^{(\alpha - 1) \cdot L} \cdot e^{(\alpha - 1) \cdot t_*}\cdot \frac{dP}{d\mu}(\mathbf{x}) \nonumber\\
        &\geq& \alpha \cdot e^{- L} \cdot e^{\alpha \cdot t_*} \cdot \frac{dQ}{d\mu}(\mathbf{x})
             + (1-\alpha)\cdot e^{- L} \cdot e^{(\alpha - 1) \cdot t_*}\cdot \frac{dP}{d\mu}(\mathbf{x}) \nonumber\\
        &=& e^{- L} \cdot \left\{ \alpha \cdot e^{\alpha \cdot t_*} \cdot \frac{dQ}{d\mu}(\mathbf{x})
                     + (1-\alpha) \cdot e^{(\alpha - 1) \cdot t_*}\cdot \frac{dP}{d\mu}(\mathbf{x})  \right\}. \label{proof_proposition_loss_is_mu_storongly_convex_eq_1}
    \end{eqnarray}
    Note that, from  (\ref{Eq_Jensen_s_inequality}), we see
    \begin{equation}
        p \cdot  X + q \cdot Y  \ge  X^p \cdot Y^q,
    \end{equation}
    for $X, Y > 0$ and $p, q>0$ with $p+q=1$, and the equality holds when $X=Y$.
    By letting  $X= e^{ \alpha \cdot t_*} \cdot \frac{dQ}{d\mu}(\mathbf{x})$, $Y = e^{(\alpha - 1) \cdot t_*}\cdot \frac{dP}{d\mu}(\mathbf{x})$,
    $p=\alpha$, and $q=1 - \alpha$ in the above equality, we obtain
    \begin{eqnarray}
        \alpha \cdot e^{\alpha \cdot t_*} \cdot \frac{dQ}{d\mu}(\mathbf{x})
        + (1-\alpha) \cdot e^{(\alpha - 1) \cdot t_*}\cdot \frac{dP}{d\mu}(\mathbf{x})  &\geq&
        \left\{e^{\alpha \cdot t_*} \cdot \frac{dQ}{d\mu}(\mathbf{x}) \right\}^{\alpha}
        \left\{ e^{(\alpha - 1) \cdot t_*}\cdot \frac{dP}{d\mu}(\mathbf{x}) \right\}^{1-\alpha} \nonumber \\
        &=& e^{ \left\{ \alpha^2 - (1 -\alpha)^2  \right\} \cdot t_*}  \left\{\frac{dQ}{d\mu}(\mathbf{x}) \right\}^{\alpha}  \left\{  \frac{dP}{d\mu}(\mathbf{x})  \right\}^{1-\alpha} \nonumber\\
        &=& \left\{\frac{dQ}{dP}(\mathbf{x}) \right\}^{1 - 2 \alpha}  \left\{\frac{dQ}{d\mu}(\mathbf{x}) \right\}^{\alpha}  \left\{  \frac{dP}{d\mu}(\mathbf{x})  \right\}^{1-\alpha} \nonumber\\
        &=& \left\{\frac{dQ}{d\mu}(\mathbf{x}) \right\}^{1 - \alpha}  \left\{  \frac{dP}{d\mu}(\mathbf{x})  \right\}^{\alpha}. \label{proof_proposition_loss_is_mu_storongly_convex_eq_2}
    \end{eqnarray}
    Thus, from (\ref{proof_proposition_loss_is_mu_storongly_convex_eq_1}) and (\ref{proof_proposition_loss_is_mu_storongly_convex_eq_2}), we see
    $f''(t)\geq \frac{\lambda}{2}$ holds for all $t \in I_{L}$.

    Next, we obtain $\left|f'(t) \right| \leq 2  e^{L}  \left\{\frac{dQ}{d\mu}(\mathbf{x}_0) \right\}^{1 - \alpha} \left\{\frac{dP}{d\mu}(\mathbf{x}_0) \right\}^{\alpha}$.
    From (\ref{proof_proposition_loss_is_mu_storongly_convex_eq_0}), we have
    \begin{eqnarray}
        \left|f'(t) \right| &\leq& e^{\alpha \cdot (t_* + L)} \cdot \frac{dQ}{d\mu}(\mathbf{x}) + e^{(\alpha - 1) \cdot (t_* - L)} \cdot \frac{dP}{d\mu}(\mathbf{x}) \nonumber\\
        &=& e^{\alpha \cdot t_*} \cdot e^{\alpha \cdot L} \cdot \frac{dQ}{d\mu}(\mathbf{x})
            + e^{(\alpha - 1) \cdot t_*} \cdot e^{(1 - \alpha) \cdot L} \cdot \frac{dP}{d\mu}(\mathbf{x}) \nonumber\\
        &\leq&
        e^{L} \left\{  e^{\alpha \cdot t_*}  \cdot \frac{dQ}{d\mu}(\mathbf{x}) +  e^{(\alpha - 1) \cdot t_*} \cdot \frac{dP}{d\mu}(\mathbf{x}) \right\}   \nonumber\\
        &=&  e^{L} \left[ \left\{ \frac{dQ}{dP}(\mathbf{x}) \right\}^{- \alpha} \frac{dQ}{d\mu}(\mathbf{x})
             + \left\{\frac{dQ}{dP}(\mathbf{x}) \right\}^{1 - \alpha} \frac{dP}{d\mu}(\mathbf{x})   \right] \nonumber\\
        &=& e^{L}
             \left[  \left\{ \frac{dQ}{d\mu}(\mathbf{x}) \right\}^{1 -\alpha}  \left\{\frac{dP}{d\mu}(\mathbf{x}) \right\}^{\alpha}
                    +  \left\{ \frac{dQ}{d\mu}(\mathbf{x}) \right\}^{1 -\alpha}  \left\{\frac{dP}{d\mu}(\mathbf{x}) \right\}^{\alpha} \right].   \label{proof_proposition_loss_is_mu_storongly_convex_eq_3}
    \end{eqnarray}
    Here, we see $f'(t) \leq 2 \cdot e^{L}  \left\{\frac{dQ}{d\mu}(\mathbf{x}_0) \right\}^{1 - \alpha}   \left\{\frac{dP}{d\mu}(\mathbf{x}_0) \right\}^{\alpha}$.

    The rest of the proposition statement follows from Lemma \ref{proposition_loss_optimal_T_is_log_density_ratio}.

    This completes the proof.
\end{proof}

\begin{corollary}\label{corollary_loss_is_mu_storongly_convex}
    For $N$ fixed points $\left\{\mathbf{x}_i \right\}_{i=1}^N \subset \mathbb{R}^d$, suppose that $\frac{dQ}{d\mu}(\mathbf{x}_i) > 0$ and $\frac{dP}{d\mu}(\mathbf{x}_i) > 0$ ($1\leq \forall \leq N$).
    For a constant $L > 0$, let $I_{L}^i$ denote an interval $[-L - \log \frac{dQ}{dP}(\mathbf{x}_i), L - \log \frac{dQ}{dP}(\mathbf{x}_i)]$.
    Subsequently, let $f^{(N)}: I_{L}^1\times I_{L}^2\times \cdots\times I_{L}^N \rightarrow \mathbb{R}$ be a function as follows:
    \begin{equation}
        f^{(N)}(\mathbf{t}) = f^{(N)}(t_1, t_2, \ldots, t_N) =  \frac{1}{\alpha}  \frac{1}{N} \sum_{i=1}^N e^{\alpha \cdot t_i} \cdot \frac{dQ}{d\mu}(\mathbf{x}_i)
        +  \frac{1}{1- \alpha} \frac{1}{N} \sum_{i=1}^N   e^{(\alpha - 1) \cdot t_i} \cdot \frac{dP}{d\mu}(\mathbf{x}_i),
    \end{equation}
    and let
    \begin{equation}
       \frac{\lambda}{2} = \frac{1}{N} \cdot \min_{1 \leq i \leq N} \left\{  e^{-  L} \cdot \left\{\frac{dQ}{d\mu}(\mathbf{x}_i) \right\}^{1 - \alpha}
          \left\{\frac{dP}{d\mu}(\mathbf{x}_i) \right\}^{\alpha} \right\}.
    \end{equation}
    Then, $f^{(N)}(\mathbf{t})$ satisfies
    \begin{equation}
       \mathbf{t}^T \cdot \nabla^2 f^{(N)}(\mathbf{t}) \cdot \mathbf{t} = \sum_{1\leq i,j \leq N} t_i\cdot t_j \cdot \frac{\partial^2 f^{(N)}}{\partial t_i t_j} \geq \frac{\lambda}{2} \cdot \|\mathbf{t}\|^2,
    \end{equation}
    that is, $f^{(N)}(\mathbf{t})$ is $\lambda$-strongly convex.

    In addition, let $D_i =  \allowbreak 2 \cdot e^{L} \cdot \allowbreak \left\{\frac{dQ}{d\mu}(\mathbf{x}_i) \right\}^{1 - \alpha} \left\{\frac{dP}{d\mu}(\mathbf{x}_i) \right\}^{\alpha}$, and let
    $D = \max\left\{D_1, \allowbreak D_2,  \allowbreak \ldots, \allowbreak D_N \right\}^2$.
    Then,
    \begin{equation}
        \left\|\nabla f^{(N)}(\mathbf{t}) \right\|^2 = \left\| \left(\frac{\partial}{\partial t_1}f^{(N)}(\mathbf{t}),
                                      \frac{\partial}{\partial t_2}f^{(N)}(\mathbf{t}), \ldots,
                                      \frac{\partial}{\partial t_N}f^{(N)}(\mathbf{t}) \right) \right\|^2 \leq D^2,
    \end{equation}
    for all $\mathbf{t} \allowbreak \in \allowbreak I_{L}^1\allowbreak \times\allowbreak I_{L}^2\allowbreak \times\allowbreak \cdots\allowbreak \times\allowbreak I_{L}^N$,
    and $f^{(N)}(\mathbf{t})$ is minimized at
    $\mathbf{t}_* = \allowbreak (t_*^1, \allowbreak t_*^2, \allowbreak \ldots, \allowbreak t_*^N) = \allowbreak
    (- \log \frac{dQ}{dP}(\mathbf{x}_1), \allowbreak - \log \frac{dQ}{dP}(\mathbf{x}_2), \allowbreak \ldots, \allowbreak - \log \frac{dQ}{dP}(\mathbf{x}_N))$.
\end{corollary}

\begin{proof}[proof of Corollary \ref{corollary_loss_is_mu_storongly_convex}]
    Let
    \begin{equation}
        f_i(t)  =  \frac{1}{\alpha}  e^{\alpha \cdot t} \cdot \frac{dQ}{d\mu}(\mathbf{x}_i) +  \frac{1}{1- \alpha} e^{(\alpha - 1) \cdot t} \cdot \frac{dP}{d\mu}(\mathbf{x}_i),
    \end{equation}
    and let $\frac{\lambda_i}{2} = e^{-  L} \cdot \left\{\frac{dQ}{d\mu}(\mathbf{x}_i) \right\}^{1 - \alpha} \left\{\frac{dP}{d\mu}(\mathbf{x}_i) \right\}^{\alpha}$.
    From Proposition \ref{proposition_loss_is_mu_storongly_convex}, for each $1 \le i \le N$,
    $f_i(t)$ satisfies that $f''_i(t) \geq \frac{\lambda_i}{2} \cdot t$  and
    $f'_i(t) \leq D_i$ holds for all $t \in I_{L}^i$, and $f_i(t)$ is minimized at $t_*^i= - \log \frac{dQ}{dP}(\mathbf{x}_i)$.

    Note that,
    \begin{eqnarray}
        \nabla^2 f^{(N)}(\mathbf{t})  &=&
       \begin{pmatrix}
            \frac{\partial^2}{\partial t_1 t_1} f^{(N)}(\mathbf{t}) & \frac{\partial^2}{\partial t_1 t_2} f^{(N)}(\mathbf{t}) & \cdots & \frac{\partial^2}{\partial t_1 t_N} f^{(N)}(\mathbf{t}) \\
            \frac{\partial^2}{\partial t_2 t_1} f^{(N)}(\mathbf{t}) & \frac{\partial^2}{\partial t_2 t_2} f^{(N)}(\mathbf{t}) & & \\
            \vdots &   & \ddots &  &\\
            \frac{\partial^2}{\partial t_N t_j} f^{(N)}(\mathbf{t}) & \frac{\partial^2}{\partial t_i t_j} f^{(N)}(\mathbf{t}) & \cdots & \frac{\partial^2}{\partial t_N t_N} f^{(N)}(\mathbf{t})
        \end{pmatrix} \nonumber\\
    &=&
    \begin{pmatrix}
        \frac{1}{N} \cdot f''_1(t_1)                                              \\
        & \frac{1}{N} \cdot f''_2(t_2)         &        & \text{\huge{0}}   \\
        &                 & \ddots                     \\
        & \text{\huge{0}} &        & \ddots            \\
        &                 &        &           & \frac{1}{N} \cdot f''_N(t_N)
    \end{pmatrix}.
    \end{eqnarray}
    From this, we see
    \begin{eqnarray}
        \mathbf{t}^T \cdot \nabla^2 f^{(N)}(\mathbf{t}) \cdot \mathbf{t}
        &=& \frac{1}{N} \sum_{i=1}^N f''_i(t_i) \cdot t_i^2  \nonumber\\
        &\geq& \frac{1}{N} \sum_{i=1}^N \frac{\lambda_i}{2} \cdot t_i^2 \nonumber\\
         &=&  \sum_{i=1}^N \frac{1}{N} \cdot \frac{\lambda_i}{2} \cdot t_i^2 \nonumber\\
        &\geq& \sum_{i=1}^N \frac{\lambda}{2} \cdot t_i^2 \nonumber\\
        &=&  \frac{\lambda}{2} \cdot \|\mathbf{t}\|^2.
    \end{eqnarray}

    In addition, since $f^{(N)}(t) = \frac{1}{N} \sum_{i=1}^N f_i(t_i)$, we have
    \begin{eqnarray}
        \left\| \nabla f^{(N)}(\mathbf{t}) \right\|^2 &=& \left\| \left(\frac{\partial}{\partial t_1} f^{(N)}(\mathbf{t}),
        \frac{\partial}{\partial t_2} f^{(N)}(\mathbf{t}), \ldots, \frac{\partial}{\partial t_N}f^{(N)}(\mathbf{t}) \right) \right\|^2 , \nonumber\\
        &=& \left\| \left(\frac{1}{N} \cdot f_1'(t_1), \frac{1}{N} \cdot  f_2'(t_2), \ldots, \frac{1}{N} \cdot  f_N' (t_N)\right) \right\|^2  \nonumber\\
        &\le& \left\| \left(\frac{1}{N} \cdot D_1, \frac{1}{N} \cdot D_2, \ldots, \frac{1}{N} \cdot D_N\right) \right\|^2 \nonumber\\
        &\le&  \left\| \left(\frac{1}{N} \cdot D, \frac{1}{N} \cdot D, \ldots, \frac{1}{N} \cdot D\right) \right\|^2 \nonumber\\
        &=&  D^2,
    \end{eqnarray}
    and $f^{(N)}(\mathbf{t})$ is minimized at $\mathbf{t}_* = \allowbreak (t_*^1, \allowbreak t_*^2, \allowbreak \ldots, \allowbreak t_*^N) \allowbreak = \allowbreak (- \log \frac{dQ}{dP}(\mathbf{x}_1), \allowbreak - \log \frac{dQ}{dP}(\mathbf{x}_2), \allowbreak \ldots, \allowbreak - \log \frac{dQ}{dP}(\mathbf{x}_N))$.

    This completes the proof.
\end{proof}

\begin{lemma} \label{lemma_loss_not_biased_lemma}
    For $T \in \mathcal{T}^{\alpha}$, let
    \begin{eqnarray}
        l_{\alpha}(\mathbf{X}^{i}_{\sim Q}, \mathbf{X}^{i}_{\sim P} ; T) &=&  \frac{1}{\alpha} \cdot e^{\alpha \cdot T(\mathbf{X}^{i}_{\sim Q} )} +  \frac{1}{1- \alpha} \cdot e^{(\alpha - 1) \cdot T(\mathbf{X}^{i}_{\sim P})}, \\
        L_{\alpha}(Q,P;T) &=&  \frac{1}{\alpha} \cdot E_{Q} \left[ e^{\alpha \cdot T(\mathbf{X})} \right] +  \frac{1}{1- \alpha} \cdot E_{P} \left[ e^{(\alpha - 1) \cdot T(\mathbf{X})} \right] ,   \label{Eq_lemma_loss_not_biased_lemma_loss_T}
    \end{eqnarray}
    and let
    \begin{eqnarray}
       l_{\alpha}(\mathbf{X}^{i}_{\sim Q}, \mathbf{X}^{i}_{\sim P}) &=& \inf_{T:\mathbb{R}^d \rightarrow \mathbb{R} } 	l_{\alpha}(\mathbf{X}^{i}_{\sim Q}, \mathbf{X}^{i}_{\sim P} ; T),
       \label{eq_lemma_loss_not_biased_lemma_inf_l} \\
       L_{\alpha}(Q,P) &=&  \inf_{T:\mathbb{R}^d \rightarrow \mathbb{R} } L_{\alpha}(Q,P;T), \label{eq_lemma_loss_not_biased_lemma_inf_L}
    \end{eqnarray}
    where the infimums of (\ref{eq_lemma_loss_not_biased_lemma_inf_l}) and (\ref{eq_lemma_loss_not_biased_lemma_inf_L}) are considered over all measurable functions with $T:\mathbb{R}^d \rightarrow \mathbb{R}$ with  $E_P[e^{(\alpha - 1) \cdot T}] < \infty$ and $E_Q[e^{\alpha \cdot T}] < \infty$.

    In addition, let
    \begin{eqnarray}
       \hat{L}_{\alpha}^{(N)}(Q, P;T)  &=&  \frac{1}{\alpha}  \frac{1}{N} \sum_{i=1}^N e^{\alpha \cdot T(\mathbf{X}^{i}_{\sim Q} )} +  \frac{1}{1- \alpha} \frac{1}{N} \sum_{i=1}^N   e^{(\alpha - 1) \cdot T(\mathbf{X}^{i}_{\sim P})},  \label{Eq_lemma_loss_not_biased_empirical_lemma_loss_T}  \\
       \hat{L}_{\alpha}^{(N)}(Q, P)  &=&   \inf_{T:\mathbb{R}^d \rightarrow \mathbb{R} }\hat{L}_{\alpha}^{(N)}(Q, P;T).
    \end{eqnarray}

    Then, it holds that
    \begin{equation}
        E_{\mu} \left[ \hat{L}_{\alpha}^{(N)}(Q,P;T) \right]  = E_{\mu} \left[l_{\alpha}(\mathbf{X}^{i}_{\sim Q}, \mathbf{X}^{i}_{\sim P} ; T)\right] = L_{\alpha}(Q,P;T), \label{Lemma_Eq_state_l_t}
    \end{equation}
    \begin{equation}
        E_{\mu} \left[ \hat{L}_{\alpha}^{(N)}(Q,P) \right] = E_{\mu} \left[ l_{\alpha}(\mathbf{X}^{i}_{\sim Q}, \mathbf{X}^{i}_{\sim P}) \right] = L_{\alpha}(Q, P). \label{Lemma_Eq_state_inf_t_l_t}
    \end{equation}
\end{lemma}

\begin{proof}[proof of Lemma \ref{lemma_loss_not_biased_lemma}]
    We first show the last equality in (\ref{Lemma_Eq_state_inf_t_l_t}) holds.
    Now, we consider the following integral:
    \begin{eqnarray}
        E\left[  \frac{1}{\alpha(1-\alpha)} -  l_{\alpha}(\mathbf{X}^{i}_{\sim Q}, \mathbf{X}^{i}_{\sim P}) \right]
        &=& \int \left\{ \frac{1}{\alpha(1-\alpha)} -  l_{\alpha}(\mathbf{X}^{i}_{\sim Q}, \mathbf{X}^{i}_{\sim P}) \right\} d\mu \nonumber\\
        &=&  \int
        \left\{ \frac{1}{\alpha(1-\alpha)} - \inf_{T:\mathbb{R}^d \rightarrow \mathbb{R} }  l_{\alpha}(\mathbf{X}^{i}_{\sim Q}, \mathbf{X}^{i}_{\sim P} ; T) \right\} d\mu \nonumber\\
        &=& \int \sup_{T:\mathbb{R}^d \rightarrow \mathbb{R} }
        \left\{ \frac{1}{\alpha(1-\alpha)} -  l_{\alpha}(\mathbf{X}^{i}_{\sim Q}, \mathbf{X}^{i}_{\sim P} ; T) \right\} d\mu. \nonumber\\
        \label{proof_lemmma_loss_not_biased_lemma_to_get_upper_bound}
    \end{eqnarray}

    Let $T^*$ be the opimal function for (\ref{Lemma_Eq_loss_func_alpha_div}) in Lemma \ref{proposition_alpha_div_resp_in_gibbs_dinsity_form}.
    Let $T_k = - \log dQ/dP + 1/k $, then from Proposition \ref{proposition_loss_optimal_T_is_log_density_ratio}, we have
    \begin{equation}
        \lim_{k \rightarrow \infty} l_{\alpha}(\mathbf{X}^{i}_{\sim Q}, \mathbf{X}^{i}_{\sim P} ; T_k)
        = \inf_{T:\mathbb{R}^d \rightarrow \mathbb{R}} l_{\alpha}(\mathbf{X}^{i}_{\sim Q}, \mathbf{X}^{i}_{\sim P} ; T).
    \end{equation}
    From this, we obtain
    \begin{eqnarray}
        \lim_{k \rightarrow \infty}  E \left[  \frac{1}{\alpha(1-\alpha)} -  l_{\alpha}(\mathbf{X}^{i}_{\sim Q}, \mathbf{X}^{i}_{\sim P} ; T_k) \right]
        &=& \frac{1}{\alpha(1-\alpha)} - \lim_{k \rightarrow \infty}    E \left[ l_{\alpha}(\mathbf{X}^{i}_{\sim Q}, \mathbf{X}^{i}_{\sim P} ; T_k) \right] \nonumber\\
        &=& \frac{1}{\alpha(1-\alpha)} - \inf_{T:\mathbb{R}^d \rightarrow \mathbb{R}} l_{\alpha}(\mathbf{X}^{i}_{\sim Q}, \mathbf{X}^{i}_{\sim P} ; T) \nonumber\\
        &=&  \sup_{T:\mathbb{R}^d \rightarrow \mathbb{R} }\Big\{ E \left[ \frac{1}{\alpha(1-\alpha)}  -   l_{\alpha}(\mathbf{X}^{i}_{\sim Q}, \mathbf{X}^{i}_{\sim P} ; T) \ \right]  \Big\}. \nonumber\\
         \label{Eq_Lemma_loss_not_biased_lemma_N_left_1}
    \end{eqnarray}
    Now, from Lemma \ref{proposition_alpha_div_resp_in_gibbs_dinsity_form}, we see
    \begin{eqnarray}
        \left| \frac{1}{\alpha(1-\alpha)} - l_{\alpha}(\mathbf{X}^{i}_{\sim Q}, \mathbf{X}^{i}_{\sim P}) \right|
        &=& \left| \sup_{T:\mathbb{R}^d \rightarrow \mathbb{R} }
        \left\{ \frac{1}{\alpha(1-\alpha)} -  l_{\alpha}(\mathbf{X}^{i}_{\sim Q}, \mathbf{X}^{i}_{\sim P} ; T) \right\}
        \right|  \label{Lemma_loss_not_biased_lemma_to_exchenge} \nonumber \\
        &=& \frac{1}{\alpha(1-\alpha)} - \frac{1}{\alpha}  e^{\alpha \cdot T^*(\mathbf{X}^{i}_{\sim Q} )}
        -  \frac{1}{1- \alpha} e^{(\alpha - 1) \cdot T^*(\mathbf{X}^{i}_{\sim P})}   \nonumber \\
        &=& \frac{1}{\alpha(1-\alpha)} - \frac{1}{\alpha}  \left(\frac{dQ}{dP}(\mathbf{X}^{i}_{\sim Q} ) \right)^{\alpha}
        -\frac{1}{1- \alpha} \left(\frac{dQ}{dP} (\mathbf{X}^{i}_{\sim P}) \right)^{\alpha - 1}. \nonumber \\
        &&  \label{Lemma_loss_not_biased_lemma_integral_is_finte}
    \end{eqnarray}
    Let $\phi(\mathbf{X})$ denote the term on the right hand side of (\ref{Lemma_loss_not_biased_lemma_integral_is_finte}).
    Then, we observe that
    \begin{equation}
        \left|\frac{1}{\alpha(1-\alpha)} - l_{\alpha}(\mathbf{X}^{i}_{\sim Q}, \mathbf{X}^{i}_{\sim P}; T) \right| \leq \phi(\mathbf{X})\quad \text{and}  \quad E[\phi(\mathbf{X})] < \infty.
        \nonumber
    \end{equation}
    That is, we see that the following sequence is uniformaly integrable for $\mu$:
    \begin{equation}
        \left\{ \frac{1}{\alpha(1-\alpha)} -  l_{\alpha}(\mathbf{X}^{i}_{\sim Q}, \mathbf{X}^{i}_{\sim P} ; T_k) \right\}_{k=1}^N.
        \nonumber
    \end{equation}
    Thus, from the property of the Lebesgue integral (\citealp{shiryaev1995probability}, P188, Theorem 4), we obtain
    \begin{equation}
        E \left[   \lim_{k \rightarrow \infty} \left\{ \frac{1}{\alpha(1-\alpha)} -  l_{\alpha}(\mathbf{X}^{i}_{\sim Q}, \mathbf{X}^{i}_{\sim P} ; T_k) \right\} \right]
        =  \lim_{k \rightarrow \infty} E \left[  \frac{1}{\alpha(1-\alpha)} -  l_{\alpha}(\mathbf{X}^{i}_{\sim Q}, \mathbf{X}^{i}_{\sim P} ; T_k) \right]. \label{exchange_integral_sup}
    \end{equation}

    Finaly, from (\ref{Eq_Lemma_loss_not_biased_lemma_N_left_1}) and (\ref{exchange_integral_sup}), we have
    \begin{eqnarray}
        \frac{1}{\alpha(1-\alpha)} - E\left[ l_{\alpha}(\mathbf{X}^{i}_{\sim Q}, \mathbf{X}^{i}_{\sim P}) \right]
        &=&  E\left[  \frac{1}{\alpha(1-\alpha)} -  l_{\alpha}(\mathbf{X}^{i}_{\sim Q}, \mathbf{X}^{i}_{\sim P}) \right] \nonumber  \\
        &=&  E \left[ \sup_{T:\mathbb{R}^d \rightarrow \mathbb{R} }
        \Big\{ \frac{1}{\alpha(1-\alpha)}  -  l_{\alpha}(\mathbf{X}^{i}_{\sim Q}, \mathbf{X}^{i}_{\sim P} ; T)
        \Big\}  \ \right] \nonumber  \\
        &=& E \left[   \lim_{k \rightarrow \infty} \left\{ \frac{1}{\alpha(1-\alpha)} -  l_{\alpha}(\mathbf{X}^{i}_{\sim Q}, \mathbf{X}^{i}_{\sim P} ; T_k) \right\} \right] \nonumber  \\
        &=& \lim_{k \rightarrow \infty} E \left[  \frac{1}{\alpha(1-\alpha)} -  l_{\alpha}(\mathbf{X}^{i}_{\sim Q}, \mathbf{X}^{i}_{\sim P} ; T_k) \right]  \qquad \therefore (\ref{exchange_integral_sup})
        \nonumber \\
        &=& \sup_{T:\mathbb{R}^d \rightarrow \mathbb{R} }\Big\{ E \left[ \frac{1}{\alpha(1-\alpha)}  -   l_{\alpha}(\mathbf{X}^{i}_{\sim Q}, \mathbf{X}^{i}_{\sim P} ; T) \ \right]  \Big\}
         \qquad \therefore (\ref{Eq_Lemma_loss_not_biased_lemma_N_left_1}) \nonumber \\
        &=& \frac{1}{\alpha(1-\alpha)}  -   \inf_{T:\mathbb{R}^d \rightarrow \mathbb{R} } 	l_{\alpha}(\mathbf{X}^{i}_{\sim Q}, \mathbf{X}^{i}_{\sim P} ; T)
        \qquad \nonumber \\
        &=& \frac{1}{\alpha(1-\alpha)}  -  L_{\alpha}(Q,P).
    \end{eqnarray}
    Here, we see
    \begin{equation}
        E\left[ l_{\alpha}(\mathbf{X}^{i}_{\sim Q}, \mathbf{X}^{i}_{\sim P}) \right] = L_{\alpha}(Q, P). \label{Lemma_loss_not_biased_lemma_N_for_final_state_left_1}
    \end{equation}

    Next, we show the first equality in (\ref{Lemma_Eq_state_inf_t_l_t}) holds.
    Note that, it holds that
    \begin{equation}
        \frac{1}{N}  \sum_{i=1}^N 	\inf_{T_i:\mathbb{R}^d \rightarrow \mathbb{R} }  l_{\alpha}(\mathbf{X}^{i}_{\sim Q}, \mathbf{X}^{i}_{\sim P} ; T_i)
        \le \inf_{T:\mathbb{R}^d \rightarrow \mathbb{R} }  \frac{1}{N}  \sum_{i=1}^N  l_{\alpha}(\mathbf{X}^{i}_{\sim Q}, \mathbf{X}^{i}_{\sim P} ; T)
        \le  \frac{1}{N}  \sum_{i=1}^N  l_{\alpha}(\mathbf{X}^{i}_{\sim Q}, \mathbf{X}^{i}_{\sim P} ; T^*).
    \end{equation}
    Since $	\inf_{T_i:\mathbb{R}^d \rightarrow \mathbb{R} }  l_{\alpha}(\mathbf{X}^{i}_{\sim Q}, \mathbf{X}^{i}_{\sim P} ; T_i) = l_{\alpha}(\mathbf{X}^{i}_{\sim Q}, \mathbf{X}^{i}_{\sim P} ; T_*)$ from Proposition \ref{proposition_loss_optimal_T_is_log_density_ratio}, we have
    \begin{equation}
        \frac{1}{N}  \sum_{i=1}^N  l_{\alpha}(\mathbf{X}^{i}_{\sim Q}, \mathbf{X}^{i}_{\sim P} ; T^*)
        \le \inf_{T:\mathbb{R}^d \rightarrow \mathbb{R} }  \frac{1}{N}  \sum_{i=1}^N  l_{\alpha}(\mathbf{X}^{i}_{\sim Q}, \mathbf{X}^{i}_{\sim P} ; T)
        \le  \frac{1}{N}  \sum_{i=1}^N   l_{\alpha}(\mathbf{X}^{i}_{\sim Q}, \mathbf{X}^{i}_{\sim P} ; T^*).
    \end{equation}
    Therefore,
    \begin{equation}
        \inf_{T:\mathbb{R}^d \rightarrow \mathbb{R} }  \frac{1}{N}  \sum_{i=1}^N  l_{\alpha}(\mathbf{X}^{i}_{\sim Q}, \mathbf{X}^{i}_{\sim P} ; T)
        =  \frac{1}{N}  \sum_{i=1}^N   l_{\alpha}(\mathbf{X}^{i}_{\sim Q}, \mathbf{X}^{i}_{\sim P} ; T^*).
    \end{equation}
    From this, we see
    \begin{eqnarray}
        \hat{L}_{\alpha}^{(N)}(Q, P)
        &=& \inf_{T:\mathbb{R}^d \rightarrow \mathbb{R} }  \frac{1}{N}  \sum_{i=1}^N  l_{\alpha}(\mathbf{X}^{i}_{\sim Q}, \mathbf{X}^{i}_{\sim P} ; T) \nonumber \\
        &=&  \frac{1}{N}  \sum_{i=1}^N  l_{\alpha}(\mathbf{X}^{i}_{\sim Q}, \mathbf{X}^{i}_{\sim P}; T^*) \nonumber\\
        &=&  \frac{1}{N}  \sum_{i=1}^N  l_{\alpha}(\mathbf{X}^{i}_{\sim Q}, \mathbf{X}^{i}_{\sim P})     \nonumber\\
        &=&  l_{\alpha}(\mathbf{X}^{i}_{\sim Q}, \mathbf{X}^{i}_{\sim P}).
    \end{eqnarray}
    Subsequently, by integrating  both sides of the above equation, we have
    \begin{equation}
        E \Big[ \hat{L}_{\alpha}^{(N)}(Q, P) \Big]   = E \Big[ l_{\alpha}(\mathbf{X}^{i}_{\sim Q}, \mathbf{X}^{i}_{\sim P}) \Big]. \label{Lemma_loss_not_biased_lemma_N_for_final_state_left_2}
    \end{equation}
    Here, we have (\ref{Lemma_Eq_state_inf_t_l_t}) from (\ref{Lemma_loss_not_biased_lemma_N_for_final_state_left_1}) and (\ref{Lemma_loss_not_biased_lemma_N_for_final_state_left_2}).

    To see (\ref{Lemma_Eq_state_l_t}), note that, it holds that
    \begin{equation}
        \hat{L}_{\alpha}^{(N)}(Q, P;T) 
        = \frac{1}{N}  \sum_{i=1}^N  l_{\alpha}(\mathbf{X}^{i}_{\sim Q}, \mathbf{X}^{i}_{\sim P} ; T).
    \end{equation}

    By integrating both sides of the above equation, we have
    \begin{eqnarray}
        E \Big[ \hat{L}_{\alpha}^{(N)}(Q, P;T) \Big]   &=& E \Big[\frac{1}{N}  \sum_{i=1}^N  l_{\alpha}(\mathbf{X}^{i}_{\sim Q}, \mathbf{X}^{i}_{\sim P} ; T) \Big] \nonumber\\
        &=& \frac{1}{N}  \sum_{i=1}^N   E \Big[l_{\alpha}(\mathbf{X}^{i}_{\sim Q}, \mathbf{X}^{i}_{\sim P} ; T) \Big] \nonumber\\
        &=&  E \Big[l_{\alpha}(\mathbf{X}^{i}_{\sim Q}, \mathbf{X}^{i}_{\sim P} ; T) \Big] \nonumber\\
        &=& L_{\alpha}(Q, P ; T).
    \end{eqnarray}
    Here, we see that (\ref{Lemma_Eq_state_l_t}) holds.

    This completes the proof.
\end{proof}

\begin{proposition} \label{prop_gaddient_unbaised}
    Let $T_{\theta}(\mathbf{x}):\mathbb{R}^d \rightarrow \mathbb{R}$ be a function such that
    the map $\theta = (\theta_1,\theta_2, \ldots, \theta_p) \in \Theta \mapsto T_{\theta}(\mathbf{x})$ is differentiable for all $\theta$ and
    $\mu$-almost every $\mathbf{x} \in \mathbb{R}^d$.
    Assume that, for a point $\bar{\theta} \in \Theta$, it holds that
    $E_{P}[e^{(\alpha -1) \cdot T_{\bar{\theta}}(\mathbf{X})}] < \infty$ and $E_{Q}[e^{\alpha \cdot T_{\bar{\theta}}(\mathbf{X})}] < \infty$,
    and there exist a compact neighborhood of the $\bar{\theta}$,  which is denoted by $B_{\bar{\theta}}$,
    and a constant value $L$, such that $|T_{\psi}(\mathbf{x}) - T_{\bar{\theta}}(\mathbf{x})| < L \|\psi - \bar{\theta}\|$ holds.
    Then, for $l_{\alpha}(\mathbf{X}^{i}_{\sim Q}$, $\mathbf{X}^{i}_{\sim P} ; T)$ and  $\hat{L}_{\alpha}^{(N)}(Q, P;T), L_{\alpha}(Q,P;T)$ in Proposition \ref{corollary_desitination_alpha_div_est}, it holds that
    \begin{equation}
        E\left[ \nabla_{\theta} \,L_{\alpha}(\hat{Q}^{(N)},\hat{P}^{(N)};T_{\theta}) |_{\theta = \bar{\theta}} \right]
        = E\Big[ \nabla_{\theta} \, l_{\alpha}(\mathbf{X}^{i}_{\sim Q}, \mathbf{X}^{i}_{\sim P} ; T_{\theta}) |_{\theta = \bar{\theta}}\Big]
        =  \nabla_{\theta} \, L_{\alpha}(Q,P;T_{\theta})|_{\theta = \bar{\theta}}. \label{Prop_Eq_state}
    \end{equation}
    Here, $E[\ \cdot \ ]$ denotes $E_{P}[E_{Q}[\ \cdot \ ]]$.
\end{proposition}

\begin{proof}[proof of Proposition \ref{prop_gaddient_unbaised}]
    We now consider the values, as $\psi \rightarrow \bar{\theta}$, of the following two integrals:
    \begin{equation}
        \int  \frac{1}{\|\psi - \bar{\theta}\|} \left\{ \frac{1}{\alpha}  e^{\alpha \cdot T_{\psi}} - \frac{1}{\alpha} e^{\alpha \cdot T_{\bar{\theta}}}  \right\}  dQ,
        \label{proof_propos_grad_non_bias_q}
    \end{equation}
    and
    \begin{equation}
        \int  \frac{1}{\|\psi - \bar{\theta}\|} \left\{ \frac{1}{1-\alpha}  e^{(\alpha-1) \cdot T_{\psi}} - \frac{1}{1-\alpha} e^{(\alpha-1) \cdot T_{\bar{\theta}}}  \right\} dP.
        \label{proof_propos_grad_non_bias_p}
    \end{equation}

    Note that, it follows from the intermediate value theorem that
    \begin{equation}
        \left|  \frac{1}{\alpha}  e^{\alpha \cdot x} - \frac{1}{\alpha} e^{\alpha \cdot y} \right|
        = |x-y| \cdot  e^{\alpha \cdot \{ y + \gamma \cdot (x-y)\} }  \quad (\, \exists \gamma \in [0, 1] \,). \label{proof_propos_grad_non_bias_eq_intermediate_theorem}
    \end{equation}

    By using the above equation as $x=T_{\psi}(\mathbf{x})$ and $y=T_{\bar{\theta}}(\mathbf{x})$ for the integrand of (\ref{proof_propos_grad_non_bias_q}), we see
    \begin{eqnarray}
        \lefteqn{\left|  \frac{1}{\|\psi - \bar{\theta}\|} \left\{ \frac{1}{\alpha}  e^{\alpha \cdot T_{\psi}(\mathbf{x})}
            - \frac{1}{\alpha} e^{\alpha \cdot T_{\bar{\theta}}(\mathbf{x})}  \right\} \right|} \nonumber\\
        &=&   \frac{1}{\|\psi - \bar{\theta}\|}  |T_{\psi}(\mathbf{x}) - T_{\bar{\theta}} (\mathbf{x})|
        \cdot  e^{\alpha \cdot \{ T_{\bar{\theta}}(\mathbf{x}) +  \gamma_{\mathbf{x}}  \cdot (T_{\psi}(\mathbf{x}) - T_{\bar{\theta}}(\mathbf{x}))  \}}
        \quad \qquad (\, \gamma_{\mathbf{x}} \in [0, 1] \,)  \nonumber\\
        &=&   \frac{1}{\|\psi - \bar{\theta}\|}  |T_{\psi}(\mathbf{x}) - T_{\bar{\theta}}(\mathbf{x}) |
        \cdot e^{\alpha  \cdot  \gamma_{\mathbf{x}}  \cdot(T_{\psi}(\mathbf{x}) - T_{\bar{\theta}}(\mathbf{x}))}
        \cdot  e^{\alpha \cdot T_{\bar{\theta}}(\mathbf{x})}   \nonumber\\
        &\leq&   \frac{1}{\|\psi - \bar{\theta}\|}  |T_{\psi}(\mathbf{x}) - T_{\bar{\theta}}(\mathbf{x}) |
        \cdot e^{\alpha    \gamma_{\mathbf{x}}  |T_{\psi}(\mathbf{x}) - T_{\bar{\theta}}(\mathbf{x})|}
        \cdot  e^{\alpha \cdot T_{\bar{\theta}}(\mathbf{x})}   \nonumber\\
        &\leq&   L  \cdot e^{\alpha  L \cdot \| \psi - \bar{\theta} \|} \cdot  e^{\alpha \cdot T_{\bar{\theta}}(\mathbf{x})},  \label{proof_propos_grad_non_bias_to_get_upper_bound_q_1}
    \end{eqnarray}
    for all $\psi \in B_{\bar{\theta}}$.

    Thus, integrating the term on the left hand side of (\ref{proof_propos_grad_non_bias_to_get_upper_bound_q_1}) by $Q$, we see
    \begin{eqnarray}
        \lefteqn{\int \left|  \frac{1}{\|\psi - \bar{\theta}\|}
            \left\{ \frac{1}{\alpha}  e^{\alpha \cdot T_{\psi}(\mathbf{X})} - \frac{1}{\alpha}
            e^{\alpha \cdot T_{\bar{\theta}}(\mathbf{X})}  \right\} \right| dQ} \nonumber\\
        &\leq&  \int   L  \cdot e^{\alpha  L \cdot \| \psi - \bar{\theta} \|} \cdot  e^{\alpha \cdot T_{\bar{\theta}}(\mathbf{X})} dQ \nonumber\\
        &=&   L  \cdot e^{\alpha  L \cdot \| \psi - \bar{\theta} \|} E_Q\left[e^{\alpha \cdot T_{\bar{\theta}}} \right].    \label{proof_propos_grad_non_bias_to_get_upper_bound_q_2}
    \end{eqnarray}
    Considering the supremum for $\psi \in B_{\bar{\theta}}$ of (\ref{proof_propos_grad_non_bias_to_get_upper_bound_q_2}), it holds that
    \begin{eqnarray}
        \lefteqn{\sup_{\psi \in B_{\bar{\theta}}} \Bigg\{   \int \left|  \frac{1}{\|\psi - \bar{\theta}\|}
            \left\{ \frac{1}{\alpha}  e^{\alpha \cdot T_{\psi}} - \frac{1}{\alpha} e^{\alpha \cdot T_{\bar{\theta}}}  \right\} \right| dQ \Bigg\}} \nonumber\\
        &\leq& \sup_{\psi \in B_{\bar{\theta}}} \Big\{  L  \cdot e^{\alpha  L \cdot \| \psi - \bar{\theta} \|} \, E_Q\left[e^{\alpha \cdot T_{\bar{\theta}}} \right]  \Big\} \nonumber\\
        &=&  E_Q\left[e^{\alpha \cdot T_{\bar{\theta}}} \right]  \cdot \sup_{\psi \in B_{\bar{\theta}}} L  \cdot e^{\alpha  L \cdot \| \psi - \bar{\theta} \|} < \infty, \label{proof_propos_grad_non_bias_uniformaly_int_q}
    \end{eqnarray}
    since $ B_{\bar{\theta}}$ is compact.

    Similarly, as for (\ref{proof_propos_grad_non_bias_p}), we see
    \begin{eqnarray}
        \sup_{\psi \in B_{\bar{\theta}}} \lefteqn{\int \left|  \frac{1}{\|\psi - \bar{\theta}\|} \left\{ \frac{1}{1- \alpha}
            e^{(\alpha-1) \cdot T_{\psi}}(\mathbf{X}) - \frac{1}{1- \alpha} e^{(\alpha-1) \cdot T_{\bar{\theta}}(\mathbf{X})}  \right\} \right| dP} \nonumber\\
        &\leq& \sup_{\psi \in B_{\bar{\theta}}} \Big\{  L  \cdot e^{(1-\alpha)  L \cdot \| \psi - \bar{\theta} \|} \, E_P \left[e^{(1-\alpha) \cdot T_{\bar{\theta}}} \right]  \Big\} \nonumber\\
        &=&  E_P \left[e^{(1-\alpha) \cdot T_{\bar{\theta}}} \right]  \cdot \sup_{\psi \in B_{\bar{\theta}}} L \cdot e^{(1-\alpha)  L \cdot \| \psi - \bar{\theta} \|} < \infty.
        \label{proof_propos_grad_non_bias_uniformaly_int_p}
    \end{eqnarray}
    From (\ref{proof_propos_grad_non_bias_uniformaly_int_q}) and (\ref{proof_propos_grad_non_bias_uniformaly_int_p}), we obtain
    \begin{eqnarray}
        \sup_{\psi \in B_{\bar{\theta}}} \lefteqn{\int \int \left|  \frac{1}{\|\psi - \bar{\theta}\|} \Big\{
            l_{\alpha}(\mathbf{X}^{i}_{\sim Q}, \mathbf{X}^{i}_{\sim P} ; T_{\psi}) - l_{\alpha}(\mathbf{X}^{i}_{\sim Q}, \mathbf{X}^{i}_{\sim P} ; T_{\bar{\theta}})
            \Big\} \right| dP\,  dQ } \nonumber\\
        &=& \sup_{\psi \in B_{\bar{\theta}}}  \int  \int \left|
        \frac{1}{ \|\psi - \bar{\theta}\|} \left\{ \frac{1}{\alpha}
        e^{\alpha \cdot T_{\psi}(\mathbf{X}^{i}_{\sim Q} )}
        - \frac{1}{\alpha} e^{\alpha \cdot T_{\bar{\theta}}(\mathbf{X}^{i}_{\sim Q} )}  \right\} \right. \nonumber\\
        &&  \qquad \qquad \quad + \left. \frac{1}{\|\psi - \bar{\theta}\|} \left\{ \frac{1}{1- \alpha}
        e^{(\alpha-1) \cdot T_{\psi}(\mathbf{X}^{i}_{\sim P})} - \frac{1}{1- \alpha} e^{(\alpha-1) \cdot T_{\bar{\theta}}(\mathbf{X}^{i}_{\sim P})}  \right\}
        \right| dP\, dQ \nonumber\\
        &\leq& \sup_{\psi \in B_{\bar{\theta}}}  \int  \int \left|
        \frac{1}{ \|\psi - \bar{\theta}\|} \left\{ \frac{1}{\alpha}
        e^{\alpha \cdot T_{\psi}(\mathbf{X}^{i}_{\sim Q} )}
        - \frac{1}{\alpha} e^{\alpha \cdot T_{\bar{\theta}}(\mathbf{X}^{i}_{\sim Q} )}  \right\} \right| \nonumber\\
        &&  \qquad \quad \quad + \left|  \frac{1}{\|\psi - \bar{\theta}\|} \left\{ \frac{1}{1- \alpha}
        e^{(\alpha-1) \cdot T_{\psi}(\mathbf{X}^{i}_{\sim P})} - \frac{1}{1- \alpha} e^{(\alpha-1) \cdot T_{\bar{\theta}}(\mathbf{X}^{i}_{\sim P})}  \right\}
        \right| dP\, dQ \nonumber\\
        &=& \sup_{\psi \in B_{\bar{\theta}}} \left\{ \int   \left|
        \frac{1}{ \|\psi - \bar{\theta}\|} \left\{ \frac{1}{\alpha}
        e^{\alpha \cdot T_{\psi}(\mathbf{X} )}
        - \frac{1}{\alpha} e^{\alpha \cdot T_{\bar{\theta}}(\mathbf{X} )}  \right\} \right|  dQ \right. \nonumber\\
        &&  \qquad \quad \quad + \left. \int \left|  \frac{1}{\|\psi - \bar{\theta}\|} \left\{ \frac{1}{1- \alpha}
        e^{(\alpha-1) \cdot T_{\psi}(\mathbf{X})} - \frac{1}{1- \alpha} e^{(\alpha-1) \cdot T_{\bar{\theta}}(\mathbf{X})}  \right\}
        \right| dP \right\}  \nonumber\\
        &\leq& \sup_{\psi \in B_{\bar{\theta}}} \left\{ \int   \left|
        \frac{1}{ \|\psi - \bar{\theta}\|} \left\{ \frac{1}{\alpha}
        e^{\alpha \cdot T_{\psi}(\mathbf{X} )}
        - \frac{1}{\alpha} e^{\alpha \cdot T_{\bar{\theta}}(\mathbf{X} )}  \right\} \right|  dQ \right\}  \nonumber\\
        &&  \qquad   +  \sup_{\psi \in B_{\bar{\theta}}} \left\{ \int \left|  \frac{1}{\|\psi - \bar{\theta}\|} \left\{ \frac{1}{1- \alpha}
        e^{(\alpha-1) \cdot T_{\psi}(\mathbf{X})} - \frac{1}{1- \alpha} e^{(\alpha-1) \cdot T_{\bar{\theta}}(\mathbf{X})}  \right\}
        \right| dP \right\}  \nonumber\\
        &<& \infty. \label{proof_propos_grad_non_bias_uniformaly_all}
    \end{eqnarray}

    Therefore, the following set is uniformaly integrable for $\mu$:
    \begin{equation}
        \left\{ \frac{1}{\|\psi - \bar{\theta}\|} \Big\{
        l_{\alpha}(\mathbf{X}^{i}_{\sim Q}, \mathbf{X}^{i}_{\sim P} ; T_{\psi}) - l_{\alpha}(\mathbf{X}^{i}_{\sim Q}, \mathbf{X}^{i}_{\sim P} ; T_{\bar{\theta}})
        \Big\} \, : \, \psi \in B_{\bar{\theta}} \right\}.
    \end{equation}
    Then, from the property of the Lebesgue integral (\citealp{shiryaev1995probability}, P188, Theorem 4),
    the integral $\int\int (\cdot) \, dPdQ$ and the limitation $\lim_{\psi \rightarrow \bar{\theta}}\,$ for the above term are exchangeable.

    Hence, we have
    \begin{eqnarray}
        \lim_{\psi \rightarrow \bar{\theta}} \lefteqn{\int \int   \frac{1}{\|\psi - \bar{\theta}\|} \Big\{
            l_{\alpha}(\mathbf{X}^{i}_{\sim Q}, \mathbf{X}^{i}_{\sim P} ; T_{\psi}) - l_{\alpha}(\mathbf{X}^{i}_{\sim Q}, \mathbf{X}^{i}_{\sim P} ; T_{\bar{\theta}})
            \Big\}  dP\,  dQ } \nonumber\\
        &=&  \int  \int \lim_{\psi \rightarrow \bar{\theta}} \left[
        \frac{1}{\|\psi - \bar{\theta}\|} \Big\{
        l_{\alpha}(\mathbf{X}^{i}_{\sim Q}, \mathbf{X}^{i}_{\sim P} ; T_{\psi})
        - l_{\alpha}(\mathbf{X}^{i}_{\sim Q}, \mathbf{X}^{i}_{\sim P} ; T_{\bar{\theta}}) \Big\} \right]
        dP\, dQ  \nonumber\\
        &=&  \int  \int\nabla_{\theta} \, l_{\alpha}(\mathbf{X}^{i}_{\sim Q}, \mathbf{X}^{i}_{\sim P} ; T_{\theta})  |_{\theta = \bar{\theta}} \, dP\, dQ \nonumber\\
        &=&  E\Big[ \nabla_{\theta}  \, l_{\alpha}(\mathbf{X}^{i}_{\sim Q}, \mathbf{X}^{i}_{\sim P} ; T_{\theta}) |_{\theta = \bar{\theta}}
        \Big].  \label{proof_propos_grad_non_bias_excahgeable_lim_int_1}
    \end{eqnarray}
    On the other hand, for the term on the left hand side of (\ref{proof_propos_grad_non_bias_excahgeable_lim_int_1}), we obtain
    \begin{eqnarray}
        \lim_{\psi \rightarrow \bar{\theta}} \lefteqn{\int \int  \frac{1}{\|\psi - \bar{\theta}\|} \Big\{
            l_{\alpha}(\mathbf{X}^{i}_{\sim Q}, \mathbf{X}^{i}_{\sim P} ; T_{\psi}) - l_{\alpha}(\mathbf{X}^{i}_{\sim Q}, \mathbf{X}^{i}_{\sim P} ; T_{\bar{\theta}})
            \Big\}  dP\,  dQ }  \nonumber\\
        &=&  \lim_{\psi \rightarrow \bar{\theta}} \frac{1}{\|\psi - \bar{\theta}\|}  \lefteqn{\int \int   \Big\{
            l_{\alpha}(\mathbf{X}^{i}_{\sim Q}, \mathbf{X}^{i}_{\sim P} ; T_{\psi}) - l_{\alpha}(\mathbf{X}^{i}_{\sim Q}, \mathbf{X}^{i}_{\sim P} ; T_{\bar{\theta}})
            \Big\}  dP\,  dQ }  \nonumber\\ \nonumber \\
        &=& \lim_{\psi \rightarrow \bar{\theta}} \frac{1}{\|\psi - \bar{\theta}\|} \Big\{L_{\alpha}(Q,P;T_{\psi}) - L_{\alpha}(Q,P;T_{\bar{\theta}}) \Big\} \nonumber \\
        &=& \nabla_{\theta}  \,L_{\alpha}(Q,P;T_{\theta})|_{\theta = \bar{\theta}}.    \label{proof_propos_grad_non_bias_excahgeable_lim_int_2}
    \end{eqnarray}
    From (\ref{proof_propos_grad_non_bias_excahgeable_lim_int_1}) and (\ref{proof_propos_grad_non_bias_excahgeable_lim_int_2}), we obtain
    \begin{equation}
        E\left[\nabla_{\theta} \, l_{\alpha}(\mathbf{X}^{i}_{\sim Q}, \mathbf{X}^{i}_{\sim P} ; T_{\theta}) |_{\theta = \bar{\theta}} \right]
        = \nabla_{\theta} \,L_{\alpha}(Q,P;T_{\theta}) |_{\theta = \bar{\theta}} . \label{proof_propos_grad_non_bias_final_eq_1}
    \end{equation}
    From this, we also have
    \begin{eqnarray}
        E\left[\nabla_{\theta} \, l_{\alpha}(\hat{Q}^{(N)},\hat{P}^{(N)};T_{\theta}) |_{\theta = \bar{\theta}} \right]
        &=&  E\left[\nabla_{\theta}|_{\theta = \bar{\theta}}\, \frac{1}{N}  \sum_{i=1}^N  l_{\alpha}(\mathbf{X}^{i}_{\sim Q}, \mathbf{X}^{i}_{\sim P} ; T_{\theta}) \right] \nonumber \\
        &=&  E\left[ \frac{1}{N}  \sum_{i=1}^N \nabla_{\theta} \, l_{\alpha}(\mathbf{X}^{i}_{\sim Q}, \mathbf{X}^{i}_{\sim P} ; T_{\theta}) |_{\theta = \bar{\theta}} \right] \nonumber \\
        &=& \frac{1}{N}  \sum_{i=1}^N   E\Big[ \nabla_{\theta} \, l_{\alpha}(\mathbf{X}^{i}_{\sim Q}, \mathbf{X}^{i}_{\sim P} ; T_{\theta}) |_{\theta = \bar{\theta}} \Big] \nonumber \\
        &=& \frac{1}{N}  \sum_{i=1}^N \nabla_{\theta} \,L_{\alpha}(Q,P;T_{\theta}) |_{\theta = \bar{\theta}} \nonumber \\
        &=&\nabla_{\theta} \,L_{\alpha}(Q,P;T_{\theta}) |_{\theta = \bar{\theta}}. \label{proof_propos_grad_non_bias_final_eq_2}
    \end{eqnarray}
    Here, we see (\ref{Prop_Eq_state}) from (\ref{proof_propos_grad_non_bias_final_eq_1}) and (\ref{proof_propos_grad_non_bias_final_eq_2}).

    This completes the proof.
\end{proof}

\begin{proposition}\label{prop_consitency_alpha_div_est}
    Let
    \begin{equation}
        \hat{D}_{\alpha}^{(N)} (Q||P) = \sup_{T:\mathbb{R}^d \rightarrow \mathbb{R}} \left[
        \frac{1}{\alpha(1-\alpha)} 	- \frac{1}{\alpha} \left\{ \frac{1}{N} \sum_{i=1}^N e^{\alpha \cdot T(\mathbf{X}^{i}_{\sim Q} )}  \right\}
        -  \frac{1}{1- \alpha} \left\{  \frac{1}{N}
        \sum_{i=1}^N  e^{(\alpha - 1) \cdot T(\mathbf{X}^{i}_{\sim P})} \right\}
        \right], \label{lemma_def_alpha_div_est}
    \end{equation}
    where supremum  is considered over all  measurable functions $T:\mathbb{R}^d \rightarrow \mathbb{R}$ with
    $E_P[e^{(\alpha - 1) \cdot T}] < \infty$ and $E_Q[e^{\alpha \cdot T}] < \infty$.

    Then, it holds that
    if $\alpha\neq 1/2$,
    \begin{eqnarray}
        \lefteqn{\sqrt{N} \left\{ \hat{D}_{\alpha}^{(N)} (Q||P) - D_{\alpha} (Q||P) \right\}} \nonumber\\
        &\xrightarrow{\ \  d \ \ }& \mathcal{N}   \Big(\ 0,\  C^1_{\alpha} \cdot D_{2 \alpha -1}(Q||P)
        +  C^2_{\alpha} \cdot  D_{\alpha}(Q||P) +  C^3_{\alpha} \cdot  D_{\alpha}(Q||P)^2 \ \Big),
    \end{eqnarray}
    where
    \begin{eqnarray}
        C^1_{\alpha} &=& \left( \frac{1}{\alpha^2} + \frac{1}{(1-\alpha)^2} \right)\cdot(2\alpha-1)\cdot(2\alpha-2), \\
        C^2_{\alpha} &=&  \frac{2\{ \alpha^2 + (1-\alpha)^2 \}}{\alpha \cdot (1-\alpha)} \quad \text{and} \quad  C^3_{\alpha} = - \alpha^2 - (1 - \alpha)^2,
    \end{eqnarray}
    and if $\alpha = 1/2$,
    \begin{eqnarray}
        \lefteqn{\sqrt{N} \left\{ \hat{D}_{\alpha}^{(N)} (Q||P) - D_{\alpha} (Q||P) \right\}} \nonumber\\
        &\xrightarrow{\ \  d \ \ }& \mathcal{N}   \Big(\ 0,  \  4\, D_{\alpha}(Q||P) - \frac{1}{2} \, D_{\alpha}(Q||P)^2 \ \Big). \label{prop_consitency_alpha_div_alpha_half}
    \end{eqnarray}
\end{proposition}

\begin{proof}[proof of Proposition \ref{prop_consitency_alpha_div_est}]
    First, we note that
    \begin{eqnarray}
        \hat{D}_{\alpha}^{(N)} (Q||P) &=& \sup_{T:\mathbb{R}^d \rightarrow \mathbb{R}} \left[
        \frac{1}{\alpha(1-\alpha)} 	- \frac{1}{\alpha} \left\{ \frac{1}{N} \sum_{i=1}^N e^{\alpha \cdot T(\mathbf{X}^{i}_{\sim Q} )}  \right\}
        -  \frac{1}{1- \alpha} \left\{  \frac{1}{N}
        \sum_{i=1}^N  e^{(\alpha - 1) \cdot T(\mathbf{X}^{i}_{\sim P})} \right\}
        \right] \nonumber\\
        &=&
        \frac{1}{\alpha(1-\alpha)}
        - \inf_{T:\mathbb{R}^d \rightarrow \mathbb{R}} \left[ \frac{1}{\alpha} \left\{ \frac{1}{N} \sum_{i=1}^N e^{\alpha \cdot T(\mathbf{X}^{i}_{\sim Q} )}  \right\}
        +  \frac{1}{1- \alpha} \left\{  \frac{1}{N}
        \sum_{i=1}^N  e^{(\alpha - 1) \cdot T(\mathbf{X}^{i}_{\sim P})} \right\} \right] \nonumber \\
        &=&
        \frac{1}{\alpha(1-\alpha)}
        - \inf_{T:\mathbb{R}^d \rightarrow \mathbb{R}} \frac{1}{N}  \sum_{i=1}^N
        \left[ \frac{1}{\alpha}  e^{\alpha \cdot T(\mathbf{X}^{i}_{\sim Q} )}
        +  \frac{1}{1- \alpha}     e^{(\alpha - 1) \cdot T(\mathbf{X}^{i}_{\sim P})}  \right] \nonumber\\
        &=&  \frac{1}{\alpha(1-\alpha)}
        - \inf_{T:\mathbb{R}^d \rightarrow \mathbb{R}} \frac{1}{N}  \sum_{i=1}^N
        \left[ l_{\alpha}(\mathbf{X}^{i}_{\sim Q}, \mathbf{X}^{i}_{\sim P} ; T) \right] \nonumber\\
        &=&  \frac{1}{\alpha(1-\alpha)}
        - \frac{1}{N} \sum_{i=1}^N  l_{\alpha}(\mathbf{X}^{i}_{\sim Q}, \mathbf{X}^{i}_{\sim P}). \label{D_alpha_hat_represented_as_sum_iids}
    \end{eqnarray}

    On the other hand, from Lemma \ref{lemma_loss_not_biased_lemma}, it holds that
    \begin{eqnarray}
        D_{\alpha} (Q||P) &=&  \sup_{T:\mathbb{R}^d \rightarrow \mathbb{R}} \left\{
        \frac{1}{\alpha(1-\alpha)} - \frac{1}{\alpha} E_{Q} \left[e^{\alpha \cdot T} \right]
        - \frac{1}{1- \alpha} E_{P} \left[ e^{(\alpha - 1) \cdot T} \right] 	\right\} \nonumber\\
        &=&  \frac{1}{\alpha(1-\alpha)} -
        \inf_{T:\mathbb{R}^d \rightarrow \mathbb{R}} \left\{ \frac{1}{\alpha} E_{Q} \left[e^{\alpha \cdot T} \right]
        - \frac{1}{1- \alpha} E_{P} \left[ e^{(\alpha - 1) \cdot T} \right] \right\} \nonumber\\
        &=& \frac{1}{\alpha(1-\alpha)}  - \frac{1}{N} \sum_{i=1}^N L_{\alpha}(Q,P) \nonumber\\
        &=& \frac{1}{\alpha(1-\alpha)}  - \frac{1}{N} \sum_{i=1}^N  E\left[ l_{\alpha}(\mathbf{X}^{i}_{\sim Q}, \mathbf{X}^{i}_{\sim P}) \right]. \label{D_alpha_represented_as_exp}
    \end{eqnarray}

    Subtracting (\ref{D_alpha_represented_as_exp}) from (\ref{D_alpha_hat_represented_as_sum_iids}), we have
    \begin{equation}
        \hat{D}_{\alpha}^{(N)} (Q||P) -  D_{\alpha} (Q||P) = \frac{1}{N} \sum_{i=1}^N
        \left\{ l_{\alpha}(\mathbf{X}^{i}_{\sim Q}, \mathbf{X}^{i}_{\sim P}) -   E\left[ l_{\alpha}(\mathbf{X}^{i}_{\sim Q}, \mathbf{X}^{i}_{\sim P}) \right] \right\}.
    \end{equation}
    Let $L_i=l_{\alpha}(\mathbf{X}^{i}_{\sim Q}, \mathbf{X}^{i}_{\sim P}) -   E\left[ l_{\alpha}(\mathbf{X}^{i}_{\sim Q}, \mathbf{X}^{i}_{\sim P}) \right]$.
    Then $\{L_i\}_{i=1}^N$ are independently identically distributed variables whose means and variances are as follows:
    \begin{eqnarray}
        E\left[ L_i \right] &=& 0, \\
        \mathrm{Var} \left[L_i  \right]
        &=& E \left[
        \left\{
        l_{\alpha}(\mathbf{X}^{i}_{\sim Q}, \mathbf{X}^{i}_{\sim P})
        -   E\left[ l_{\alpha}(\mathbf{X}^{i}_{\sim Q}, \mathbf{X}^{i}_{\sim P}) \right]
        \right\}^2
        \right] \nonumber\\
        &=& E_{P} \left[ \ E_{Q} \left[\
        \frac{1}{\alpha} \left\{ \left( \frac{dQ}{dP} \right)^{-\alpha} (\mathbf{X}^{i}_{\sim Q} )
        -   E_{Q} \left[ \left( \frac{dQ}{dP} \right)^{-\alpha} \right]  \right\} \right. \right. \nonumber \\
        && \qquad
        + \left. \left.   \frac{1}{1- \alpha} \left\{\left( \frac{dQ}{dP} \right)^{1 - \alpha} (\mathbf{X}^{i}_{\sim P})
        - E_{P} \left[ \left( \frac{dQ}{dP} \right)^{1 - \alpha} \  \right]    \right\}
        \ \right]^2  \right] \nonumber \\
        &=& \frac{1}{\alpha^2} \cdot E_{Q} \left\{ \left( \frac{dQ}{dP} \right)^{-\alpha} (\mathbf{X}^{i} )
        -   E_{Q} \left[ \left( \frac{dQ}{dP} \right)^{-\alpha} \ \right] \right\}^2 \nonumber \\
        && \qquad  + \frac{1}{(1- \alpha)^2} \cdot E_{P} \left\{ \left( \frac{dQ}{dP} \right)^{1 - \alpha} (\mathbf{X}^{i})
        - E_{P} \left[  \left( \frac{dQ}{dP} \right)^{1 - \alpha}  \right] \right\}^2 \nonumber \\
        &=& \frac{1}{\alpha^2} \cdot E_{P} \left\{ \frac{dQ}{dP} \cdot  \left( \frac{dQ}{dP} \right)^{-\alpha} (\mathbf{X}^{i} )
        -   E_{P} \left[  \frac{dQ}{dP} \cdot \left( \frac{dQ}{dP} \right)^{-\alpha} \ \right] \right\}^2 \nonumber \\
        && \qquad  + \frac{1}{(1- \alpha)^2} \cdot E_{P} \left\{ \left( \frac{dQ}{dP} \right)^{1 - \alpha} (\mathbf{X}^{i})
        - E_{P} \left[  \left( \frac{dQ}{dP} \right)^{1 - \alpha}  \right] \right\}^2 \nonumber \\
        &=& \left\{ \frac{1}{\alpha^2} +  \frac{1}{(1- \alpha)^2}   \right\}
        \cdot E_{P}\left\{\left(\frac{dQ}{dP} \right)^{1 - \alpha} - E_{P} \left[\left(\frac{dQ}{dP} \right)^{1 - \alpha} \right] \right\}^2 \nonumber \\
        &=&  
        \left\{ \frac{1}{\alpha^2} +  \frac{1}{(1- \alpha)^2} \right\}
        \Biggr\{  E_{P} \left[ \left(\frac{dQ}{dP} \right)^{2\cdot(1 - \alpha)}  \right]
        - \left\{  E_{P} \left[\left(\frac{dQ}{dP} \right)^{1 - \alpha} \right]  \right\}^2 \Biggr\}  \nonumber\\
        &=& \left\{ \frac{1}{\alpha^2} + \frac{1}{(1- \alpha)^2} \right\}\nonumber\\
        &&  \  \times   \Bigg[ \  (2\alpha-1)(2\alpha -2) \left\{ \frac{1}{ (2\alpha-1)(2\alpha -2)}
        E_{P} \left[ \left(\frac{dQ}{dP} \right)^{1 - (2\alpha-1)}  - 1 \right]\right\}  \ + \  1 \nonumber\\
        && \quad \quad - \alpha^2(1 - \alpha)^2 \left\{
        \frac{1}{\alpha\cdot (\alpha - 1)}
        E_{P} \left[ \left(\frac{dQ}{dP} \right)^{1 - \alpha}  - 1 \  \right]
        \right\}^2 \nonumber\\
        && \quad \quad + \alpha^2(1 - \alpha)^2 \left\{
        \frac{2}{\alpha\cdot (\alpha - 1)}
        E_{P} \left[ \left(\frac{dQ}{dP} \right)^{1 - \alpha}  - 1 \  \right]
        \right\}  \quad - \  1 \  \Bigg].
    \end{eqnarray}
    From this,  if $\alpha\neq 1/2$, we have
    \begin{eqnarray}
        \mathrm{Var} \left[L_i  \right]  &=&
        C^1_{\alpha} \cdot D_{2 \alpha -1}(Q||P)
        +  C^2_{\alpha} \cdot  D_{\alpha}(Q||P) +  C^3_{\alpha} \cdot  D_{\alpha}(Q||P)^2, \label{eq_variance_alpha_not_half}
    \end{eqnarray}
    where
    \begin{eqnarray}
        C^1_{\alpha} &=& \left( \frac{1}{\alpha^2} + \frac{1}{(1-\alpha)^2} \right)\cdot(2\alpha-1)\cdot(2\alpha-2), \nonumber \\
        C^2_{\alpha} &=&  \frac{2\{ \alpha^2 + (1-\alpha)^2 \}}{\alpha \cdot (1-\alpha)} \quad \text{and} \quad  C^3_{\alpha} = - \alpha^2 - (1 - \alpha)^2, \nonumber
    \end{eqnarray}
    and if $\alpha = 1/2$, we obtain
    \begin{eqnarray}
        \mathrm{Var} \left[L_i  \right]   &=& 4\, D_{\alpha}(Q||P) - \frac{1}{2} \, D_{\alpha}(Q||P)^2.
        \label{eq_variance_alpha_is_half}
    \end{eqnarray}

    Therefore, by the central limit theorem, we see
    \begin{equation}
        \frac{1}{N} \sum_{i=1}^N
        \left\{ l_{\alpha}(\mathbf{X}^{i}_{\sim Q}, \mathbf{X}^{i}_{\sim P}) -   E\left[ l_{\alpha}(\mathbf{X}^{i}_{\sim Q}, \mathbf{X}^{i}_{\sim P}) \right] \right\}
        \xrightarrow{\ \  d \ \ } \mathcal{N} \left(0, \ast \right).
    \end{equation}
    Here the “$\ast$” is (\ref{eq_variance_alpha_not_half}) or (\ref{eq_variance_alpha_is_half}),  which corresponds to the cases that $\alpha\neq 1/2$ or $\alpha=1/2$, respectively.

    This completes the proof.
\end{proof}

We mention that the statement of the following corollary is the same as Corollary 1 in \citeauthor{birrell2022optimizing}(\citeyear{birrell2022optimizing}).
\begin{corollary}[\citeauthor{birrell2022optimizing}(\citeyear{birrell2022optimizing}), P19, Corollary 1]\label{col_alpha_half_sample_requirement}
    For $\alpha=1/2$, it holds that
    \begin{equation}
        \lim_{N \rightarrow \infty} \frac{N \cdot  \mathrm{Var} \Big[\  \hat{D}_{1/2}^{(N)} (Q||P)     \Big]  }{D_{1/2} (Q||P)^2} = \frac{8\, D_{1/2}(Q||P) - D_{1/2} (Q||P)^2}{2\, D_{1/2} (Q||P)^2}.
    \end{equation}
    Thus, the sample complexity of $D_{\alpha}$ for $\alpha=1/2$ is $O(1)$.
\end{corollary}

\begin{proof}[proof of Corollary \ref{col_alpha_half_sample_requirement}]
    The statement of the corollary follows from (\ref{prop_consitency_alpha_div_alpha_half}) in Proposition \ref{prop_consitency_alpha_div_est}.
\end{proof}

\begin{proposition}\label{proposition_desitination_alpha_div_est}
    Let $\{T_k\}_{k=1}^{\infty}$ be a sequence of functions in $\mathcal{T}^{\alpha}$ with $E_{P}[e^{-T_k(\mathbf{X})}] = 1$ such that $\lim_{k \rightarrow \infty} T_k= - \log dQ/dP$, $P$-almost everywhere.
    Subsequently, let $\{\mathbf{X}^{Q}_k \}_{k=1}^{\infty}$ be a sequence of variables on $\mathbb{R}^d$ defiened as follows:
    \begin{equation}
        \mathbf{X}^{Q}_k = e^{-T_k(\mathbf{X}^{P})} \cdot \mathbf{X}^{P}. \nonumber
    \end{equation}
    Then, it holds that
    \begin{equation}
        \mathbf{X}^{Q}_k \xrightarrow{\ \  \text{d} \ \ } \mathbf{X}^{Q}, \quad \text{as} \quad k \longrightarrow \infty.  \label{Eq_proof_of_Proposition_desitination_alpha_div_est_valanced_val_converges_in_dist}
    \end{equation}
\end{proposition}

\begin{proof}[proof of Proposition \ref{proposition_desitination_alpha_div_est}]
    Let $Q_k$ denote the probability distribution of $\mathbf{X}^{Q}_k$: $Q_k(A) = P(\mathbf{X}^{Q}_k \in A)$ for all $A \in \mathscr{F}$.
    Then,  since  $\frac{dQ_k}{dP} =  e^{-T_k(\mathbf{X})}$, we see
    \begin{equation}
        \frac{dQ_k}{dQ} =  e^{-T_k(\mathbf{X})} \cdot \frac{dP}{dQ}.
    \end{equation}

    Now, from Corollary 6 in \cite{gilardoni2010pinsker}, for probalility measures $A$ and $B$ with $A \ll \mu$ and $B \ll \mu$,  it holds that
    \begin{equation}
        \frac{1}{2} \left\{ E_{\mu}\left|\frac{dA}{d\mu} - \frac{dB}{d\mu} \right| \right\}^2 \le  D_{\alpha} (A||B). \label{Eq_proof_of_Proposition_desitination_alpha_div_est_total_variation_and_alpha_div}
    \end{equation}

    By substituting $A=Q_k$ and $B=Q$ into (\ref{Eq_proof_of_Proposition_desitination_alpha_div_est_total_variation_and_alpha_div}), we have
    \begin{eqnarray}
        \frac{1}{2} \left\{ E_{\mu}\left|\frac{dQ_k}{d\mu} - \frac{dQ}{d\mu} \right| \right\}^2 &\le&  D_{\alpha} (Q_k||Q) \nonumber \\
        &=& \int \frac{1}{\alpha(\alpha-1)} \left\{ \left( \frac{dQ_k}{dQ}\right)^{1 - \alpha} - 1 \right\} dQ  \nonumber \\
        &=& \int \frac{1}{\alpha(\alpha-1)} \int \left( \frac{dQ_k}{dQ}\right)^{1 - \alpha} dQ - \frac{1}{\alpha(\alpha - 1)}   \nonumber \\
        &=& \frac{1}{\alpha(\alpha - 1)} \int \left( \frac{dP}{dQ}\right)^{1 - \alpha} \cdot e^{(\alpha - 1) \cdot T_k} \frac{dQ}{dP} dP   - \frac{1}{\alpha(\alpha - 1)} \nonumber\\
        &=& \frac{1}{\alpha(\alpha - 1)} \int \left( \frac{dQ}{dP}\right)^{\alpha}  \cdot e^{(\alpha - 1) \cdot T_k} dP -  \frac{1}{\alpha(\alpha - 1)}. \label{Eq_proof_of_Proposition_desitination_alpha_div_est_total_upper_ineq_variation}
    \end{eqnarray}

    Now, from H\"{o}lder's inequality, we have
    \begin{eqnarray}
        \int \left| \left( \frac{dQ}{dP}\right)^{\alpha}  \cdot e^{(\alpha - 1) \cdot T_k} \right|  dP
        &\le&  \left\{\int \left( \frac{dP}{dQ}\right)^{\frac{- \alpha}{\alpha} } dP  \right\}^{\alpha}  \nonumber \\
        && \quad  \times \  \left\{ \int \left( e^{(\alpha - 1) \cdot T_k} \right)^{\frac{1 }{1 - \alpha} }  dP \right\}^{1 - \alpha}  \nonumber \\
        &=& \left( E_P \left[ e^{- T_k}  \right]  \right)^{1 - \alpha} \nonumber\\
        &=& 1 < \infty.
    \end{eqnarray}
    Hence, we see the following sequence is uniformaly integrable for $P$:
    \begin{equation}
        \left\{  \left( \frac{dQ}{dP}\right)^{\alpha}  \cdot e^{(\alpha - 1) \cdot T_k} \right\}_{k=1}^N.
    \end{equation}

    Then, for (\ref{Eq_proof_of_Proposition_desitination_alpha_div_est_total_upper_ineq_variation}) as $k \rightarrow \infty$, we see
    \begin{eqnarray}
       \lim_{k \rightarrow \infty}  \frac{1}{2} \left\{ E_{\mu}\left|\frac{dQ_k}{d\mu} - \frac{dQ}{d\mu} \right| \right\}^2 &\le& \lim_{k \rightarrow \infty}  \frac{1}{\alpha(\alpha - 1)} \int \left( \frac{dQ}{dP}\right)^{\alpha}  \cdot e^{(\alpha - 1) \cdot T_k} dP -  \frac{1}{\alpha(\alpha - 1)} \nonumber\\
       &=& \frac{1}{\alpha(\alpha - 1)} \int \left( \frac{dQ}{dP}\right)^{\alpha}  \cdot \left( \frac{dQ}{dP}\right)^{1 - \alpha} dP -  \frac{1}{\alpha(\alpha - 1)}  \nonumber\\
       &=& \frac{1}{\alpha(\alpha - 1)} \int \frac{dQ}{dP} dP -  \frac{1}{\alpha(\alpha - 1)}  \nonumber\\
       &=& 0.  \nonumber
    \end{eqnarray}
    Thus, $Q_k$ converges to $Q$ in total variation.

    The statement (\ref{Eq_proof_of_Proposition_desitination_alpha_div_est_valanced_val_converges_in_dist}), the convergence of $\hat{\mathbf{X}}^{Q}_k$ to $\mathbf{X}^{Q}$ in distribution, is derived from the convergence of $Q_k$ to $Q$ in total variation.

    This completes the proof.
\end{proof}

\begin{corollary}\label{corollary_desitination_alpha_div_est}
    Let $\{T_k\}_{k=1}^{\infty}$ be a sequence of functions in $\mathcal{T}^{\alpha}$
    such that
    \begin{eqnarray}
        \lefteqn{D_{\alpha}(\hat{Q}^{(N)}||\hat{P}^{(N)})} \nonumber \\
        &=& \lim_{k \rightarrow \infty}  \left\{ \frac{1}{\alpha(1-\alpha)} - \frac{1}{\alpha} \frac{1}{N} \sum_{i=1}^N e^{\alpha \cdot T_k(\mathbf{X}^{i}_{\sim Q} )} +  \frac{1}{1- \alpha} \frac{1}{N} \sum_{i=1}^N   e^{(\alpha - 1) \cdot T_k(\mathbf{X}^{i}_{\sim P})} \right\}.
    \end{eqnarray}
    Subsequently, let $\{\hat{\mathbf{X}}_{\mathbb{Q}}^{(N)}(k) \}_{k=1}^{\infty}$ be a sequence of variables on $\mathbb{R}^d$ defiened as follows:
    \begin{equation}
        \hat{\mathbf{X}}_{\mathbb{Q}}^{(N)}(k) = e^{-T_{k}} \cdot \mathbf{X}_{\mathbb{P}}^{(N)}
    \end{equation}
    Then, it holds that, as $k \longrightarrow \infty$,
    \begin{equation}
        \hat{\mathbf{X}}_{\mathbb{Q}}^{(N)}(k) \xrightarrow{\ \  \text{d} \ \ } \hat{\mathbf{X}}_{\mathbb{Q}}^{(N)}.
    \end{equation}
\end{corollary}

\begin{proof}[proof of Corollary \ref{corollary_desitination_alpha_div_est}]
    Let $\nu$ be the countable measure on $\left\{\mathbf{X}_1, \mathbf{X}_2, \ldots, \mathbf{X}_N\right\}$:
    \begin{equation}
        \nu(\mathbf{x}) = \left\{
        \begin{array}{ll}
            1  &  \text{if} \ \   1 \le \exists i \le N \ \text{s.t.} \  \mathbf{X}_i = \mathbf{x}, \\
            0 &\text{otherwise.}
        \end{array}
        \right.
    \end{equation}
    Then, $\hat{P}^{(N)} \ll \nu$ and $\hat{Q}^{(N)} \ll \nu$.

    For Proposition \ref{proposition_desitination_alpha_div_est} and its proof, substituting $\hat{P}^{(N)}$ for $P$, $\hat{Q}^{(N)}$ for $Q$, and $\nu$ for $\mu$,
    we see that the statement of the corollary holds.

    This completes the proof.
\end{proof}

\begin{proposition}\label{proposition_variational_upper_bound_by_alpha_div}
    For $\hat{T} \in \mathcal{T}^{\alpha}$, let $\hat{Q}$ and $\hat{P}$ be two probalities defined as
    \begin{equation}
        d\hat{Q} =  e^{-\hat{T}} \cdot dP \quad \text{and} \quad d\hat{P} = e^{\hat{T}} \cdot dQ, \nonumber
    \end{equation}
    and let $T_* = - \log dQ/dP$.

    Then, it hols that
    \begin{eqnarray}
        \lefteqn{\frac{\alpha}{2} \left\{ E_{\mu}\left|\hat{Q} - Q \right| \right\}^2  + \frac{1- \alpha}{2} \left\{ E_{\mu}\left|\hat{P} - P \right| \right\}^2} \nonumber \\
        &\le&  \frac{1}{\alpha(1-\alpha)}
        - \frac{1}{ \alpha} E_Q\left[ e^{\alpha \cdot(T_* - \hat{T})} \right]
        - \frac{1}{1 - \alpha} E_P\left[ e^{(\alpha - 1) \cdot (T_* - \hat{T})} \right]  \label{Eq_lemma_variational_upper_bound_by_alpha_div_final_state} \\
        &=& L_{\alpha}(Q, P; \hat{T}) -  L_{\alpha}(Q, P;  T_*). \label{Eq_lemma_variational_upper_bound_by_alpha_div_final_state_2}
    \end{eqnarray}
    Here, $L_{\alpha}(Q, P; \cdot)$ in (\ref{Eq_lemma_variational_upper_bound_by_alpha_div_final_state_2}) is defined as (\ref{Eq_lemma_loss_not_biased_lemma_loss_T}) in Lemma \ref{lemma_loss_not_biased_lemma}
\end{proposition}

\begin{proof}[proof of Proposition \ref{proposition_variational_upper_bound_by_alpha_div}]
    First, we see (\ref{Eq_lemma_variational_upper_bound_by_alpha_div_final_state}).
    Note that, it holds that
    \begin{equation}
        \frac{\hat{dQ}}{dQ} =  e^{-\hat{T}} \cdot \frac{dP}{dQ}     \qquad \text{and} \qquad \frac{d\hat{P}}{dP} = e^{\hat{T}} \cdot \frac{dQ}{dP}.
    \end{equation}
    By using (\ref{Eq_proof_of_Proposition_desitination_alpha_div_est_total_variation_and_alpha_div}), we have
    \begin{equation}
        \frac{\alpha}{2} \left\{ E_{\mu}\left|\hat{Q} - Q \right| \right\}^2  + \frac{1- \alpha}{2} \left\{ E_{\mu}\left|\hat{P} - P \right| \right\}^2
        \le  (1 - \alpha) \cdot D_{1 - \alpha}(\hat{Q}|| Q) +  \alpha \cdot D_{\alpha}(\hat{P}|| P). \nonumber
    \end{equation}
    Thus, we obtain
    \begin{eqnarray}
        \lefteqn{\frac{\alpha}{2} \left\{ E_{\mu}\left|\hat{Q} - Q \right| \right\}^2  + \frac{1- \alpha}{2} \left\{ E_{\mu}\left|\hat{P} - P \right| \right\}^2} \nonumber \\
        &\le&  (1 - \alpha) \cdot D_{1 - \alpha}(\hat{Q}|| Q) +  \alpha \cdot D_{\alpha}(\hat{P}|| P) \nonumber\\
        &=& \frac{1 - \alpha}{\alpha(\alpha - 1)} \left\{  \int \left( \frac{d\hat{Q}}{dQ} \right)^{ \alpha}   dQ - 1\right\} \nonumber\\
        &&  \qquad + \ \frac{\alpha}{\alpha(\alpha - 1)} \left\{  \int \left( \frac{d\hat{P}}{dP} \right)^{1 - \alpha}   dP - 1\right\} \nonumber\\
        &=& -  \ \frac{1}{\alpha} \left\{  \int  e^{- \alpha \cdot \hat{T}}  \left( \frac{dP}{dQ} \right)^{\alpha} dQ - 1\right\} \nonumber\\
        &&  \qquad - \  \frac{1}{1 - \alpha} \left\{  \int e^{(1-\alpha) \cdot  \hat{T}} \left( \frac{dQ}{dP} \right)^{1 - \alpha}  dP - 1\right\} \nonumber\\
        &=& - \ \frac{1}{ \alpha} \left\{  \int  e^{- \alpha \cdot \hat{T}} e^{\alpha\cdot T_* }  dQ - 1\right\} \nonumber\\
        &&  \qquad -  \ \frac{1}{1 - \alpha} \left\{  \int e^{ - (\alpha - 1) \cdot  \hat{T}} e^{(\alpha - 1) \cdot T_* }  dP - 1\right\}  \nonumber\\
        &=& \frac{1}{\alpha(1 - \alpha)} - \frac{1}{\alpha} \int  e^{\alpha \cdot (T_* - \hat{T})}  dQ -   \frac{1}{1 - \alpha}  \int e^{(\alpha - 1) \cdot  (T_* - \hat{T})}  dP.
    \end{eqnarray}
    Here, we see (\ref{Eq_lemma_variational_upper_bound_by_alpha_div_final_state}).

    To obtain (\ref{Eq_lemma_variational_upper_bound_by_alpha_div_final_state_2}), we have
    \begin{eqnarray}
        \lefteqn{\frac{1}{\alpha(1 - \alpha)} - \frac{1}{\alpha} \int  e^{\alpha \cdot (T_* - \hat{T})}  dQ -   \frac{1}{1 - \alpha}  \int e^{(\alpha - 1) \cdot  (T_* - \hat{T})}  dP}  \nonumber\\
        &=& \frac{1}{\alpha} +  \frac{1}{1 - \alpha}   - \frac{1}{\alpha} \int  e^{\alpha \cdot (T_* - \hat{T})}  dQ -   \frac{1}{1 - \alpha}  \int e^{(\alpha - 1) \cdot  (T_* - \hat{T})}  dP \nonumber\\
        &=&   \int \left\{   \frac{1}{\alpha} \cdot \frac{dQ}{d\mu} - \frac{1}{\alpha} \cdot e^{\alpha \cdot (T_* - \hat{T})}   \frac{dQ}{d\mu}  \right\}  d\mu   \nonumber\\
        &&  \qquad \qquad +  \int  \left\{\frac{1}{1 - \alpha} \cdot \frac{dP}{d\mu}
        - \frac{1}{1 - \alpha} \cdot e^{(\alpha - 1) \cdot  (T_* - \hat{T})} \frac{dP}{d\mu}  \right\}  d\mu. \label{proof_corollary_desitination_alpha_div_est_to_change_measure}
    \end{eqnarray}
    By replacing the measures $\mu$ of the two integrals in (\ref{proof_corollary_desitination_alpha_div_est_to_change_measure}) with
    \begin{equation}
        \nu  =   e^{ - \alpha \cdot \hat{T}} d\mu    \quad \text{and} \quad \tau  =   e^{ - (\alpha-1) \cdot \hat{T}} d\mu,
    \end{equation}
    respectively, we obtain
    \begin{eqnarray}
        && \int \left\{   \frac{1}{\alpha} \cdot \frac{dQ}{d\nu } - \frac{1}{\alpha} \cdot e^{\alpha \cdot (T_* - \hat{T})}   \frac{dQ}{d\nu}  \right\} e^{\alpha \cdot \hat{T}}  \cdot d\nu  \nonumber\\
        && \qquad \qquad +  \int  \left\{\frac{1}{1 - \alpha} \cdot \frac{dP}{d\tau}
        - \frac{1}{1 - \alpha} \cdot e^{(\alpha - 1) \cdot  (T_* - \hat{T})} \frac{dP}{d\tau}  \right\}  e^{ (\alpha-1) \cdot \hat{T}}\cdot  d\tau \nonumber\\
        &=& \int \left\{   \frac{1}{\alpha} \cdot e^{\alpha \cdot \hat{T}} \frac{dQ}{d\nu } - \frac{1}{\alpha} \cdot e^{\alpha \cdot T_*}   \frac{dQ}{d\nu}  \right\} d\nu  \nonumber\\
        && \qquad \qquad +  \int  \left\{\frac{1}{1 - \alpha}  \cdot e^{ (\alpha-1) \cdot  \hat{T}} \frac{dP}{d\tau}
        - \frac{1}{1 - \alpha} \cdot e^{(\alpha - 1) \cdot T_* } \frac{dP}{d\tau}  \right\}  d\tau \nonumber\\
        &=& \left\{  \frac{1}{\alpha} \int  e^{ \alpha \cdot \hat{T}} dQ  + \frac{1}{1 - \alpha} \int  e^{(\alpha - 1) \cdot \hat{T} } dP \right\}   \nonumber\\
        &&  \qquad \qquad \qquad  - \  \left\{   \frac{1}{\alpha} \int  e^{\alpha \cdot T_*} dQ + \frac{1}{1 - \alpha} \int e^{(\alpha - 1) \cdot T_*} dP  \right\} \nonumber\\
        &=& L_{\alpha}(Q, P; \hat{T}) -  L_{\alpha}(Q, P;  T_*).
    \end{eqnarray}
    This completes the proof.
\end{proof}

\subsection{Proofs for Section \ref{Section_Method}}
In this Section, we present two theorems for the proposed method in Section \ref{Section_Method}.
Before presenting the first theorem, we briefly review Pearl's $do$-calculus (\citeauthor{pearl1995causal}(\citeyear{pearl1995causal})) used in the proof of the first theorem.

\begin{theorem}[$do$-calculus, \citeauthor{pearl1995causal}(\citeyear{pearl1995causal})] \label{Theorem_do_calus}
    Causal effects can be transformed  by following rules R1-R3:
    \begin{enumerate}
        \item[R1. ]  $P(\mathbf{Y}|do(\mathbf{X}),\mathbf{Z},\mathbf{W})
        = P(\mathbf{Y}|do(\mathbf{X}), \mathbf{W})$,
        \ if \
        $(\mathbf{Y} \indep \mathbf{Z}|\mathbf{X},\mathbf{W})_{
            \overline{G}(\mathbf{X})}$.
        \item[R2. ] $P(\mathbf{Y}|do(\mathbf{X}),do(\mathbf{Z}),\mathbf{W})
        = P(\mathbf{Y}|do(\mathbf{X}), \mathbf{Z},\mathbf{W})$,
        \ if \  $(\mathbf{Y} \indep \mathbf{Z}|\mathbf{X},\mathbf{W})_{
            \overline{\underline{G}}(\mathbf{X}, \mathbf{Z})}$.
        \item[R3. ]  $P(\mathbf{Y}|do(\mathbf{X}),do(\mathbf{Z}),\mathbf{W})
        = P(\mathbf{Y}|do(\mathbf{X}), \mathbf{W})$,
        \ if \ $(\mathbf{Y} \indep \mathbf{Z}|\mathbf{X},\mathbf{W})_{
            \overline{G}(\mathbf{X},\mathbf{Z}^*)}$, \ where \  $\mathbf{Z}^* = \mathbf{Z} \setminus An(\mathbf{W})_{\overline{G}(\mathbf{X})}$.
    \end{enumerate}

    Here, $\overline{G}(\mathbf{A})$ denotes a graph obtained from $G$ by deleting
    all arrows emerging from variables to $\mathbf{A}$, and
    $\underline{\overline{G}}(\mathbf{A}, \mathbf{B})$ denotes a graph obtained from $G$
    by deleting both of all arrows emerging from any variables to $\mathbf{A}$
    and all arrows emerging from $\mathbf{B}$ to any variables, and
    $(\mathbf{A} \indep \mathbf{B})_G$ represents that there is no path between $\mathbf{A}$ and  $\mathbf{B}$ in $G$.
\end{theorem}

We now provide the first theorem, which presents a sufficient condition for explanatory variables to be available for estimating causal effects.
\begin{theorem}\label{theorem_intervention_eq}
    Let $G$ be a DAG for $\mathbf{V}$ and  $\mathbf{U}$.
    For disjoint sets $\mathbf{X}, \mathbf{Y}, \mathbf{Z} \subset \mathbf{V}$,
    suppose that $P(\mathbf{Y}|do(\mathbf{X}),\mathbf{Z})$ is identifiable in $G$,
    and $\mathbf{X} \subset An(\mathbf{Y})_G$. Let $\mathbf{Z}_{De} = \mathbf{Z} \cap De(\mathbf{Y})_G$.
    Then,
    \begin{eqnarray}
        \lefteqn{P(\mathbf{Y} |do(\mathbf{X}),\mathbf{Z})} \quad \nonumber\\
        &=&
        \begin{cases}
            P(\mathbf{Y} |\mathbf{X},\mathbf{Z}) & \text{if  $\mathbf{Z}_{De}=\phi$}, \\
            \frac{P(\mathbf{Y} |\mathbf{X}, \mathbf{Z}\setminus \mathbf{Z}_{De})
                P(\mathbf{Z}_{De}|\mathbf{Y},\mathbf{X},\mathbf{Z}\setminus\mathbf{Z}_{De})}{P(\mathbf{Z}_{De}|\mathbf{X}, \mathbf{Z}\setminus\mathbf{Z}_{De})} & \text{if  $\mathbf{Z}_{De}\neq\phi$}.
        \end{cases} \label{eq_prob_formula_intervent_first_lemma}
    \end{eqnarray}
\end{theorem}

\begin{proof}[proof of Theorem \ref{theorem_intervention_eq}]
    We note that each $Z_i \in \mathbf{Z}$ can be assumed to be that either $Z_i \in An(\mathbf{Y})_{G}$ or $Z_i \in De(\mathbf{Y})_{G}$.
    To see this, suppose that there exist some $Z_i \in \mathbf{Z}$ such that $Z_i \notin An(\mathbf{Y})_{G}$ and $Z_i \notin De(\mathbf{Y})_{G}$.
    Let $\mathbf{V}'=\mathbf{V}\setminus(\mathbf{Y}\cup\mathbf{X}\cup\mathbf{Z})$.
    Since $(\mathbf{Y} \indep Z_i|\mathbf{X},\mathbf{V}', \mathbf{U})_{\overline{G}(\mathbf{X})}$ holds for the $Z_i$,
    by applying $do$-calculus R1 in Theorem \ref{Theorem_do_calus}, we have
    \begin{equation}
        P(\mathbf{Y}|do(\mathbf{X}),\mathbf{Z} \setminus \{Z_i\}, Z_i, \mathbf{V}', \mathbf{U})
        = P(\mathbf{Y}|do(\mathbf{X}), \mathbf{Z} \setminus \{Z_i\}, \mathbf{V}', \mathbf{U}). \label{proof_theorem_intervention_eq_Eq_obtained_from_do_cal_R1}
    \end{equation}
    By  marginalizing both sides of (\ref{proof_theorem_intervention_eq_Eq_obtained_from_do_cal_R1})
    for $\mathcal{X}_{\mathbf{V}' \cup \mathbf{U}}$,
    we obtain
    \begin{equation}
        P(\mathbf{Y}|do(\mathbf{X}),\mathbf{Z} \setminus \{Z_i\}, Z_i)
        = P(\mathbf{Y}|do(\mathbf{X}), \mathbf{Z} \setminus \{Z_i\}).
    \end{equation}
    Thus, after repeating the above calculation, $P(\mathbf{Y}|do(\mathbf{X}),\mathbf{Z})$
    finally includes only $Z_i \in \mathbf{Z}$
    such that $Z_i \in An(\mathbf{Y})_{G}$ or $Z_i \in De(\mathbf{Y})_{G}$.

    Therefore, in this proof, we assume that
    \begin{equation}
        \mathbf{Z}=An(\mathbf{Y})_{G}\cup De(\mathbf{Y})_{G}. \label{proof_lemma_do_condi_Eq_assume}
    \end{equation}

    Next, we note that $\mathbf{Z}\cap An(\mathbf{X})_{G}\cap De(\mathbf{Y})_{G}=\phi$.
    To see this, suppose $\mathbf{Z}\cap An(\mathbf{X})_{G}\cap De(\mathbf{Y})_{G} \neq \phi$.
    Let $\dotarrow{\mathbf{V}'}$ denote a path through only variables of $\mathbf{V}'$.
    Then there exists a directed path such that $\mathbf{Y} \dotarrow{\mathbf{V}'} \mathbf{Z} \dotarrow{\mathbf{V}'}  \mathbf{X}$,
    which contradicts the assumption $\mathbf{X} \subset An(\mathbf{Y})_G$.

    From the above discussion, $\mathbf{Z}$ can be divided into the three disjoint sets as follows:
    \begin{eqnarray}
        \mathbf{Z} &=& \mathbf{Z}_1\cup\mathbf{Z}_2\cup\mathbf{Z}_3,  \nonumber \\
        \mathbf{Z}_1 &=& (\mathbf{Z} \setminus De(\mathbf{X})_G) \cap An(\mathbf{Y})_G, \nonumber \\
        \mathbf{Z}_2 &=& \mathbf{Z} \cap De(\mathbf{X})_G \cap An(\mathbf{Y})_G, \nonumber \\
        \mathbf{Z}_3 &=& (\mathbf{Z} \setminus An(\mathbf{X})_{G}) \cap De(\mathbf{Y})_G. \nonumber
    \end{eqnarray}
    Then, each of the paths between $\mathbf{Z}_1$, $\mathbf{Z}_2$ and $\mathbf{Z}_3$ is one
    of the following P1, P2 and P3:
    \begin{enumerate}
        \item[P1. ] $\mathbf{Z}_1 \dotarrow{\mathbf{V}'} \mathbf{Z}_2$,
        \item[P2. ] $\mathbf{Z}_2 \dotarrow{\mathbf{V}'} \mathbf{Z}_3$,
        \item[P3. ] $\mathbf{Z}_1\dotarrow{\mathbf{V}'} \mathbf{Z}_3$.
    \end{enumerate}

    In fact, if there exists a directed path in the opposite direction of P1,
    that is	$\mathbf{Z}_2 \dotarrow{\mathbf{V}'} \mathbf{Z}_1$,
    then there exists a path such that
    $X_i \dotarrow{\mathbf{V}'}  \mathbf{Z}_2 \dotarrow{\mathbf{V}'}  \mathbf{Z}_1$.
    This contradicts the assumption $\mathbf{Z}_1 \subset \mathbf{Z} \setminus De(\mathbf{X})_G$.
    Similarly, if there exists a directed path in the opposite direction of P2,
    that is	$\mathbf{Z}_3 \dotarrow{\mathbf{V}'} \mathbf{Z}_2$,
    then there exists a path such that
    $Y_i \dotarrow{\mathbf{V}'}  \mathbf{Z}_3 \dotarrow{\mathbf{V}'}  \mathbf{Z}_2$, which
    contradicts the assumption $\mathbf{Z}_2 \subset An(\mathbf{Y})_G$.
    In addition, if there exists a directed path in the opposite direction of P3,
    that is	$\mathbf{Z}_3 \dotarrow{\mathbf{V}'} \mathbf{Z}_1$, then there exists a path such that
    $Y_i \dotarrow{\mathbf{V}'}  \mathbf{Z}_3 \dotarrow{\mathbf{V}'}  \mathbf{Z}_1$, which
    contradicts the assumption $\mathbf{Z}_1 \subset An(\mathbf{Y})_G$.
    Therefore, all paths expect P1, P2 and P3 are denied.

    Hence, by marginalizing $P(\mathbf{V})$ for $\mathcal{X}_{\mathbf{V}'}$, we obtain
    \begin{eqnarray}
        P(\mathbf{Y}, \mathbf{X}, \mathbf{Z}) &=&  \sum_{\mathcal{X}_{\mathbf{V}'}} P(\mathbf{V}) \nonumber\\
        &=&  P(\mathbf{Y}|\mathbf{X}, \mathbf{Z}_1, \mathbf{Z}_2)
        \cdot P(\mathbf{X}|\mathbf{Z}_1) \cdot P(\mathbf{Z}_1) \nonumber\\
        && \times P(\mathbf{Z}_2|\mathbf{X}, \mathbf{Z}_1)
        \cdot P(\mathbf{Z}_3|\mathbf{Y}, \mathbf{X},\mathbf{Z}_1 ,\mathbf{Z}_2). \nonumber
    \end{eqnarray}
    In additon, from (\ref{def_causal_effect_no_condi_eq}), we have
    \begin{eqnarray}
        P(\mathbf{Y}, \mathbf{Z}|do(\mathbf{X})) &=& P(\mathbf{Y}|\mathbf{X}, \mathbf{Z}_1, \mathbf{Z}_2)
        \cdot  P(\mathbf{Z}_1) \cdot P(\mathbf{Z}_2|\mathbf{X}, \mathbf{Z}_1) \nonumber\\
        && \times P(\mathbf{Z}_3|\mathbf{Y}, \mathbf{X},\mathbf{Z}_1 ,\mathbf{Z}_2). \label{eq_do_x_p_y_z}
    \end{eqnarray}
    In the case that $\mathbf{Z}_3 = \phi$, by marginalizing out $\mathbf{Y}$ of (\ref{eq_do_x_p_y_z}), we have
    \begin{eqnarray}
        P(\mathbf{Z}|do(\mathbf{X})) &=& \sum_{\mathbf{y} \in \mathcal{X}_{\mathbf{Y}}}
        P(\mathbf{Y}=\mathbf{y}, \mathbf{Z}|do(\mathbf{X})) \nonumber \\
        &=& \sum_{\mathbf{y} \in \mathcal{X}_{\mathbf{Y}}} P(\mathbf{Y}=\mathbf{y}|\mathbf{X}, \mathbf{Z}_1, \mathbf{Z}_2) \cdot  P(\mathbf{Z}_1) \cdot P(\mathbf{Z}_2|\mathbf{X}, \mathbf{Z}_1) \nonumber\\
        &=& P(\mathbf{Z}_1) \cdot P(\mathbf{Z}_2|\mathbf{X}, \mathbf{Z}_1)
        \sum_{\mathbf{y} \in \mathcal{X}_{\mathbf{Y}}} P(\mathbf{Y}=\mathbf{y}|\mathbf{X}, \mathbf{Z}_1, \mathbf{Z}_2) \nonumber\\
        &=& P(\mathbf{Z}_1) \cdot P(\mathbf{Z}_2|\mathbf{X}, \mathbf{Z}_1). \nonumber
    \end{eqnarray}
    On the other hand, in the case that $\mathbf{Z}_3 \neq \phi$, we obtain
    \begin{eqnarray}
        P(\mathbf{Z}|do(\mathbf{X})) &=& \sum_{\mathbf{y} \in \mathcal{X}_{\mathbf{Y}}}
        P(\mathbf{Y}=\mathbf{y}, \mathbf{Z}|do(\mathbf{X})) \nonumber\\
        &=& \sum_{\mathbf{y} \in \mathcal{X}_{\mathbf{Y}}} P(\mathbf{Y}=\mathbf{y}|\mathbf{X}, \mathbf{Z}_1, \mathbf{Z}_2) \cdot  P(\mathbf{Z}_1) \nonumber\\
        && \times P(\mathbf{Z}_2|\mathbf{X}, \mathbf{Z}_1) \cdot P(\mathbf{Z}_3|\mathbf{Y}=\mathbf{y}, \mathbf{X},\mathbf{Z}_1 ,\mathbf{Z}_2) \nonumber\\
        &=& P(\mathbf{Z}_1) \cdot P(\mathbf{Z}_2|\mathbf{X}, \mathbf{Z}_1) \nonumber\\
        && \qquad \times \sum_{\mathbf{y} \in \mathcal{X}_{\mathbf{Y}}}  P(\mathbf{Z}_3|\mathbf{Y}=\mathbf{y}, \mathbf{X},\mathbf{Z}_1 ,\mathbf{Z}_2) \cdot P(\mathbf{Y}=\mathbf{y}|\mathbf{X}, \mathbf{Z}_1, \mathbf{Z}_2) \nonumber\\
        &=& P(\mathbf{Z}_1) \cdot P(\mathbf{Z}_2|\mathbf{X}, \mathbf{Z}_1) \cdot P(\mathbf{Z}_3|\mathbf{X},\mathbf{Z}_1 ,\mathbf{Z}_2). \nonumber
    \end{eqnarray}
    Summarizing the above results, we have
    \begin{equation}
        P(\mathbf{Z}|do(\mathbf{X})) =
        \begin{cases}
            P(\mathbf{Z}_1) \cdot P(\mathbf{Z}_2|\mathbf{X}, \mathbf{Z}_1),  & \text{if  $\mathbf{Z}_3=\phi$}, \\
            P(\mathbf{Z}_1) \cdot P(\mathbf{Z}_2|\mathbf{X}, \mathbf{Z}_1) \cdot P(\mathbf{Z}_3|\mathbf{X},\mathbf{Z}_1 ,\mathbf{Z}_2) & \text{if  $\mathbf{Z}_3 \neq \phi$}.
        \end{cases} \label{eq_do_x_p_z}
    \end{equation}
    Inserting (\ref{eq_do_x_p_y_z}) and (\ref{eq_do_x_p_z}) into  (\ref{def_causal_effect_on_condi_eq}), we see
    \begin{eqnarray}
        P(\mathbf{Y}|do(\mathbf{X}), \mathbf{Z}) &=& \frac{P(\mathbf{Y}, \mathbf{Z}|do(\mathbf{X}))}{P(\mathbf{Z}|do(\mathbf{X}))} \nonumber \\
        &=&
        \begin{cases}
            P(\mathbf{Y}|\mathbf{X}, \mathbf{Z}_1, \mathbf{Z}_2) & \text{if  $\mathbf{Z}_3=\phi$}, \\
            \frac{P(\mathbf{Y}|\mathbf{X}, \mathbf{Z}_1, \mathbf{Z}_2) P(\mathbf{Z}_3|\mathbf{Y}, \mathbf{X},\mathbf{Z}_1 ,\mathbf{Z}_2)}
            {P(\mathbf{Z}_3|\mathbf{X},\mathbf{Z}_1 ,\mathbf{Z}_2)} & \text{if  $\mathbf{Z}_3 \neq \phi$} \label{proof_lemma_do_z_condi}.
        \end{cases}
    \end{eqnarray}
    Note that, $\mathbf{Z}_3 = \mathbf{Z} \cap De(\mathbf{Y})_{G}$, since $\mathbf{Z}\cap An(\mathbf{X})_{G}\cap De(\mathbf{Y})_{G}=\phi$.

    Therefore, by rewriting $\mathbf{Z}_3$ as $\mathbf{Z}_{De}$ and $\mathbf{Z}_1 \cup \mathbf{Z}_2$ as $\mathbf{Z} \setminus \mathbf{Z}_{De}$ for (\ref{proof_lemma_do_z_condi}),
    we obtain (\ref{eq_prob_formula_intervent_first_lemma}).

    This completes the proof.
\end{proof}

Next, we provide the main theorem presented in Section \ref{Section_Method}.
\begin{theorem}[Theorem \ref{Theorem_main_result} restated] \label{restated_Theorem_main_result}
    Given disjoint sets of
    $\mathbf{X} = \{\mathbf{X}_1, \allowbreak \mathbf{X}_2,\allowbreak \dots, \allowbreak\mathbf{X}_n \}, \allowbreak
    \mathbf{Y},\allowbreak \mathbf{Z} \subset\allowbreak  \mathbf{V}$  satisfying
    \begin{equation}
        \mathbf{X}=\{\mathbf{X}_1,\mathbf{X}_2,\dots, \mathbf{X}_n\} \subset An(\mathbf{Y})_G, \label{restate_theorem_Assumpion_Variables1}
    \end{equation}
    and
    \begin{equation}
        \mathbf{Z} \cap De(\mathbf{Y})_G=\phi. \label{restate_theorem_Assumpion_Variables2}
    \end{equation}
    Let  $\mathbb{P} = P(\mathbf{X}_1,\mathbf{X}_2,\dots, \mathbf{X}_n, \mathbf{Z})$ and
    $\mathbb{Q} = P(\mathbf{X_1}) \times P(\mathbf{X_2}) \times \cdots \times P(\mathbf{X_n}) \times P(\mathbf{Z})$,
    and $\widetilde{P} =
    P(\mathbf{Y}|do(\mathbf{X}), \mathbf{Z}) \times P(\mathbf{X_1}) \times P(\mathbf{X_2}) \times \cdots
    \times P(\mathbf{X_n}) \times P(\mathbf{Z})$.

    Suppose $P$ satisfies Assumptions 1 and 2 in the above setting,
    and it holds that $E_{\mathbb{P}}\left[ \left( d\mathbb{Q}/d\mathbb{P}\right)^{1 - \alpha} \right] < \infty$
    for some $0 < \alpha < 1$, then, for the optimal function $T^*$, such that
    \begin{eqnarray}
        \lefteqn{ T^*(\mathbf{X}_1,\mathbf{X}_2,\dots, \mathbf{X}_n, \mathbf{Z})} \nonumber\\
        &=& \arg \inf_{T \in \mathcal{T}^{\alpha}} \left\{ \frac{1}{\alpha} E_{\mathbb{Q}}\left[e^{\alpha \cdot T}\right] \right.  \nonumber \\
        && \qquad \qquad \quad \left.  + \frac{1}{1 - \alpha } E_{\mathbb{P}}\left[e^{(\alpha - 1) \cdot T }\right] \right\} ,  \label{Eq_restate_theorem_optimize}
    \end{eqnarray}
    it holds that
    \begin{equation}
        \frac{d \widetilde{P}}{dP} = e^{-T^*(\mathbf{X}_1,\mathbf{X}_2,\dots, \mathbf{X}_n, \mathbf{Z})}. \label{Eq_restate_theorem_gibbs_density_estimate}
    \end{equation}
    Here, $\mathcal{T}^{\alpha}$ denotes the set of all non-constant functions $T(\mathbf{x}):\mathbb{R}^d \rightarrow \mathbb{R}$
    with $E_{\mathbb{P}}[e^{(\alpha -1) \cdot T(\mathbf{X})}] < \infty$.
\end{theorem}

\begin{proof}[proof of Theorem \ref{restated_Theorem_main_result}]
    From Theorem \ref{theorem_intervention_eq} and the assumption (\ref{restate_theorem_Assumpion_Variables2}), we have
    \begin{eqnarray}
        \widetilde{P} &=& P(\mathbf{Y}|do(\mathbf{X}), \mathbf{Z}) \times P(\mathbf{X_1}) \times P(\mathbf{X_2}) \times \cdots \times P(\mathbf{X_n}) \times P(\mathbf{Z}) \nonumber \\
        &=& P(\mathbf{Y}|\mathbf{X}, \mathbf{Z}) \times P(\mathbf{X_1}) \times P(\mathbf{X_2}) \times \cdots \times P(\mathbf{X_n}) \times P(\mathbf{Z}). \nonumber
    \end{eqnarray}
    Thus, from Lemma \ref{proposition_alpha_div_resp_in_gibbs_dinsity_form}, we obtain
    \begin{equation}
        e^{-T^*(\mathbf{X}_1,\mathbf{X}_2,\dots, \mathbf{X}_n, \mathbf{Z})} = \frac{d\mathbb{Q}}{d\mathbb{P}}  =  \frac{d \widetilde{P}}{dP}.  \label{proof_Eq_restate_theorem_gibbs_density_estimate}
    \end{equation}

    This completes the proof.
\end{proof}

\subsection{Proofs for Section \ref{Section_techniques_estimating_NGD_balancing_weights}}
In this section, we first present a proposition for obtaining the density ratio between empirical distributions of the source and target distributions.
Next, we present a proposition and lemmas for the early stopping method proposed in this study.
\\

\begin{proposition}\label{proposition_representation_of_density_ratios_between_emprical_distributions}
    It holds that
   \begin{equation}
       \frac{d \hat{Q}^{(N)}}{d \hat{P}^{(N)}}(\mathbf{x}) = \left\{
       \begin{array}{ll}
           dQ/dP(\mathbf{x})  &  \text{if} \ \   1 \le \exists i \le N \ \text{s.t.} \  \mathbf{X}_i = \mathbf{x}, \\
           0 &\text{otherwise.}
       \end{array}
       \right.
   \end{equation}
\end{proposition}

\begin{proof}[proof of Proposition \ref{corollary_early_stopping_1}]
    Let $\nu$ be the countable measure on $\left\{\mathbf{X}_1, \mathbf{X}_2, \ldots, \mathbf{X}_N\right\}$:
    \begin{equation}
        \nu(\mathbf{x}) = \left\{
         \begin{array}{ll}
            1  &  \text{if} \ \   1 \le \exists i \le N \ \text{s.t.} \  \mathbf{X}_i = \mathbf{x}, \\
            0 &\text{otherwise.}
        \end{array}
        \right. \label{Eq_proof_proposition_representation_of_density_ratios_between_emprical_distributions_countable_measure_df}
    \end{equation}
    Then, $\hat{P}^{(N)} \ll \nu$ and $\hat{Q}^{(N)} \ll \nu$.

    Note that, from the definitions of $\hat{P}^{(N)}(\mathbf{x})$ and $\hat{Q}^{(N)}(\mathbf{x})$, we have
      \begin{equation}
          \hat{P}^{(N)}(\mathbf{x}) = \frac{1}{N} \sum_{i} \mathbf{1}(\mathbf{X}^{i}_{\sim P}=\mathbf{x}) =  \frac{1}{N} \sum_{i} \mathbf{1}(\mathbf{X}^{i}=\mathbf{x})\cdot \frac{dP}{d\mu} (\mathbf{x}),
           \label{Eq_def_enprical_dist_PN}
      \end{equation}
      and
      \begin{equation}
          \hat{Q}^{(N)}(\mathbf{x}) = \frac{1}{N} \sum_{i} \mathbf{1}(\mathbf{X}^{i}_{\sim Q}=\mathbf{x}) = \frac{1}{N} \sum_{i} \mathbf{1}(\mathbf{X}^{i}=\mathbf{x})\cdot \frac{dQ}{d\mu} (\mathbf{x}),
          \label{Eq_def_enprical_dist_QN}
      \end{equation}
      where $\mathbf{1}(\cdot)$ equals one if the statement in parentheses is true and zero otherwise.

    From (\ref{Eq_def_enprical_dist_PN}) and (\ref{Eq_def_enprical_dist_QN}), if $\mathbf{X}_i = \mathbf{x}$, we see
    \begin{equation}
        \frac{d \hat{P}^{(N)}}{d \nu}(\mathbf{x}) = \hat{P}^{(N)}(\mathbf{x}) = \frac{1}{N} \frac{dP}{d\mu} (\mathbf{x}),
    \end{equation}
    and
    \begin{equation}
        \frac{d \hat{Q}^{(N)}}{d \nu}(\mathbf{x}) = \hat{Q}^{(N)}(\mathbf{x}) = \frac{1}{N} \frac{dQ}{d\mu} (\mathbf{x}).
    \end{equation}
    Then, we have
    \begin{equation}
        \frac{d \hat{Q}^{(N)}}{d \hat{P}^{(N)}}(\mathbf{x}) = \frac{\frac{d \hat{P}^{(N)}}{d \nu}(\mathbf{x})}{\frac{d \hat{Q}^{(N)}}{d \nu}(\mathbf{x})} = \frac{dQ}{dP}(\mathbf{x}).
    \end{equation}

    For $\mathbf{x} \notin \left\{\mathbf{X}_1, \mathbf{X}_2, \ldots, \mathbf{X}_N\right\}$, we observe $d\hat{Q}^{(N)}/d\nu (\mathbf{x}) = 0$.
    Note that, $d \hat{Q}^{(N)}/d \hat{P}^{(N)}(\mathbf{x})$ is defined as zero for $\mathbf{x} \in \Omega$ such that $d\hat{Q}^{(N)}/d\nu (\mathbf{x}) = 0$.
    Subsequently, we see  $d \hat{Q}^{(N)}/d \hat{P}^{(N)}(\mathbf{x}) = 0$.
\end{proof}

Next, we present a proposition for the early stopping method proposed in Section \ref{Subsetion_Improving_The_Generalization_Performance}.
We obtain an early stopping step as the step that minimizes the $W_1$ distance of the balanced distribution and target distribution, $\hat{Q}^{(N)}_k$ and $Q$ in (\ref{Eq_the_final_destination_balancing_has_COD}).
To obtain the early stopping step, we assume that the two distributions differ the worst outside the neighborhood of the observations because we cannot know the closeness of the two distributions, $\hat{Q}^{(N)}_k$ and $Q$ in (\ref{Eq_the_final_destination_balancing_has_COD}) except in the neighborhood of the observations.

We now provide a note on the convergence rate for optimizing the loss function (\ref{Eq_delta_loss_alpha}). Let
\begin{equation}
    f^{(N)}(\mathbf{t}) = f^{(N)}(t_1, t_2, \ldots, t_N) =  \frac{1}{\alpha}  \frac{1}{N} \sum_{i=1}^N e^{\alpha \cdot t_i} \cdot \frac{dQ}{d\mu}(\mathbf{x}_i)
    +  \frac{1}{1- \alpha} \frac{1}{N} \sum_{i=1}^N   e^{(\alpha - 1) \cdot t_i} \cdot \frac{dP}{d\mu}(\mathbf{x}_i).
    \label{Eq_convergence_rate_discription_loss_func_to_optim}
\end{equation}
Subsequently, let $\mathbf{t}_{K}$ denote a model at step $K$ when optimizing (\ref{Eq_convergence_rate_discription_loss_func_to_optim}) with a Stochastic Gradient Desent (SGD) algorithm.
Because, from Corollary \ref{corollary_loss_is_mu_storongly_convex}, $f^{(N)}(\mathbf{t})$ is strongly convex with $\left\| \nabla f^{(N)}(\mathbf{t}) \right\|^2 \ \le D^2$ ($\exists D \in \mathbb{R}$) around the optimal point $\mathbf{t}_* = \allowbreak (t_*^1, \allowbreak t_*^2, \allowbreak \ldots, \allowbreak t_*^N) = \allowbreak
(- \log \frac{dQ}{dP}(\mathbf{x}_1), \allowbreak - \log \frac{dQ}{dP}(\mathbf{x}_2), \allowbreak \ldots, \allowbreak - \log \frac{dQ}{dP}(\mathbf{x}_N))$, an $O(1/K)$ convergence rate can be achieved at step $K$ when optimizing (\ref{Eq_convergence_rate_discription_loss_func_to_optim}) with SGD algorithms under regular conditions for $\mathbf{t}$:
\begin{equation}
    E\left[f_{N}(\bar{\mathbf{t}}_K) \right] - f_{N}(\mathbf{t}_*) \le  \frac{C}{K+1}, \label{Eq_convergence_rate_discription_convergence_rate_discription}
\end{equation}
where $\bar{\mathbf{t}}_K$ is a weighted averaging such that $\bar{\mathbf{t}}_K = \frac{1}{(K+1)\cdot (K+2)} \sum_k (k+1) \cdot \mathbf{t}_k$ and $C > 0$ is constant.
Here $E\left[\cdot \right]$ denotes the expectation for the randomness of batch sampling of SGD.\footnote{For the convergence rate of SGD algorithms, for example, readers can refer to \cite{lacoste2012simpler}.}
As assumptions close to (\ref{Eq_convergence_rate_discription_convergence_rate_discription}), we briefly assume (\ref{Eq_proposition_early_stopping_assm_emprical_cov_rate}) in Assumption E1 and (\ref{Eq_proposition_early_stopping_assm_exp_cov_rate}) in Assumption E2 to obtain an early stopping step,
which are simpler and more relaxed than  (\ref{Eq_convergence_rate_discription_convergence_rate_discription}).

Herein, we make the following assumptions for the early stopping method presented in Section \ref{Section_techniques_estimating_NGD_balancing_weights}.
\begin{itemize}
    \item Assumption E1. Let $\{T_k^{(N)}\}_{k=1}^{\infty}$ be a sequence of functions in $\mathcal{T}^{\alpha}$ 
    such that \\
    $\lim_{k \rightarrow \infty} \allowbreak  T_k^{(N)}(\mathbf{X}_i) \allowbreak = \allowbreak - \log dQ/dP(\mathbf{X}_i)$,
    for $1 \allowbreak \le \allowbreak \forall \allowbreak  i \allowbreak \le \allowbreak N$. Suppose that
    \begin{equation}
        \hat{L}_{\alpha}^{(N)}(Q, P;T_k^{(N)}) - \hat{L}_{\alpha}^{(N)}(Q, P;T_*) \le \frac{C_0}{K}, \label{Eq_proposition_early_stopping_assm_emprical_cov_rate}
    \end{equation}
    where $\hat{L}_{\alpha}^{(N)}(Q, P;\cdot)$ is defined as (\ref{Eq_lemma_loss_not_biased_empirical_lemma_loss_T}) in Lemma \ref{lemma_loss_not_biased_lemma} and $C_0 > 0$ is constant.

    \item Assumption E2. Let $\{T_k\}_{k=1}^{\infty}$ be a sequence of functions in $\mathcal{T}^{\alpha}$ with $E_{P}[e^{-T_k(\mathbf{X})}] = 1$
    such that $\lim_{k \rightarrow \infty} T_k= - \log dQ/dP$, $P$-almost everywhere. Suppose that
    \begin{equation}
        L_{\alpha}(Q, P;T_k) -L_{\alpha}(Q, P;T_*) \le  \frac{C_1}{K}, \label{Eq_proposition_early_stopping_assm_exp_cov_rate}
    \end{equation}
    where  $L_{\alpha}(Q, P; \cdot)$ is defined as (\ref{Eq_lemma_loss_not_biased_lemma_loss_T}) in Lemma \ref{lemma_loss_not_biased_lemma} and $C_1 > 0$ is constant.
\end{itemize}

In addition, we make the following assumptions to simplify the discussion in the proofs.
\begin{itemize}
    \item Assumption E3. Let $\Omega$ be a compact set in $\mathbb{R}^{d}$ with $\mathrm{diam}(\Omega)=1$, where $\mathrm{diam}(\Omega)$ denotes the diameter of $\Omega$.
                         Then $\lambda$ denotes the Lebesgue measure on $\mathbb{R}^{d}$.

    \item Assumption E4. Let $Q$ and $P$ be two probabilities on $\Omega$ with continuous probability densities $p(\mathbf{x})$ and $q(\mathbf{x})$, respectively.
    Assume $0 < p_{min} \le p(\mathbf{x}) \le p_{max}$ and $0< q_{min} \le q(\mathbf{x}) \le  q_{max}$ for all $\mathbf{x} \in \Omega$.

    \item Assumption E5. For $\{T_k^{(N)}\}_{k=1}^{\infty}$ in Assumption E1,
    assume that each function of $T_k^{(N)}(\mathbf{X})$ is Lipschitz continuous: for $1 \le k \le \infty$,
    \begin{equation}
        |T_k^{(N)}(\mathbf{x}) - T_k^{(N)}(\mathbf{y})| \le \rho_k \cdot \|\mathbf{x} - \mathbf{y} \|.
    \end{equation}

    \item Assumption E6. For $\{T_k^{(N)}\}_{k=1}^{\infty}$ in Assumption E2,
    assume that each function of $T_k(\mathbf{X})$ is Lipschitz continuous: for $1 \le k \le \infty$,
    \begin{equation}
        |T_k(\mathbf{x}) - T_k(\mathbf{y})| \le \widetilde{\rho}_k \cdot \|\mathbf{x} - \mathbf{y} \|.
    \end{equation}
\end{itemize}
Note that, the Lipschitz coefficient in Assumption E5 does not depend on the sample size $N$.
\\
\\

\begin{lemma}\label{lemma_e_T_K_N_upper_bound}
    For $\{T_k^{(N)}\}_{k=1}^{\infty}$ in Assumption E5, it holds that
    for $\mathbf{x} \in \Omega$ and $\|\mathbf{y} -  \mathbf{x}\| < D$,
    \begin{equation}
       e^{-T_k^{(N)}(\mathbf{y})} = e^{-T_k^{(N)}(\mathbf{x})} + e^{-T_k^{(N)}(\mathbf{x})} \cdot \left\{ O \left(D\right) + O \left(D^2\right) \right\} +  O_{\mathbf{x}} \left(D\right).
    \end{equation}
\end{lemma}

\begin{proof}[proof of Lemma \ref{lemma_e_T_K_N_upper_bound}]
    From the intermediate value theorem for the second derivative of $e^{-x}$, we have
    \begin{equation}
        e^{-y} = e^{-x} - e^{-x}\cdot(y-x) + \frac{e^{-x + \theta \cdot (x - y)}}{2} \cdot(y-x)^2,
    \end{equation}
    where $0 < \theta < 1$.

    By substituting $y=T_k^{(N)}(\mathbf{y})$ and $x=T_k^{(N)}(\mathbf{x})$ into the above formula, we obtain
    \begin{eqnarray}
        e^{-T_k^{(N)}(\mathbf{y})} &=& e^{-T_k^{(N)}(\mathbf{x})}
        - e^{-T_k^{(N)}(\mathbf{x})} \left(T_k^{(N)}(\mathbf{y}) - T_k^{(N)}(\mathbf{x}) \right) \nonumber \\
        &&  \  + \  \frac{e^{- T_k^{(N)}(\mathbf{x}) +  \theta(\mathbf{x}, \mathbf{y})\cdot \left( T_k^{(N)}(\mathbf{y}) - T_k^{(N)}(\mathbf{x})\right) }}{2}
            \left(T_k^{(N)}(\mathbf{y}) - T_k^{(N)}(\mathbf{x}) \right)^2, \label{Eq_proof_lemma_e_T_K_N_upper_bound_interm_th_for_T_k}
    \end{eqnarray}
    where $0< \theta(\mathbf{x}, \mathbf{y}) < 1$.

    Now, note that, from Assumption E5,
    \begin{equation}
        T_k^{(N)}(\mathbf{y}) - T_k^{(N)}(\mathbf{x}) \le \rho_k \cdot \|\mathbf{x} - \mathbf{y} \| \le \rho_k \cdot D. \label{Eq_proof_lemma_e_T_K_N_upper_bound_upper_bound_for_T_k}
    \end{equation}

    From (\ref{Eq_proof_lemma_e_T_K_N_upper_bound_interm_th_for_T_k}) and (\ref{Eq_proof_lemma_e_T_K_N_upper_bound_upper_bound_for_T_k}), we see
    \begin{eqnarray}
        \left|e^{-T_k^{(N)}(\mathbf{y})} - e^{-T_k^{(N)}(\mathbf{x})} \right| &=& \left| e^{-T_k^{(N)}(\mathbf{x})} \cdot \left(T_k^{(N)}(\mathbf{y}) - T_k^{(N)}(\mathbf{x}) \right) \right.\nonumber \\
         &&  \left. \ \ + \ \frac{e^{- T_k^{(N)}(\mathbf{x}) + \theta(\mathbf{x}, \mathbf{y}) \cdot(T_k^{(N)} \left(\mathbf{y}) - T_k^{(N)}(\mathbf{x})\right)}}{2}
                    \cdot \left(T_k^{(N)}(\mathbf{y}) - T_k^{(N)}(\mathbf{x}) \right)^2  \right|  \nonumber\\
    &\le& e^{-T_k^{(N)}(\mathbf{x})} \cdot \rho_k \cdot D + \frac{e^{- T_k^{(N)}(\mathbf{x}) + D }}{2}\cdot D^2 \nonumber\\
    &=& e^{-T_k^{(N)}(\mathbf{x})} \cdot \left\{ O\left(D\right) + O\left(D^2\right) \right\} +  O_{\mathbf{x}}\left(D\right).
    \end{eqnarray}

    Therefore, we have
    \begin{equation}
    e^{-T_k^{(N)}(\mathbf{y})} = e^{-T_k^{(N)}(\mathbf{x})} + e^{-T_k^{(N)}(\mathbf{x})} \cdot \left\{ O\left(D\right) + O\left(D^2\right) \right\} + O_{\mathbf{x}}\left(D\right).
    \end{equation}

    This completes the proof.
\end{proof}

\begin{lemma}\label{lemma_e_T_K_upper_bound}
    For $\{T_k\}_{k=1}^{\infty}$ in Assumption E6, it holds that
    for $\mathbf{x} \in \Omega$ and $\|\mathbf{y} -  \mathbf{x}\| < D$,
    \begin{equation}
        e^{-T_k(\mathbf{y})} = e^{-T_k}(\mathbf{x}) + e^{-T_k(\mathbf{x})} \cdot\left\{ O\left(D\right) + O\left(D^2\right) \right\}  + O_{\mathbf{x}}\left(D\right). \label{Eq_lemma_e_T_K_upper_bound}
    \end{equation}
\end{lemma}

\begin{proof}[proof of Lemma \ref{lemma_e_T_K_upper_bound}]
    Note that, we use only the Lipschitz continuity of $T_k^{(N)}(\mathbf{x})$ to prove Lemma \ref{lemma_e_T_K_N_upper_bound}.
    From Assumption E6, $T_k(\mathbf{x})$ is Lipschitz continuous. Then, (\ref{Eq_lemma_e_T_K_upper_bound}) can be proven in a manner similar to Lemma \ref{lemma_e_T_K_N_upper_bound}.

    This completes the proof.
\end{proof}

\begin{lemma}\label{lemma_e_T_aster_upper_bound}
    Let $T_* = - \log dQ/dP$.
    Under Assumption E3 and E4, it holds that
    for $\mathbf{x} \in \Omega$ and $\|\mathbf{y} -  \mathbf{x}\| < D$,
    \begin{equation}
        e^{-T_*}(\mathbf{y}) = e^{-T_*(\mathbf{x})} - e^{-T_*(\mathbf{x})} \cdot \left\{  O \left(D\right) + O \left(D^2\right) \right\}  +  O_{\mathbf{x}} \left(D\right). \label{Eq_lemma_e_T_aster_upper_bound}
    \end{equation}
\end{lemma}

\begin{proof}[proof of Lemma \ref{lemma_e_T_aster_upper_bound}]
    Note that, we use only the Lipschitz continuity of $T_k^{(N)}(\mathbf{x})$ to prove Lemma \ref{lemma_e_T_K_N_upper_bound}.
    From Assumption E3 and Assumption E4,  $T_*(\mathbf{x})$ is a bounded continuous function on $\Omega$.
    Since bounded continuous functions are Lipschitz continuous, $T_*(\mathbf{x})$ is Lipschitz continuous.
    Thus, (\ref{Eq_lemma_e_T_aster_upper_bound}) can be proven in a manner similar to Lemma \ref{lemma_e_T_K_N_upper_bound}.

    This completes the proof.
\end{proof}

\begin{lemma}\label{Lemma_aprox_neighbor_of_x}
   Let $B(\mathbf{x}_0, D)=\{\mathbf{y}: \|\mathbf{y}- \mathbf{x}_0\| < D\}$.
   Then,
   \begin{eqnarray}
       p_{min} \cdot D^d \le P(B(\mathbf{x}_0, D)) \le p_{max} \cdot D^d.
   \end{eqnarray}
\end{lemma}

\begin{proof}[proof of Lemma \ref{Lemma_aprox_neighbor_of_x}]
   From Assumption E4, $p_{min} \le p(\mathbf{x}) \le p_{max}$ holds, and by integrating over $B(\mathbf{x}_0, D)$ with $\lambda$, we obtain
   \begin{eqnarray}
        \int_{B(\mathbf{x}_0, D)}  p_{min} \ d\lambda \le &P(B(\mathbf{x}_0, D))& \le \int_{B(\mathbf{x}_0, D)} p_{max}, \   d\lambda \nonumber \\
       \therefore \ p_{min} \cdot D^d \le   &P(B(\mathbf{x}_0, D))& \le p_{max} \cdot D^d.
   \end{eqnarray}

    This completes the proof.
\end{proof}

\begin{lemma}\label{Lemma_upper_bd_neighbor_of_x}
    Let $B(\mathbf{x}_0, D)=\{\mathbf{y}: \|\mathbf{y}- \mathbf{x}_0\| < D\}$.
    Then,
    \begin{eqnarray}
        P(B(\mathbf{x}_0, D)) = C \cdot p(\mathbf{x}_0) \cdot D^d,
    \end{eqnarray}
    where $C$ is constant.
\end{lemma}

\begin{proof}[proof of Lemma \ref{Lemma_upper_bd_neighbor_of_x}]
    From Assumption E4, $p$ is a bounded continuous function on $\Omega$.
    Since bounded continuous functions are Lipschitz continuous, $p(\mathbf{x})$ is Lipschitz continuous.

    Then, there exist a constant $C$ such that
    \begin{eqnarray}
        p(\mathbf{x}) \le p(\mathbf{\mathbf{x}_0}) + C\cdot \| \mathbf{x} - \mathbf{x}_0\|,
    \end{eqnarray}
    and by integrating over $B(\mathbf{x}_0, D)$ with $\lambda$, we obtain
    \begin{eqnarray}
        P(B(\mathbf{x}_0, D)) \le  C \cdot p({\mathbf{x}_0}) \cdot D^d.
    \end{eqnarray}

    This completes the proof.
\end{proof}

\begin{proposition}\label{proposition_early_stopping}
    For $\{T_k^{(N)}\}_{k=1}^{\infty}$ in Assumption E1, let $\hat{Q}^{(N)}_k$ be a probability defined as
    \begin{equation}
        d\hat{Q}^{(N)}_k =  e^{ - T_k^{(N)}} \cdot dP. \nonumber
    \end{equation}

    Then, under Assumpution E1-E6, for a sufficiently large $K > 0$, it holds that
    \begin{equation}
        E_{\mathbf{X}_P^{(N)}}[W_1(Q, \hat{Q}^{(N)}_{K})] \le 2 - N \cdot K^{- \frac{d}{2}} +  K^{-\frac{1}{2}}. \label{Eq_proposition_early_stopping_statement1}
    \end{equation}
\end{proposition}

\begin{corollary}\label{corollary_early_stopping_1}
    Let $K_0 =  \allowbreak N^{\frac{2}{d+\delta}}$ with $\delta > 0$.
    Then, under Assumpution E1-E6, for a sufficiently large $N$, it holds that
    \begin{equation}
        E_{\mathbf{X}_P^{(N)}}[W_1(Q, \hat{Q}^{(N)}_{K_0})] \le 2 - K_0^{\frac{\delta}{2}} + K_0^{-\frac{1}{2}}.
    \end{equation}
\end{corollary}

\begin{corollary}\label{corollary_early_stopping_2}
   In Corollary \ref{corollary_early_stopping_1}, let $\delta' > 0$ such that $N^{\frac{\delta'}{d+\delta'}} \allowbreak = \allowbreak 2$, and let $K_0  \allowbreak = \allowbreak N^{\frac{2}{d+\delta'}}$.
   Then, under Assumpution E1-E6, for a sufficiently large $N$, it holds that
    \begin{equation}
        E_{\mathbf{X}_P^{(N)}}[W_1(Q, \hat{Q}^{(N)}_{K_0})] \le K_0^{-\frac{1}{2}}.
    \end{equation}
   Thus, if $N > \left( \frac{1}{\varepsilon}\right)^{d+\delta'}$ then $E_{\mathbf{X}_P^{(N)}}[W_1({Q} ,\hat{{Q}}^{(N)}_{K_0})] < \varepsilon$.
\end{corollary}

\begin{proof}[proof of Proposition \ref{proposition_early_stopping}]
    Let $Q_K$ be a probability defined as
    \begin{equation}
        dQ_K =  e^{ - T_K} \cdot dP.
    \end{equation}
    Intuitively, $Q_K$ is the true balanced probability distribution at a step $K$.

    First, from the triangle inequality for the $L_{1}$ norm, we have
    \begin{equation}
        E_{\mu}\left|\hat{Q}^{(N)}_{K} - Q\right| \le E_{\mu}\Big| \hat{Q}^{(N)}_{K} - Q_K \Big| +  E_{\mu}\Big|Q_K - Q\Big|.  \label{Eq_proof_proposition_early_stopping_final_upper_0}
    \end{equation}
    Considering the expectation $E_{\mathbf{X}_P^{(N)}} [\cdot]$ for the both sides of the above equation, we see
    \begin{equation}
         E_{\mathbf{X}_P^{(N)}} \left[ E_{\mu}\left|\hat{Q}^{(N)}_{K} - Q\right| \right]  \le
              E_{\mathbf{X}_P^{(N)}} \left[ E_{\mu}\Big| \hat{Q}^{(N)}_{K} - Q_K \Big| \right]
              + E_{\mathbf{X}_P^{(N)}} \left[  E_{\mu}\Big|Q_K - Q\Big|\right].  \label{Eq_proof_proposition_early_stopping_L1_divided}
    \end{equation}

    Next, we obtain the upper bound of the first term in (\ref{Eq_proof_proposition_early_stopping_L1_divided}).

    Let $\Delta_i = B(\mathbf{X}_i, 1/\sqrt{K})$. Subsequently, let $\Delta = \bigcup_{i=1}^N \Delta_i$.
    Then, we have
    \begin{eqnarray}
        && E_{\mu}\left|  \hat{Q}^{(N)}_{K} - Q_K \right| \nonumber\\
        &=& \int \left| e^{ - T_k^{(N)}} \cdot \frac{dP}{d\mu} -  e^{-T_k}\cdot \frac{dP}{d\mu} \right|  d\mu \nonumber\\
        &=& \int \left| e^{ - T_k^{(N)}}  - e^{-T_k} \right| \frac{dP}{d\mu} d\mu \nonumber\\
        &=& E_{P}\left| e^{ - T_k^{(N)}}  - e^{-T_k} \right| \nonumber\\
        &=& E_{P}\left[ id_{\Delta } \left| e^{ - T_k^{(N)}}  - e^{-T_k} \right|  \right] +  E_{P}\left[ id_{\Omega \setminus \Delta } \left| e^{ - T_k^{(N)}}  - e^{-T_k} \right|  \right] \nonumber\\
        &\le& E_{P}\left[ id_{\Delta } \left| e^{ - T_k^{(N)}}  - e^{-T_k} \right|  \right]
              + E_{P}\left[ id_{\Omega \setminus \Delta } \left| e^{ - T_k^{(N)}} \right|  \right]
              + E_{P}\left[ id_{\Omega \setminus \Delta } \left| e^{-T_k} \right|  \right]. \nonumber
    \end{eqnarray}
    Considering the expectation $E_{\mathbf{X}_P^{(N)}} [\cdot]$ for the both sides of the above equation, we see
    \begin{eqnarray}
      &&  E_{\mathbf{X}_P^{(N)}} \left[ E_{\mu}\left|  \hat{Q}^{(N)}_{K} - Q_K \right|  \right] \nonumber\\
      &\le& E_{\mathbf{X}_P^{(N)}} \left[E_{P}\left[ id_{\Delta } \left| e^{ - T_k^{(N)}}  - e^{-T_k} \right|  \right] \right] \nonumber\\
      &&  \qquad + \  E_{\mathbf{X}_P^{(N)}} \left[E_{P}\left[ id_{\Omega \setminus \Delta } \left| e^{ - T_k^{(N)}} \right|  \right] \right]
          + E_{\mathbf{X}_P^{(N)}} \left[E_{P}\left[ id_{\Omega \setminus \Delta } \left| e^{-T_k} \right|  \right] \right].
          \label{Eq_proof_proposition_early_stopping_neghboor_divided}
    \end{eqnarray}

    To obtain the upper bound of the first term in (\ref{Eq_proof_proposition_early_stopping_neghboor_divided}), we see
    \begin{eqnarray}
       && E_{P}\left[ id_{\Delta } \left| e^{ - T_k^{(N)}}  - e^{-T_k} \right|  \right]  \nonumber\\
       &=& E_{P}\left[ \sum_{i=1}^N id_{\Delta_i} \left| e^{ - T_k^{(N)}}  - e^{-T_k} \right|  \right]  \nonumber\\
       &=&  \sum_{i=1}^N E_{P} \left[id_{\Delta_i} \left| e^{ - T_k^{(N)}}  -  e^{-T_k} \right|  \right]  \nonumber\\
       &=&  \sum_{i=1}^N E_{P} \left[  id_{B(\mathbf{X}_i, 1/\sqrt{K})} \left| e^{ - T_k^{(N)}}  -  e^{-T_k} \right|  \right] \nonumber \\
       &=&  E_{P}\left[\left|  \sum_{i=1}^N  id_{B(\mathbf{X}_i, 1/\sqrt{K})} \cdot e^{ - T_k^{(N)}}
            -  \sum_{i=1}^N  id_{B(\mathbf{X}_i, 1/\sqrt{K})} \cdot  e^{-T_k(\mathbf{X}_i)} \right|  \right]. \label{Eq_proof_proposition_early_stopping_local_i_divided_1}
    \end{eqnarray}
    Subsequently, we have
    \begin{eqnarray}
       &&  E_{P}\left[\left|  \sum_{i=1}^N  id_{B(\mathbf{X}_i, 1/\sqrt{K})} \cdot e^{ - T_k^{(N)}}
               -  \sum_{i=1}^N  id_{B(\mathbf{X}_i, 1/\sqrt{K})} \cdot  e^{-T_k(\mathbf{X}_i)} \right|  \right] \nonumber\\
       &=&   E_{P}\left[\left| \sum_{i=1}^N  id_{B(\mathbf{X}_i, 1/\sqrt{K})}(\mathbf{x})  \cdot e^{ - T_k^{(N)}(\mathbf{x})}
             -  \sum_{i=1}^N  id_{B(\mathbf{X}_i, 1/\sqrt{K})}(\mathbf{x})  \cdot e^{- T_k^{(N)}(\mathbf{X}_i)} \right. \right. \nonumber\\
       &&   \qquad \qquad + \ \sum_{i=1}^N  id_{B(\mathbf{X}_i, 1/\sqrt{K})}(\mathbf{x})\cdot e^{- T_k^{(N)}(\mathbf{X}_i)}
             -   \sum_{i=1}^N  id_{B(\mathbf{X}_i, 1/\sqrt{K})} (\mathbf{x}) \cdot e^{- T_*(\mathbf{X}_i)}  \nonumber\\
       &&   \qquad \qquad + \ \sum_{i=1}^N  id_{B(\mathbf{X}_i, 1/\sqrt{K})} (\mathbf{x}) \cdot e^{- T_*(\mathbf{X}_i)}
               - \sum_{i=1}^N  id_{B(\mathbf{X}_i, 1/\sqrt{K})} (\mathbf{x}) \cdot e^{-T_*(\mathbf{x}) }  \nonumber\\
       &&   \qquad \qquad + \left. \left. \sum_{i=1}^N  id_{B(\mathbf{X}_i, 1/\sqrt{K})} (\mathbf{x})  \cdot e^{- T_*(\mathbf{x})}
               - \sum_{i=1}^N  id_{B(\mathbf{X}_i, 1/\sqrt{K})}(\mathbf{x}) \cdot  e^{-T_k(\mathbf{x})} \right| \right] \nonumber\\
       &=&  E_{P}\left[\left|  \sum_{i=1}^N  id_{B(\mathbf{X}_i, 1/\sqrt{K})}(\mathbf{x}) \left\{  e^{ - T_k^{(N)}(\mathbf{x})}  -  e^{-T_k^{(N)}(\mathbf{X}_i)} \right\}  \right.  \right. \nonumber\\
       && \qquad \qquad + \ \sum_{i=1}^N  id_{B(\mathbf{X}_i, 1/\sqrt{K})} \left\{ e^{-T_k^{(N)}(\mathbf{X}_i)} -  e^{-T_*(\mathbf{X}_i)} \right\} \nonumber\\
       && \qquad \qquad + \ \sum_{i=1}^N  id_{B(\mathbf{X}_i, 1/\sqrt{K})} \left\{  e^{- T_*(\mathbf{X}_i)} -  e^{-T_*(\mathbf{x})} \right\} \nonumber\\
       && \qquad \qquad + \ \left. \left. \sum_{i=1}^N  id_{B(\mathbf{X}_i, 1/\sqrt{K})} (\mathbf{x})
           \left\{e^{- T_*(\mathbf{x})} - e^{-T_k(\mathbf{x})} \right\}  \right| \right]  \nonumber\\
       &\le&  E_{P}\left[ \sum_{i=1}^N  id_{B(\mathbf{X}_i, 1/\sqrt{K})}(\mathbf{x}) \left|   e^{ - T_k^{(N)}(\mathbf{x})}  -  e^{-T_k^{(N)}(\mathbf{X}_i)}   \right| \right] \nonumber\\
       && \qquad \qquad + \ E_{P}\left[ \sum_{i=1}^N  id_{B(\mathbf{X}_i, 1/\sqrt{K})}(\mathbf{x}) \left| e^{-T_k^{(N)}(\mathbf{X}_i)} -  e^{-T_*(\mathbf{X}_i)}   \right| \right] \nonumber\\
       && \qquad \qquad + \ E_{P}\left[\sum_{i=1}^N  id_{B(\mathbf{X}_i, 1/\sqrt{K})}(\mathbf{x}) \left| e^{- T_*(\mathbf{X}_i)} -  e^{-T_*(\mathbf{x})} \right| \right]  \nonumber\\
       && \qquad \qquad + E_{P}\left[ \sum_{i=1}^N  id_{B(\mathbf{X}_i, 1/\sqrt{K})} (\mathbf{x})
       \left| e^{- T_*(\mathbf{x})} - e^{-T_k(\mathbf{x})}  \right| \right]. \nonumber
    \end{eqnarray}

    Considering the expectation $E_{\mathbf{X}_P^{(N)}} [\cdot]$ for the both sides of the above equation, we obtain
    \begin{eqnarray}
        && E_{\mathbf{X}_P^{(N)}} \left[  E_{P}\left[ id_{\Delta } \left| e^{ - T_k^{(N)}}  - e^{-T_k} \right|  \right]  \right]   \nonumber\\
        &\le& E_{\mathbf{X}_P^{(N)}} \left[  E_{P}\left[ \sum_{i=1}^N  id_{B(\mathbf{X}_i, 1/\sqrt{K})}(\mathbf{x}) \left|   e^{ - T_k^{(N)}(\mathbf{x})}  -  e^{-T_k^{(N)}(\mathbf{X}_i)}   \right| \right] \right]  \nonumber\\
        && \qquad \qquad
            + \ E_{\mathbf{X}_P^{(N)}} \left[  \ E_{P}\left[ \sum_{i=1}^N  id_{B(\mathbf{X}_i, 1/\sqrt{K})}(\mathbf{x}) \left| e^{-T_k^{(N)}(\mathbf{X}_i)} -  e^{-T_*(\mathbf{X}_i)}   \right| \right]  \right] \nonumber\\
        && \qquad \qquad
            + \ E_{\mathbf{X}_P^{(N)}} \left[  \ E_{P}\left[\sum_{i=1}^N  id_{B(\mathbf{X}_i, 1/\sqrt{K})}(\mathbf{x}) \left| e^{- T_*(\mathbf{X}_i)} -  e^{-T_*(\mathbf{x})} \right| \right]   \right]  \nonumber\\
        && \qquad \qquad
            + \ E_{\mathbf{X}_P^{(N)}} \left[  E_{P}\left[ \sum_{i=1}^N  id_{B(\mathbf{X}_i, 1/\sqrt{K})} (\mathbf{x})
        \left| e^{- T_*(\mathbf{x})} - e^{-T_k(\mathbf{x})}  \right| \right]   \right].   \label{Eq_proof_proposition_early_stopping_local_i_divided_2}
    \end{eqnarray}

    Now, from Lemma \ref{lemma_e_T_K_upper_bound}, we have, for $\mathbf{x} \in id_{B(\mathbf{X}_i, 1/\sqrt{K})}$
    \begin{equation}
        \left|e^{- T_k^{(N)}(\mathbf{x})} - e^{-T_k^{(N)}(\mathbf{X}_i)} \right| =  e^{-T_k^{(N)}(\mathbf{X}_i)}  \left\{ O\left(\frac{1}{\sqrt{K}}\right) + O\left(\frac{1}{K}\right) \right\} + O_{\mathbf{X}_i}\left(\frac{1}{K}\right). \nonumber\\
    \end{equation}
    Then, we see
     \begin{eqnarray}
        &&  E_{P}\left[ \sum_{i=1}^N  id_{B(\mathbf{X}_i, 1/\sqrt{K})}(\mathbf{x}) \left|   e^{ - T_k^{(N)}(\mathbf{x})}  -  e^{-T_k^{(N)}(\mathbf{X}_i)}   \right| \right]  \nonumber\\
        &=& \sum_{i=1}^N  E_{P}\left[   id_{B(\mathbf{X}_i, 1/\sqrt{K})}(\mathbf{x})
                 \left\{  e^{-T_k^{(N)}(\mathbf{X}_i)}  \left\{ O\left(\frac{1}{\sqrt{K}}\right) + O\left(\frac{1}{K}\right) \right\} + O_{\mathbf{X}_i}\left(\frac{1}{\sqrt{K}} \right) \right\}\right] \nonumber\\
        &=& \sum_{i=1}^N  E_{P}\left[id_{B(\mathbf{X}_i, 1/\sqrt{K})}(\mathbf{x})  \right]
              \left\{  e^{-T_k^{(N)}(\mathbf{X}_i)}  \left\{ O\left(\frac{1}{\sqrt{K}}\right) + O\left(\frac{1}{K}\right) \right\} + O_{\mathbf{X}_i}\left(\frac{1}{\sqrt{K}} \right)  \right\}
         \nonumber\\
        &=& \sum_{i=1}^N P\left(B(\mathbf{X}_i, 1/\sqrt{K})\right)
             \left\{  e^{-T_k^{(N)}(\mathbf{X}_i)}  \left\{ O\left(\frac{1}{\sqrt{K}}\right) + O\left(\frac{1}{K}\right) \right\} + O_{\mathbf{X}_i}\left(\frac{1}{\sqrt{K}}\right) \right\} \nonumber\\
        &=&  \sum_{i=1}^N P\left(B(\mathbf{X}_i, 1/\sqrt{K})\right) \cdot O_{\mathbf{X}_i}\left(\frac{1}{\sqrt{K}}\right) \nonumber\\
        &\le&  \sum_{i=1}^N  p_{max} \cdot  O_{\mathbf{X}_i}\left(\frac{1}{\left(\sqrt{K} \right)^{d}}\right) \cdot O_{\mathbf{X}_i}\left(\frac{1}{\sqrt{K}}\right)
            \label{proof_proposition_early_stopping_upper_bound_local_sum_k_N_using_lemma_Lemma_aprox_neighbor_of_x} \\
        &=&  \sum_{i=1}^N O_{\mathbf{X}_i} \left( K^{-\frac{d+1}{2}}\right). \nonumber
    \end{eqnarray}
    Here, we obtain (\ref{proof_proposition_early_stopping_upper_bound_local_sum_k_N_using_lemma_Lemma_aprox_neighbor_of_x}) by using Lemma \ref{Lemma_aprox_neighbor_of_x}.

    Considering the expectation $E_{\mathbf{X}_P^{(N)}} [\cdot]$ for the both sides of the above equation, we have
    \begin{eqnarray}
        &&E_{\mathbf{X}_P^{(N)}} \left[ E_{P}\left[ \sum_{i=1}^N  id_{B(\mathbf{X}_i, 1/\sqrt{K})}(\mathbf{x}) \left|   e^{ - T_k^{(N)}(\mathbf{x})}  -  e^{-T_k^{(N)}(\mathbf{X}_i)}   \right| \right]  \right] \nonumber\\
        &=&  E_{\mathbf{X}_P^{(N)}} \left[\sum_{i=1}^N O_{\mathbf{X}_i} \left( K^{-\frac{d+1}{2}}\right)\right] \nonumber \\
        &=&  \sum_{i=1}^N O \left( K^{-\frac{d+1}{2}}\right) \nonumber \\
        &=&  N \cdot  O\left( K^{-\frac{d+1}{2}}\right). \label{proof_proposition_early_stopping_upper_bound_local_sum_k_N_contin}
    \end{eqnarray}

    In addition, from Lemma \ref{lemma_e_T_K_upper_bound} and \ref{lemma_e_T_aster_upper_bound}, it holds that, for $\mathbf{x} \in id_{B(\mathbf{X}_i, 1/\sqrt{K})}$,
    \begin{equation}
        \left|e^{- T_k(\mathbf{x})} - e^{-T_k(\mathbf{X}_i)} \right|
        =  e^{-T_k(\mathbf{X}_i)}  \left\{ O\left(\frac{1}{\sqrt{K}}\right) + O\left(\frac{1}{K}\right) \right\} + O_{\mathbf{X}_i}\left(\frac{1}{K}\right), \nonumber
    \end{equation}
    and
    \begin{equation}
        \left|e^{- T_*(\mathbf{x})} - e^{-T_*(\mathbf{X}_i)} \right|
        =  e^{-T_*(\mathbf{X}_i)}  \left\{ O\left(\frac{1}{\sqrt{K}}\right) + O\left(\frac{1}{K}\right) \right\} + O_{\mathbf{X}_i}\left(\frac{1}{K}\right).  \nonumber
    \end{equation}

    In the similar manner to obtain (\ref{proof_proposition_early_stopping_upper_bound_local_sum_k_N_contin}), it holds that, for $\mathbf{x} \in id_{B(\mathbf{X}_i, 1/\sqrt{K})}$,
    we have
    \begin{equation}
        E_{\mathbf{X}_P^{(N)}} \left[
            E_{P}\left[ \sum_{i=1}^N  id_{B(\mathbf{X}_i, 1/\sqrt{K})}(\mathbf{x}) \left| e^{- T_k(\mathbf{x})}  -  e^{-T_k(\mathbf{X}_i)} \right| \right]  \right]
          = N \cdot  O \left( K^{-\frac{d+1}{2}}\right), \label{proof_proposition_early_stopping_upper_bound_local_sum_k_true}
    \end{equation}
    and
    \begin{equation}
         E_{\mathbf{X}_P^{(N)}} \left[E_{P}\left[ \sum_{i=1}^N  id_{B(\mathbf{X}_i, 1/\sqrt{K})}(\mathbf{x}) \left| e^{- T_*(\mathbf{x})} - e^{-T_*(\mathbf{X}_i)} \right|\right]\right]
          =  N \cdot  O \left( K^{-\frac{d+1}{2}}\right).  \label{proof_proposition_early_stopping_upper_bound_local_sum_true}
    \end{equation}
    Here, we obtain the upper bounds of the first, third, and fourth terms in (\ref{Eq_proof_proposition_early_stopping_local_i_divided_2}).

    We now have the upper bounds of the second term in (\ref{Eq_proof_proposition_early_stopping_local_i_divided_2}).

    First, we obtain
    \begin{eqnarray}
        && E_{P}\left[ \sum_{i=1}^N  id_{B(\mathbf{X}_i, 1/\sqrt{K})}(\mathbf{x}) \left| e^{-T_k^{(N)}(\mathbf{X}_i)} -  e^{-T_*(\mathbf{X}_i)}   \right| \right] \nonumber\\
        &=& \sum_{i=1}^N  E_{P}\left[  id_{B(\mathbf{X}_i, 1/\sqrt{K})}(\mathbf{x})  \right]  \left| e^{-T_k^{(N)}(\mathbf{X}_i)} -  e^{-T_*(\mathbf{X}_i)} \right| \nonumber\\
        &=& \sum_{i=1}^N  P\left(B(\mathbf{X}_i, 1/\sqrt{K} \right) \left| e^{-T_k^{(N)}(\mathbf{X}_i)} -  e^{-T_*(\mathbf{X}_i)} \right|. \label{Eq_proof_proposition_early_stopping_upper_bnd_T_K_N}
    \end{eqnarray}
    Then, from Lemma \ref{Lemma_upper_bd_neighbor_of_x}, we have
    \begin{eqnarray}
        &&  \sum_{i=1}^N P\left(B(\mathbf{X}_i, 1/\sqrt{K} \right)\left| e^{-T_k^{(N)}(\mathbf{X}_i)} -  e^{-T_*(\mathbf{X}_i)} \right| \nonumber\\
        &=&  C \, \sum_{i=1}^N p\left(\mathbf{X}_i\right)\cdot  O\left(\frac{1}{\left(\sqrt{K} \right)^{d}}\right)   \left| e^{-T_k^{(N)}(\mathbf{X}_i)} -  e^{-T_*(\mathbf{X}_i)} \right|.
        \label{Eq_proof_proposition_early_stopping_T_K_n_conv_0}
    \end{eqnarray}
    Next, note that,
    \begin{equation}
       \sum_{i=1}^N p\left(\mathbf{X}_i \right) \cdot \left| e^{-T_k^{(N)}(\mathbf{X}_i)} -  e^{-T_*(\mathbf{X}_i)} \right| = N \cdot E_{\nu}\left|\hat{Q}_k^{(N)} - Q^{(N)} \right|. \label{Eq_proof_proposition_early_stopping_T_K_n_conv_1}
    \end{equation}
    Here, $\nu$ is the countable measure on $\left\{\mathbf{X}_1, \mathbf{X}_2, \ldots, \mathbf{X}_N\right\}$ defined as (\ref{Eq_proof_proposition_representation_of_density_ratios_between_emprical_distributions_countable_measure_df}).

    In addition, since
    \begin{equation}
        \hat{L}_{\alpha}^{(N)}(Q, P;T) = L_{\alpha}(\hat{Q}^{(N)}, \hat{P}^{(N)};T) \nonumber\\ \label{Eq_proof_proposition_early_stopping_T_K_L_L_empi}
    \end{equation}
    holds, we obtain, from Propsition \ref{proposition_variational_upper_bound_by_alpha_div} and Assumption E1,
    \begin{eqnarray}
        \sqrt{\frac{\alpha}{2}} \cdot E_{\nu}\Big| \hat{Q}_K^{(N)} - Q^{(N)} \Big|
         &\le& \sqrt{L_{\alpha}(Q^{(N)}, P^{(N)}; T_K^{(N)}) -  L_{\alpha}(Q^{(N)}, P^{(N)}; T_*)} \nonumber \\
         &=& \sqrt{\hat{L}_{\alpha}^{(N)}(Q, P; T_K^{(N)}) -  \hat{L}_{\alpha}^{(N)}(Q, P; T_*)} \nonumber \\
         &=& O\left( \frac{1}{\sqrt{K}} \right). \label{Eq_proof_proposition_early_stopping_T_K_n_conv_2}
    \end{eqnarray}
    Finally, from (\ref{Eq_proof_proposition_early_stopping_upper_bnd_T_K_N}),
    (\ref{Eq_proof_proposition_early_stopping_T_K_n_conv_0}), (\ref{Eq_proof_proposition_early_stopping_T_K_n_conv_1}), and (\ref{Eq_proof_proposition_early_stopping_T_K_n_conv_2}), we have
    \begin{eqnarray}
        && E_{P}\left[ \sum_{i=1}^N  id_{B(\mathbf{X}_i, 1/\sqrt{K})}(\mathbf{x}) \left| e^{-T_k^{(N)}(\mathbf{X}_i)} -  e^{-T_*(\mathbf{X}_i)}   \right| \right] \nonumber\\
        &\le&  O\left(\frac{1}{\left(\sqrt{K} \right)^{d}}\right) \sum_{i=1}^N p\left(\mathbf{X}_i\right) \left| e^{-T_k^{(N)}(\mathbf{X}_i)} -  e^{-T_*(\mathbf{X}_i)} \right|. \nonumber\\
        &\le&  O\left(\frac{1}{\left(\sqrt{K} \right)^{d}}\right) \cdot  N \cdot E_{\nu}\left|\hat{Q}_k^{(N)} - Q \right|\nonumber\\
        &\le&  O\left(\frac{1}{\left(\sqrt{K} \right)^{d}}\right) \cdot  N \cdot O\left( \frac{1}{\sqrt{K}} \right) \nonumber\\
        &=& N \cdot O\left(K^{- \frac{d+1}{2}} \right).
    \end{eqnarray}
    Considering the expectation $E_{\mathbf{X}_P^{(N)}} [\cdot]$ for the both sides of the above equation, we have
    \begin{eqnarray}
        &&E_{\mathbf{X}_P^{(N)}} \left[ E_{P}\left[ \sum_{i=1}^N  id_{B(\mathbf{X}_i, 1/\sqrt{K})}(\mathbf{x}) \left| e^{-T_k^{(N)}(\mathbf{X}_i)} -  e^{-T_*(\mathbf{X}_i)}   \right| \right] \right] \nonumber\\
        &=&  E_{\mathbf{X}_P^{(N)}} \left[N \cdot  O\left( K^{-\frac{d+1}{2}}\right)\right] \nonumber \\
        &=&  N \cdot  O\left( K^{-\frac{d+1}{2}}\right). \label{proof_proposition_early_stopping_upper_bound_local_sum_t_k_N_each}
    \end{eqnarray}

    Summarizing (\ref{Eq_proof_proposition_early_stopping_local_i_divided_2}),
    (\ref{proof_proposition_early_stopping_upper_bound_local_sum_k_N_contin}),
    (\ref{proof_proposition_early_stopping_upper_bound_local_sum_k_true}),
    (\ref{proof_proposition_early_stopping_upper_bound_local_sum_true}),
    and (\ref{proof_proposition_early_stopping_upper_bound_local_sum_t_k_N_each}), we obtain
   \begin{eqnarray}
        && E_{\mathbf{X}_P^{(N)}} \left[ E_{P}\left[ id_{\Delta } \left| e^{ - T_k^{(N)}}  - e^{-T_k} \right|  \right]   \right] \nonumber\\
        &\le& E_{\mathbf{X}_P^{(N)}} \left[ E_{P}\left[ \sum_{i=1}^N  id_{B(\mathbf{X}_i, 1/\sqrt{K})}(\mathbf{x}) \left|   e^{ - T_k^{(N)}(\mathbf{x})}  -  e^{-T_k^{(N)}(\mathbf{X}_i)}   \right| \right]\right] \nonumber\\
        && \qquad \qquad + E_{\mathbf{X}_P^{(N)}} \left[\ E_{P}\left[ \sum_{i=1}^N  id_{B(\mathbf{X}_i, 1/\sqrt{K})}(\mathbf{x}) \left| e^{-T_k^{(N)}(\mathbf{X}_i)} -  e^{-T_*(\mathbf{X}_i)}   \right| \right] \right]\nonumber\\
        && \qquad \qquad + E_{\mathbf{X}_P^{(N)}} \left[\ E_{P}\left[\sum_{i=1}^N  id_{B(\mathbf{X}_i, 1/\sqrt{K})}(\mathbf{x}) \left| e^{- T_*(\mathbf{X}_i)} -  e^{-T_*(\mathbf{x})} \right| \right] \right] \nonumber\\
        && \qquad \qquad + E_{\mathbf{X}_P^{(N)}} \left[ E_{P}\left[ \sum_{i=1}^N  id_{B(\mathbf{X}_i, 1/\sqrt{K})} (\mathbf{x})
        \left| e^{- T_*(\mathbf{x})} - e^{-T_k(\mathbf{x})}  \right| \right] \right]  \nonumber\\
        &=&  N \cdot O \left(K^{- \frac{d+1}{2}} \right) +  N \cdot O \left(K^{- \frac{d+1}{2}} \right)
             + N \cdot O \left(K^{- \frac{d+1}{2}} \right) +  N \cdot O \left(K^{- \frac{d+1}{2}} \right) \nonumber\\
        &=&  N \cdot O \left(K^{- \frac{d+1}{2}} \right).  \label{Eq_proof_proposition_early_stopping_upper_bound_of_first_term_upper_bound}
    \end{eqnarray}
    Here, we see the upper bound of the first term in (\ref{Eq_proof_proposition_early_stopping_neghboor_divided}).

    Next, we obtain the upper bound of the second and third term in (\ref{Eq_proof_proposition_early_stopping_neghboor_divided}).

    First,  we obtain the upper bound of the second term in (\ref{Eq_proof_proposition_early_stopping_neghboor_divided}).
    Now, we have
    \begin{eqnarray}
        && E_{\mathbf{X}_P^{(N)}} \left[E_{P}\left[ id_{\Omega \setminus \Delta } \left| e^{-T_k} \right|  \right] \right] \nonumber\\
        &=&  E_{\mathbf{X}_P^{(N)}} \left[ E_{P}\left[ id_{\Omega \setminus \bigcup_{i=1}^N \Delta_i } (\mathbf{x}) \cdot \left| e^{ - T_k^{(N)}} \right|  \right]  \right] \nonumber\\
        &=& E_{\mathbf{X}_P^{(N)}} \left[ E_{P}\left[ \left\{1 - id_{\bigcup_{i=1}^N \Delta_i }(\mathbf{x}) \right\} \cdot e^{ - T_k^{(N)}} \right]  \right] \nonumber\\
        &=& E_{\mathbf{X}_P^{(N)}} \left[  1 -  E_{P}\left[\sum_{i=1}^N id_{B(\mathbf{X}_i, 1/\sqrt{K})} (\mathbf{x}) \cdot  e^{ - T_k^{(N)}}  \right]  \right] \nonumber\\
        &=& E_{\mathbf{X}_P^{(N)}} \left[  1 -  \sum_{i=1}^N E_{P}\left[ id_{B(\mathbf{X}_i, 1/\sqrt{K})} (\mathbf{x}) \cdot  e^{ - T_k^{(N)}}  \right]  \right].
        \label{Eq_proof_proposition_early_stopping_for_second_term_upper_bound_3}
    \end{eqnarray}
     Then, from Lemma \ref{lemma_e_T_K_N_upper_bound} and Lemma \ref{Lemma_aprox_neighbor_of_x}, we have
     \begin{eqnarray}
         && E_{P} \left[id_{B(\mathbf{X}_i, 1/\sqrt{K})}(\mathbf{x}) \cdot e^{ - T_k^{(N)}(\mathbf{x})} \right]\nonumber\\
         &=& E_{P} \left[ id_{B\left(\mathbf{X}_i, 1/\sqrt{K}\right)} (\mathbf{x})
         \cdot \left\{ e^{-T_k^{(N)}(\mathbf{X}_i)} + e^{-T_k^{(N)}(\mathbf{X}_i)}
         \cdot \left\{ O \left( \frac{1}{\sqrt{K}}\right) + O \left( \frac{1}{K} \right) \right\} +  O_{\mathbf{x}} \left(\frac{1}{\sqrt{K}}\right) \right\}\right] \nonumber\\
          &=& E_{P} \left[
          id_{ B(\mathbf{X}_i, 1/\sqrt{K})}(\mathbf{x}) \right] \cdot \left\{ e^{-T_k^{(N)}(\mathbf{X}_i)} + e^{-T_k^{(N)}(\mathbf{X}_i)}
          \cdot \left\{ O \left( \frac{1}{\sqrt{K}}\right) + O \left( \frac{1}{K} \right) \right\} +  O_{\mathbf{x}} \left(\frac{1}{\sqrt{K}}\right) \right\}  \nonumber\\
          &=& P\left( B(\mathbf{X}_i, 1/\sqrt{K}) \right) \cdot \left\{ e^{-T_k^{(N)}(\mathbf{X}_i)} + e^{-T_k^{(N)}(\mathbf{X}_i)}
          \cdot \left\{ O \left( \frac{1}{\sqrt{K}}\right) + O \left( \frac{1}{K} \right) \right\} +  O_{\mathbf{x}} \left(\frac{1}{\sqrt{K}}\right) \right\} \nonumber \\
          &\ge&  p_{min} \cdot   O_{\mathbf{X}_i}\left(\frac{1}{\left(\sqrt{K} \right)^{d}}\right)  \cdot \left\{ e^{-T_k^{(N)}(\mathbf{X}_i)} + e^{-T_k^{(N)}(\mathbf{X}_i)}
          \cdot \left\{ O \left( \frac{1}{\sqrt{K}}\right) + O \left( \frac{1}{K} \right) \right\} +  O_{\mathbf{x}} \left(\frac{1}{\sqrt{K}}\right) \right\} \nonumber \\
          &=&  O_{\mathbf{X}_i}\left(K^{- \frac{d}{2}}\right) + O_{\mathbf{X}_i}\left(K^{- \frac{d+1}{2}}\right). \label{Eq_proof_proposition_early_stopping_for_second_term_upper_bound_2}
     \end{eqnarray}

    From (\ref{Eq_proof_proposition_early_stopping_for_second_term_upper_bound_3}) and (\ref{Eq_proof_proposition_early_stopping_for_second_term_upper_bound_2}), we see
    \begin{eqnarray}
        &&E_{\mathbf{X}_P^{(N)}} \left[E_{P}\left[ id_{\Omega \setminus \Delta } \left| e^{-T_k^{(N)}} \right|  \right] \right]      \nonumber\\
        &=& E_{\mathbf{X}_P^{(N)}} \left[  1 -  \sum_{i=1}^N E_{P}\left[ id_{B(\mathbf{X}_i, 1/\sqrt{K})} (\mathbf{x}) \cdot  e^{ - T_k^{(N)}}  \right]  \right]  \nonumber\\
        &\le& E_{\mathbf{X}_P^{(N)}} \left[  1 -  \sum_{i=1}^N O_{\mathbf{X}_i}\left(K^{- \frac{d}{2}}\right) + O_{\mathbf{X}_i}\left(K^{- \frac{d+1}{2}}\right) \right]  \nonumber\\
        &=&   1 -  N\cdot \left\{ O\left(K^{- \frac{d}{2}}\right) + O\left(K^{- \frac{d+1}{2}}\right)   \right\} \nonumber\\
        &=&   1 -  N\cdot O\left(K^{- \frac{d}{2}}\right) \label{Eq_proof_proposition_early_stopping_upper_bound_of_second_term}
    \end{eqnarray}

    In the similar manner to obtain (\ref{Eq_proof_proposition_early_stopping_upper_bound_of_second_term}),
    we obtain the upper bound of the third term in (\ref{Eq_proof_proposition_early_stopping_neghboor_divided}):
    \begin{equation}
       E_{\mathbf{X}_P^{(N)}} \left[E_{P}\left[ id_{\Omega \setminus \Delta } \left| e^{-T_k} \right|  \right] \right] = 1 -  N \cdot O\left(K^{- \frac{d}{2}}\right).
       \label{Eq_proof_proposition_early_stopping_upper_bound_of_third_term}
    \end{equation}

    Summarizing (\ref{Eq_proof_proposition_early_stopping_neghboor_divided}), (\ref{Eq_proof_proposition_early_stopping_upper_bound_of_first_term_upper_bound}),  (\ref{Eq_proof_proposition_early_stopping_upper_bound_of_second_term}) and (\ref{Eq_proof_proposition_early_stopping_upper_bound_of_third_term}), we have
    \begin{eqnarray}
        && E_{\mathbf{X}_P^{(N)}} \left[E_{\mu}\left|  \hat{Q}^{(N)}_{K} - Q_K \right| \right] \nonumber\\
        &\le& N \cdot O \left(K^{- \frac{d+1}{2}} \right) +  \left\{1 -  N \cdot O\left(K^{- \frac{d}{2}}\right)  \right\}
         +  \left\{1 -  N \cdot O\left(K^{- \frac{d}{2}}\right)  \right\} \nonumber\\
        &=& 2 - N\cdot O\left(K^{- \frac{d}{2}}\right). \label{Eq_proof_proposition_early_stopping_final_upper_1}
    \end{eqnarray}

    For the upper bound of the second term in (\ref{Eq_proof_proposition_early_stopping_L1_divided}), from Propsition \ref{proposition_variational_upper_bound_by_alpha_div} we obtain
    \begin{equation}
      \sqrt{\frac{\alpha}{2}} E_{\mu}\Big| Q_K - Q \Big|  \le  \sqrt{L_{\alpha}(Q, P; T_K) -  L_{\alpha}(Q, P;  T_*)}.
    \end{equation}
    Thus, under Assumption E2, we see
    \begin{equation}
      E_{\mu}\Big| Q_K - Q \Big|  \le  \frac{C_0'}{\sqrt{K}},
    \end{equation}
    where $C_0'=\sqrt{(2\cdot C_0)/ \alpha}$.

    Considering the expectation $E_{\mathbf{X}_P^{(N)}} [\cdot]$ for the both sides of the above equation, we have
    \begin{equation}
        E_{\mathbf{X}_P^{(N)}} \left[ E_{\mu}\Big| Q_K - Q \Big|\right] = O\left( K^{-\frac{1}{2}}\right).
        \label{Eq_proof_proposition_early_stopping_final_upper_2}
    \end{equation}

    Finally, (\ref{Eq_proof_proposition_early_stopping_final_upper_0}), (\ref{Eq_proof_proposition_early_stopping_final_upper_1}), and (\ref{Eq_proof_proposition_early_stopping_final_upper_2}),
    we have
    \begin{eqnarray}
           E_{\mathbf{X}_P^{(N)}}[W_1(Q, \hat{Q}^{(N)}_{K_0})] &\le& E_{\mathbf{X}_P^{(N)}} \left[ \mathrm{diam}(\Omega) \cdot E_{\mu}\left|\hat{Q}^{(N)}_{K} - Q\right|  \right] \nonumber \\
           &=& E_{\mathbf{X}_P^{(N)}} \left[E_{\mu}\left|\hat{Q}^{(N)}_{K} - Q\right|  \right]  \nonumber \\
           &\le& E_{\mathbf{X}_P^{(N)}} \left[E_{\mu}\Big| \hat{Q}^{(N)}_{K} - Q_K \Big| \right] + E_{\mathbf{X}_P^{(N)}} \Big[ E_{\mu}\left|Q_K - Q\right| \Big] \nonumber\\
           &=&  2 -  N \cdot O\left(K^{- \frac{d}{2}}\right)  +  O\left( K^{-\frac{1}{2}}\right).  \nonumber\\
    \end{eqnarray}

    From this, for sufficiently large $K > 0$, we see
    \begin{equation}
        E_{\mathbf{X}_P^{(N)}}[W_1(Q, \hat{Q}^{(N)}_{K_0})] \le 2 - N \cdot K^{- \frac{d}{2}} +  K^{-\frac{1}{2}}. \nonumber
    \end{equation}
    Here, we show (\ref{Eq_proposition_early_stopping_statement1}).

    This completes the proof.
\end{proof}

\begin{proof}[proof of Corollary \ref{corollary_early_stopping_1}]
For \ref{Eq_proposition_early_stopping_statement1}, substituting $K_0$ for $K$, and $N=K_0^\frac{d+\delta}{2}$,
we have
\begin{eqnarray}
    E_{\mathbf{X}_P^{(N)}}[W_1(Q, \hat{Q}^{(N)}_{K_0})] &\le& 2 -  K_0^\frac{d+\delta}{2} \cdot  K_0^{- \frac{d}{2}} +  K_0^{-\frac{1}{2}} \nonumber\\
    &=& 2 - K_0^\frac{\delta}{2} + K_0^{-\frac{1}{2}}. \label{Eq_proposition_early_stopping_statement2}
\end{eqnarray}

This completes the proof.
\end{proof}

\begin{proof}[proof of Corollary \ref{corollary_early_stopping_2}]
For the setting of the proposition, we have $K_0^{\frac{\delta'}{2}} = N^{\frac{\delta'}{d+\delta'}} = 2$.
Thus, for \ref{Eq_proposition_early_stopping_statement2}, we see
\begin{eqnarray}
    E_{\mathbf{X}_P^{(N)}}[W_1(Q, \hat{Q}^{(N)}_{K_0})] &\le& 2 -  K_0^\frac{d+\delta}{2} \cdot  K_0^{- \frac{d}{2}} +  K_0^{-\frac{1}{2}} \nonumber\\
    &=& 2 - 2 + K_0^{-\frac{1}{2}} \nonumber \\
    &=&  K_0^{-\frac{1}{2}}. \nonumber
\end{eqnarray}

This completes the proof.
\end{proof}

\newpage
\section{Neumerical Experiments}\label{section_neumerical_experiments}
In this section, we report the results of numerical experiments conducted in this study.

\subsection{Experiments on convergence for different values of $\alpha$}\label{subsection_neumerical_experiments_convegence_various_alphas}
In this section, we report the results of the numerical experiments related to the discussion in Section \ref{Section_Estimation_of_Balancing_Weights}:
the results of the numerical experiments on the convergence of learning for different values of $\alpha$ are presented.

\paragraph{Experimental Setup.}
For $\alpha=-3, -2, -1, 0.2, 0.5, 0,8, 2.0, 3.0$, and $4.0$, we generated training and test dataset, and then trained an NGB model with the training dataset while estimating the $\alpha$ divergence at each learning step with the test dataset.
One hundred numerical simulations were performed for each $\alpha$.
As a result of the experiment, the median of the estimated value and ranges between the 45th and 55th percentile quartiles and between the 5th and 95th percentile quartiles at each learning step are reported.

\paragraph{Synthetic Data.}
We generated synthetic data of size 5000 from 5-dimantional normal distribution $\{X_1, X_2, \ldots, X_5\}$ such that $E[X_i]=0$, $\mathrm{Var}[X_i] = 1$ and $E[X_i\cdot X_j]=0.8$ ($i \neq j$),
for each of the training and test datasets.

\paragraph{Estimating the $\alpha$ divergence.}
The $\alpha$ divergence was estimated in the following way
\begin{eqnarray}
    \hat{D}_{\alpha}(Q||P)(t) &=& \frac{1}{\alpha \cdot(1-\alpha)} - \frac{1}{\alpha} \hat{E}_Q\left[e^{\alpha \cdot T_{\theta_t}(\mathbf{x}^{te})}\right] -
    \frac{1}{1- \alpha}  \hat{E}_P\left[e^{(\alpha - 1) \cdot T_{\theta_t}(\mathbf{x}^{te})}\right] \label{Eq_subsection_neumerical_experiments_convegence_various_alphas_loss_0}\\
    &=& \frac{1}{\alpha \cdot(1-\alpha)} - \mathcal{L}_{\alpha}(\theta_t), \label{Eq_subsection_neumerical_experiments_convegence_various_alphas_loss_1}
\end{eqnarray}
where $T_{\theta_t}$ is a model at learning step $t$ in Algorithm \ref{algo_train} and $\mathbf{x}^{te}$ denotes the test dataset.
Note that, decreasing of the estimated divergence $\hat{D}_{\alpha}(Q||P)(t)$ in (\ref{Eq_subsection_neumerical_experiments_convegence_various_alphas_loss_0}) implies increasing of the loss $\mathcal{L}_{\alpha}(\theta_t)$ in (\ref{Eq_subsection_neumerical_experiments_convegence_various_alphas_loss_1}).

\paragraph{Implementation and Training Details.}
We used a neural network which has 3 hidden layers of 100 units in each layer.
The Adam algorithm in PyTorch was used.
For the hyperparameters in the training, the learning rate was 0.001, BathSize was 2500, and the number of epochs was 500.
A NVDIA Tesla K80 GPU was used. It took approximately four hours to conduct all simulations for each value of $\alpha$.

\begin{figure}
    \centering
    \begin{tabular}{cc}
        \begin{minipage}[t]{0.55\linewidth}
            \includegraphics[keepaspectratio, scale=0.29]{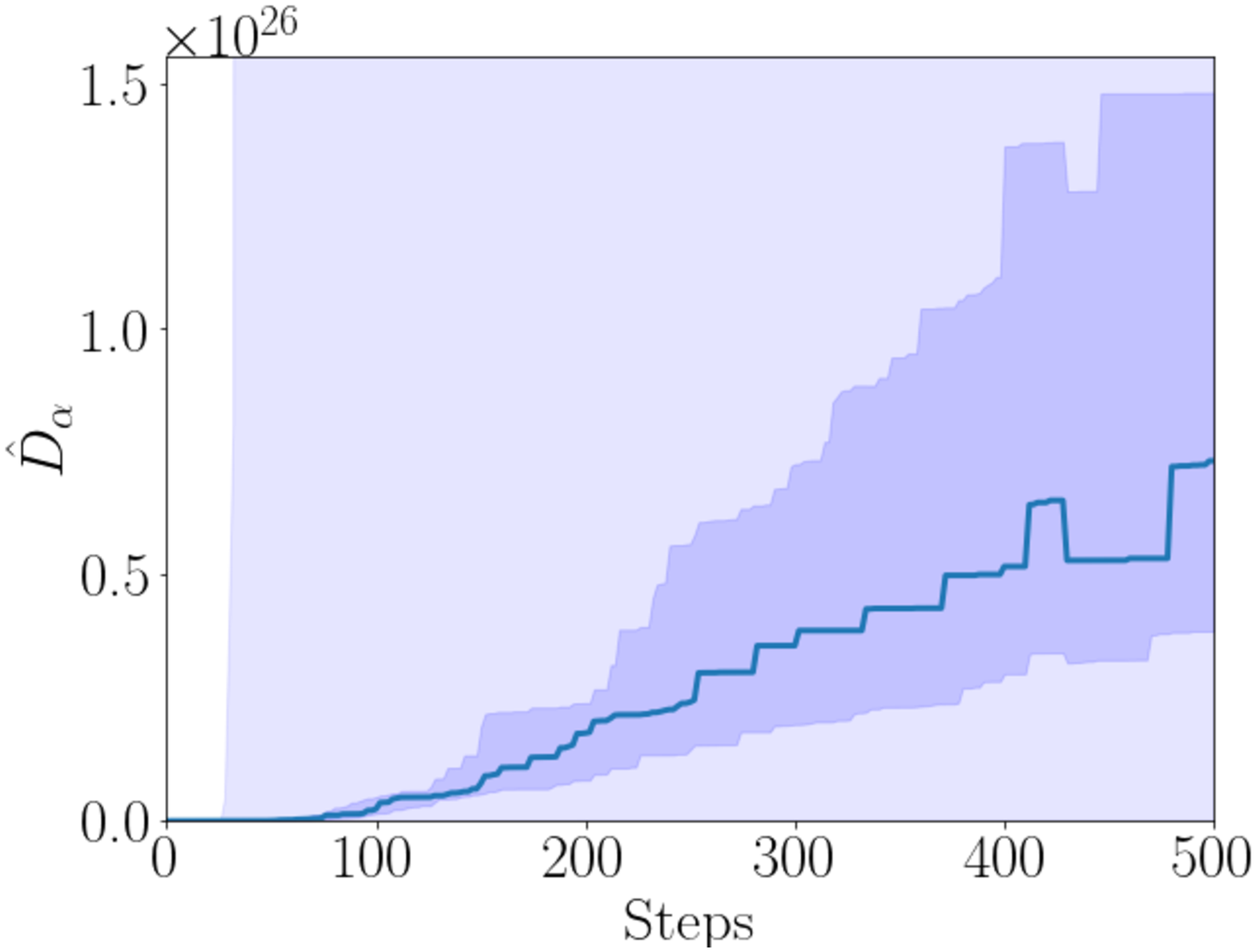}
            \subcaption{$\alpha=-3$}
        \end{minipage} &
        \begin{minipage}[t]{0.55\linewidth}
            \includegraphics[keepaspectratio, scale=0.29]{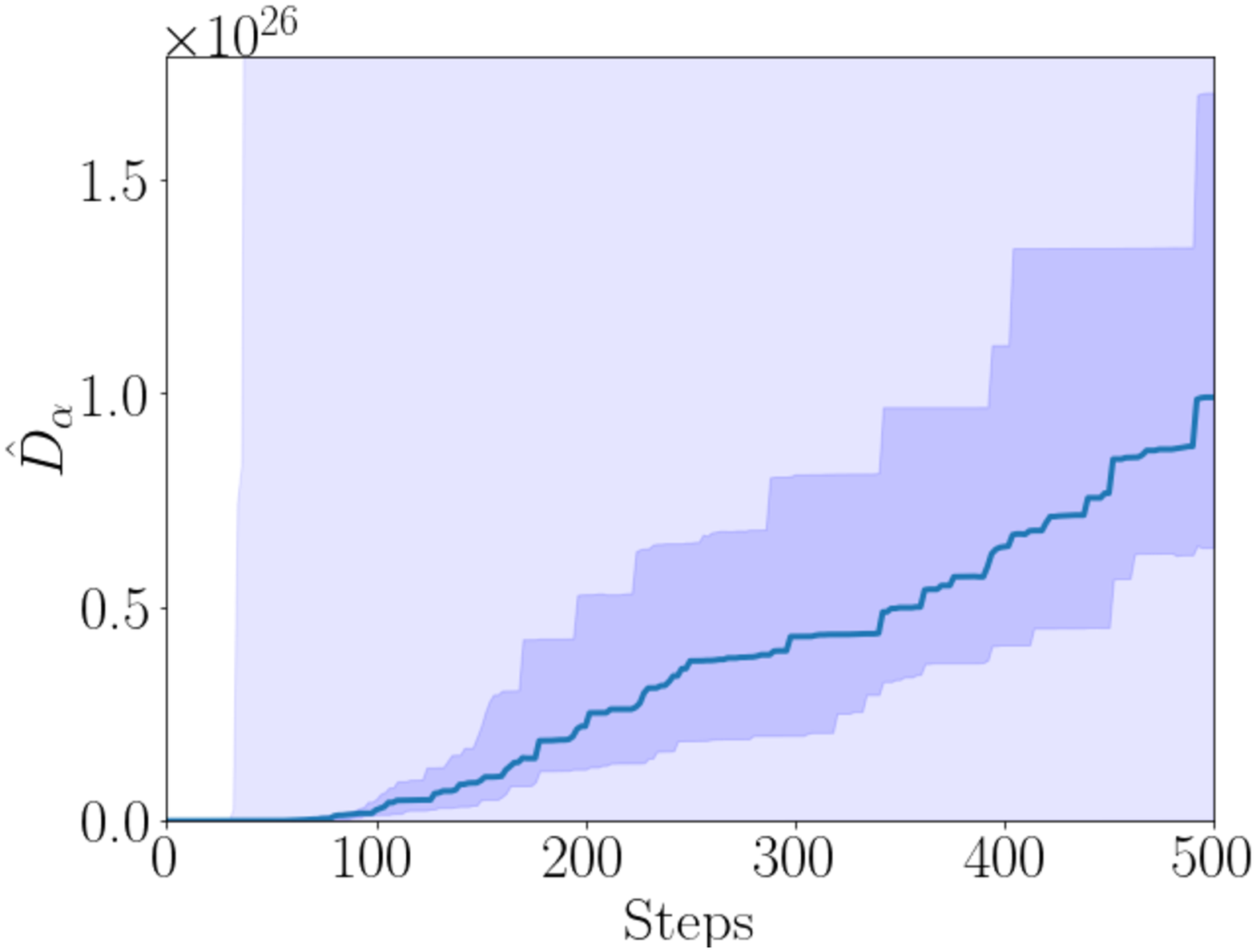}
            \subcaption{$\alpha=-2$}
        \end{minipage} \\
        \begin{minipage}[t]{0.55\hsize}
            \includegraphics[keepaspectratio, scale=0.29]{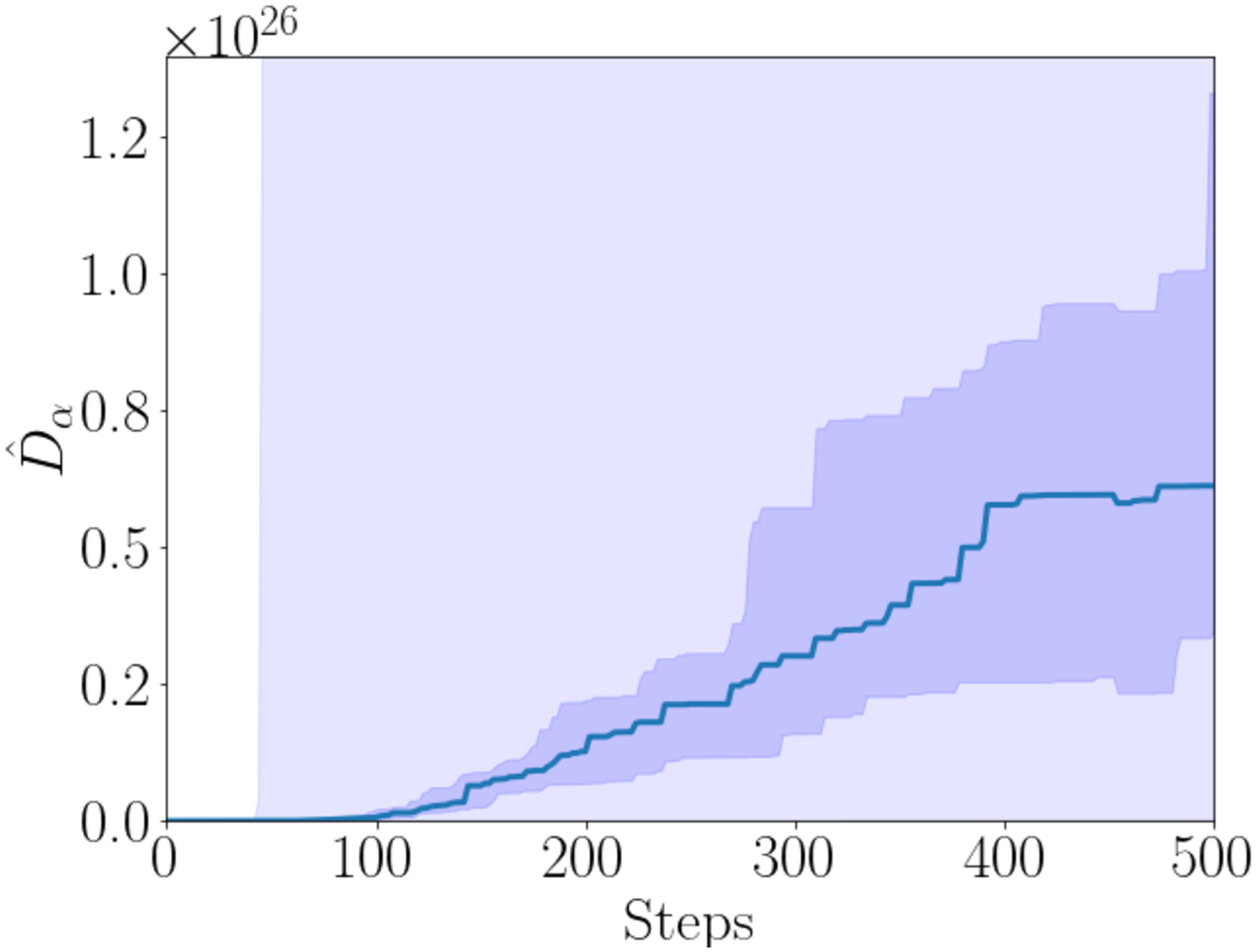}
            \subcaption{$\alpha=-1$}
        \end{minipage}&
        \begin{minipage}[t]{0.55\hsize}
            \includegraphics[keepaspectratio, scale=0.29]{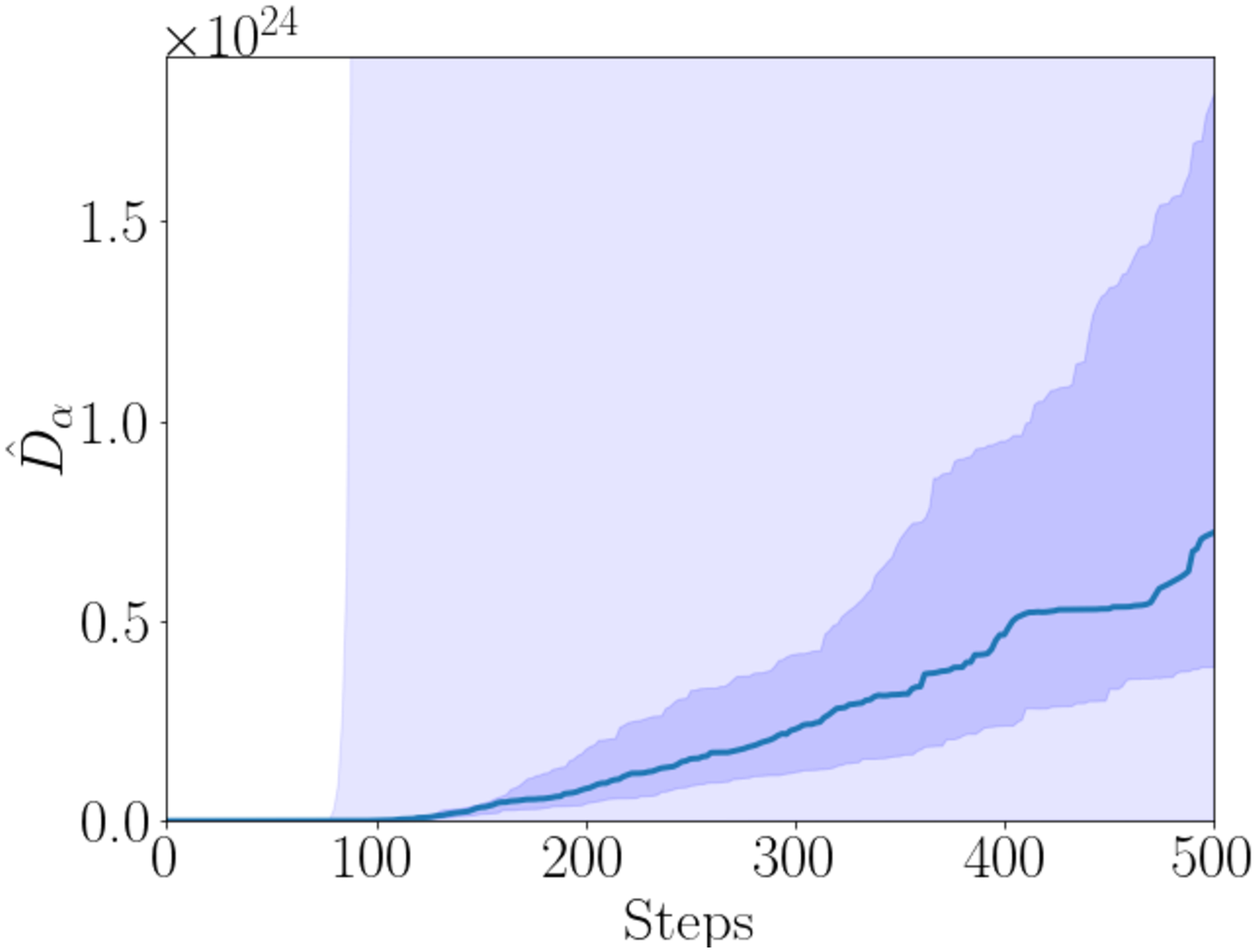}
            \subcaption{$\alpha=2.0$}
        \end{minipage} \\
        \begin{minipage}[t]{0.55\hsize}
            \includegraphics[keepaspectratio, scale=0.29]{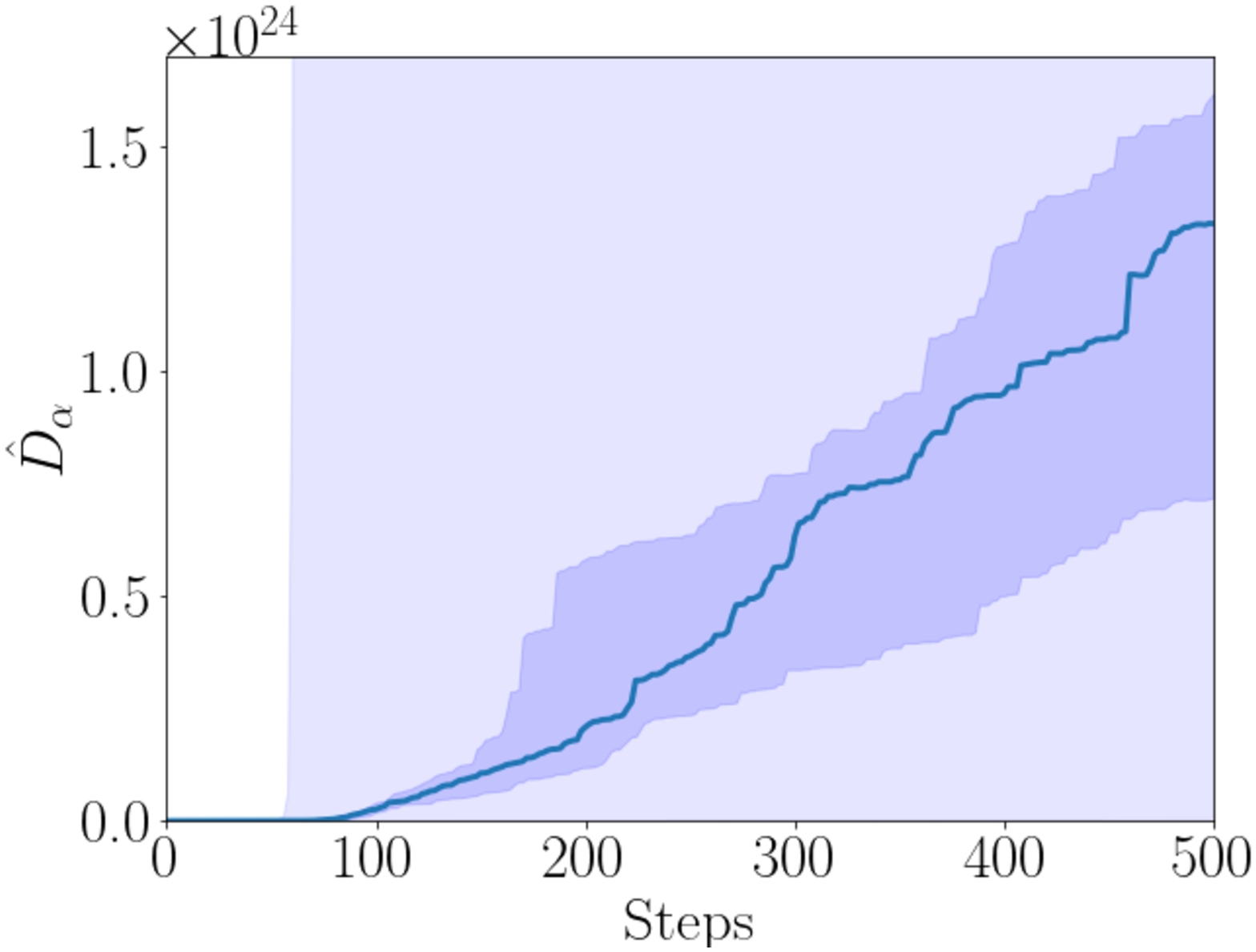}
            \subcaption{$\alpha=3.0$}
        \end{minipage} &
        \begin{minipage}[t]{0.55\hsize}
            \includegraphics[keepaspectratio, scale=0.29]{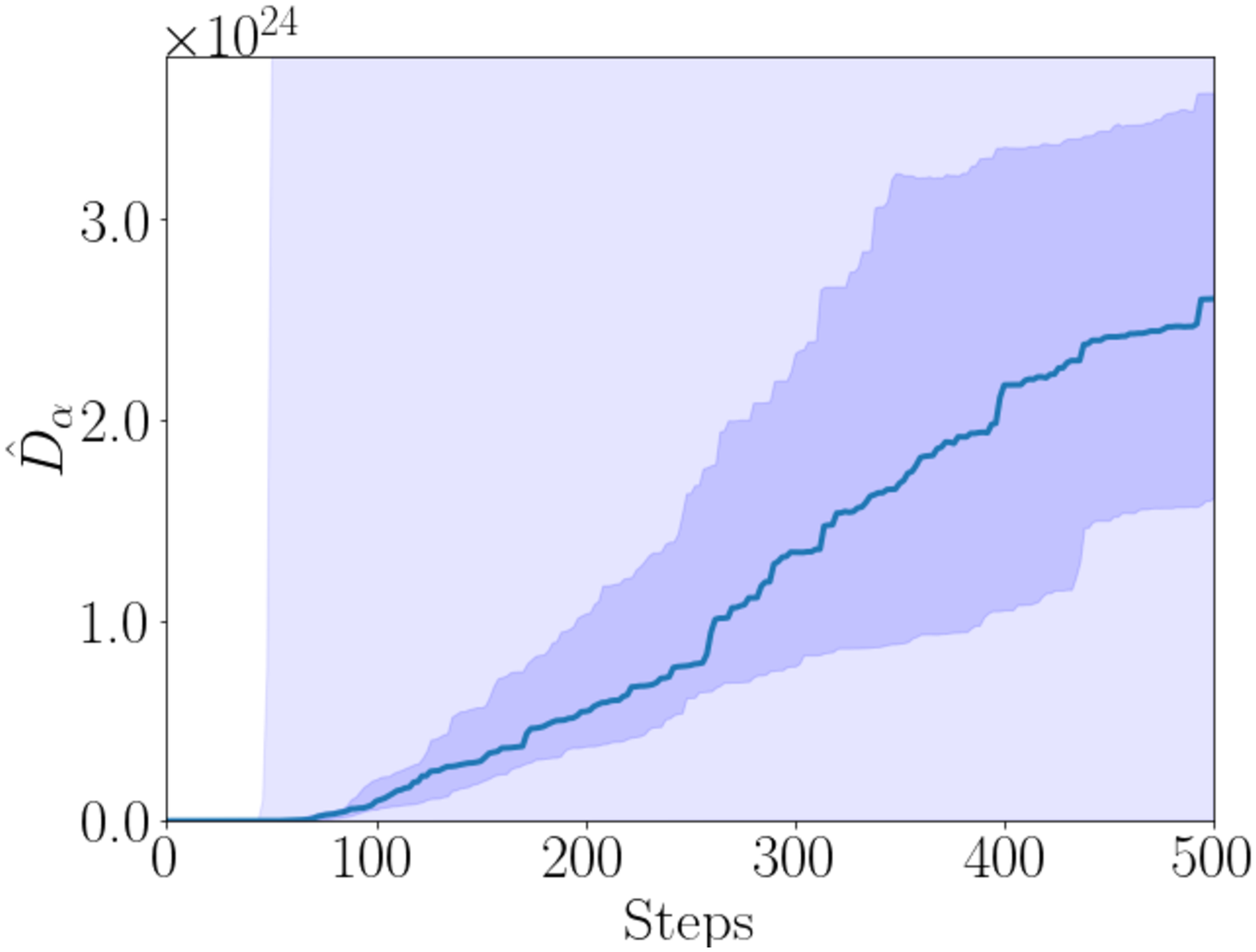}
            \subcaption{$\alpha=4.0$}
        \end{minipage}
    \end{tabular}
    \caption{Results of estimating the $\alpha$ divergence for $\alpha=-3, -2, -1, 2, 3$ and $4$, over the number of learning steps during the optimization.
        The $y$-axis of each graph represents the estimated value of the $\alpha$ divergence, and the $x$-axis of each graph represents the learning step.
        The solid blue line shows the median of the estimates of the $\alpha$ divergence.
        The dark blue area shows the ranges of the estimates between the 45th and 55th percentiles, and the light blue area shows the range of the estimates between the 5th and 95th percentile quartiles.}
    \label{Fig_convergence_fof_alpha_from_m3_to_p3}
\end{figure}

\begin{figure}
    \subcaptionbox*{}{}
    \setcounter{subfigure}{6}
    \centering
    \begin{tabular}{cc}
        \begin{minipage}[t]{0.55\hsize}
            \includegraphics[keepaspectratio, scale=0.29]{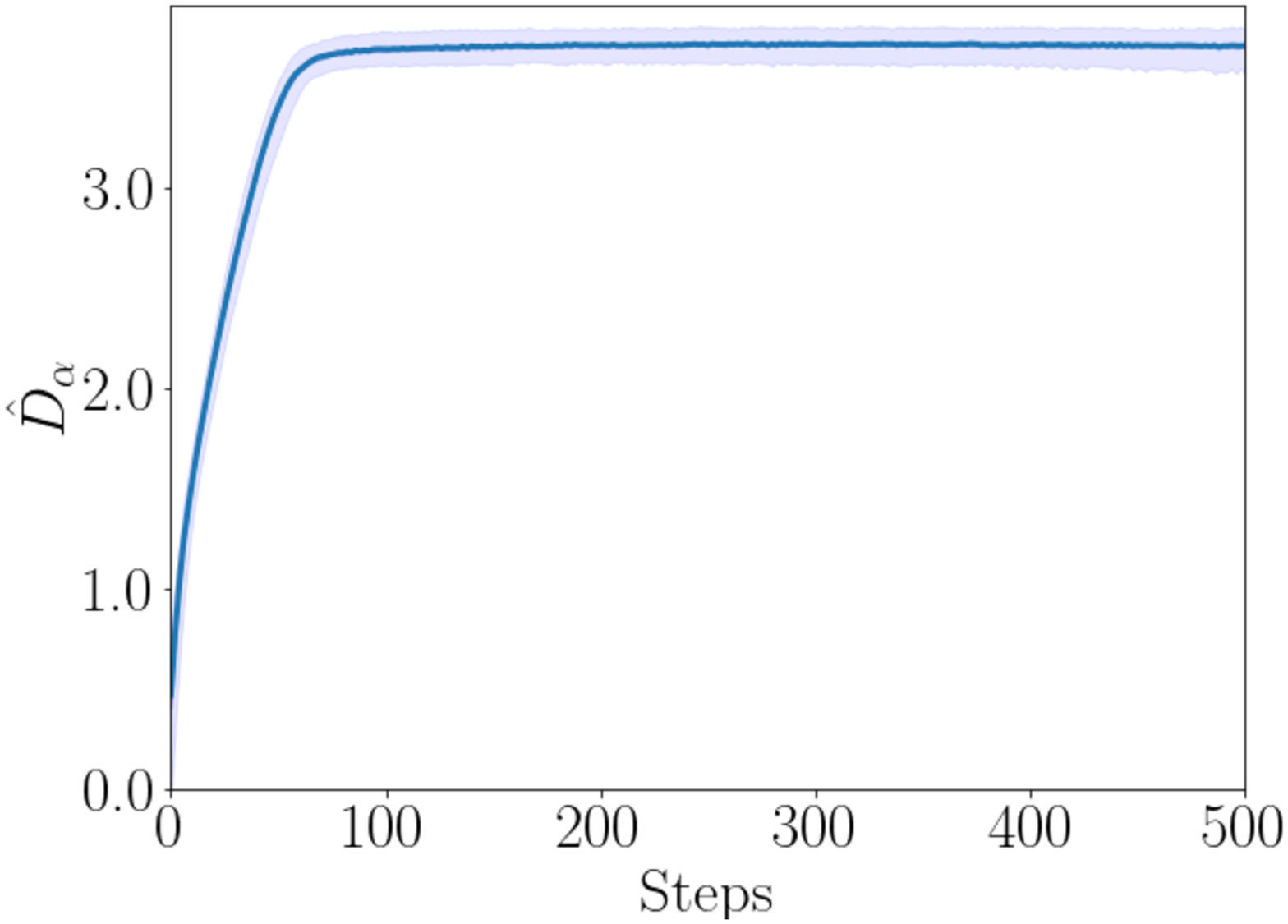}
            \subcaption{$\alpha=0.2$}
        \end{minipage} &
        \begin{minipage}[t]{0.55\hsize}
            \includegraphics[keepaspectratio, scale=0.29]{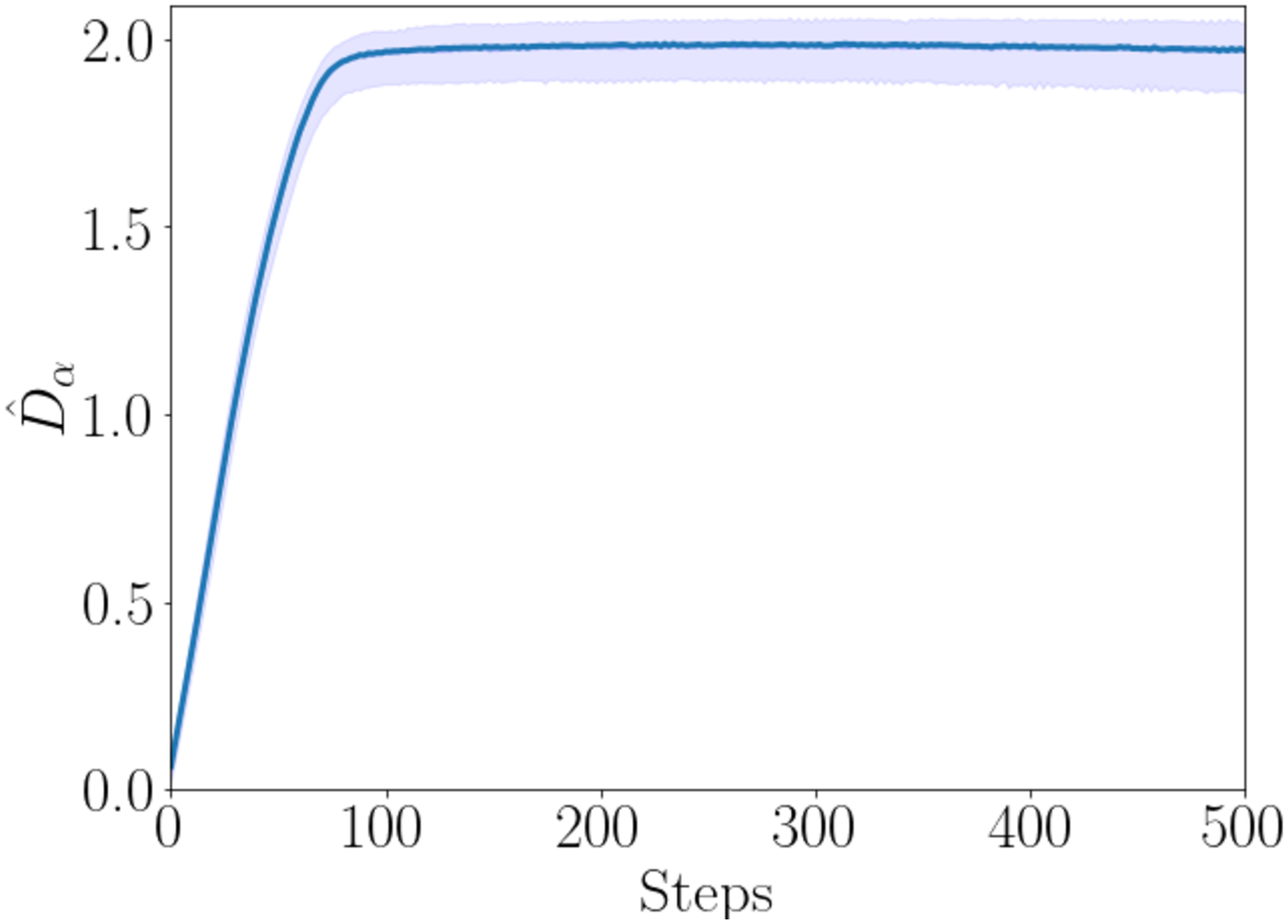}
            \subcaption{$\alpha=0.5$}
        \end{minipage} \\
        \begin{minipage}[t]{0.55\hsize}
            \includegraphics[keepaspectratio, scale=0.29]{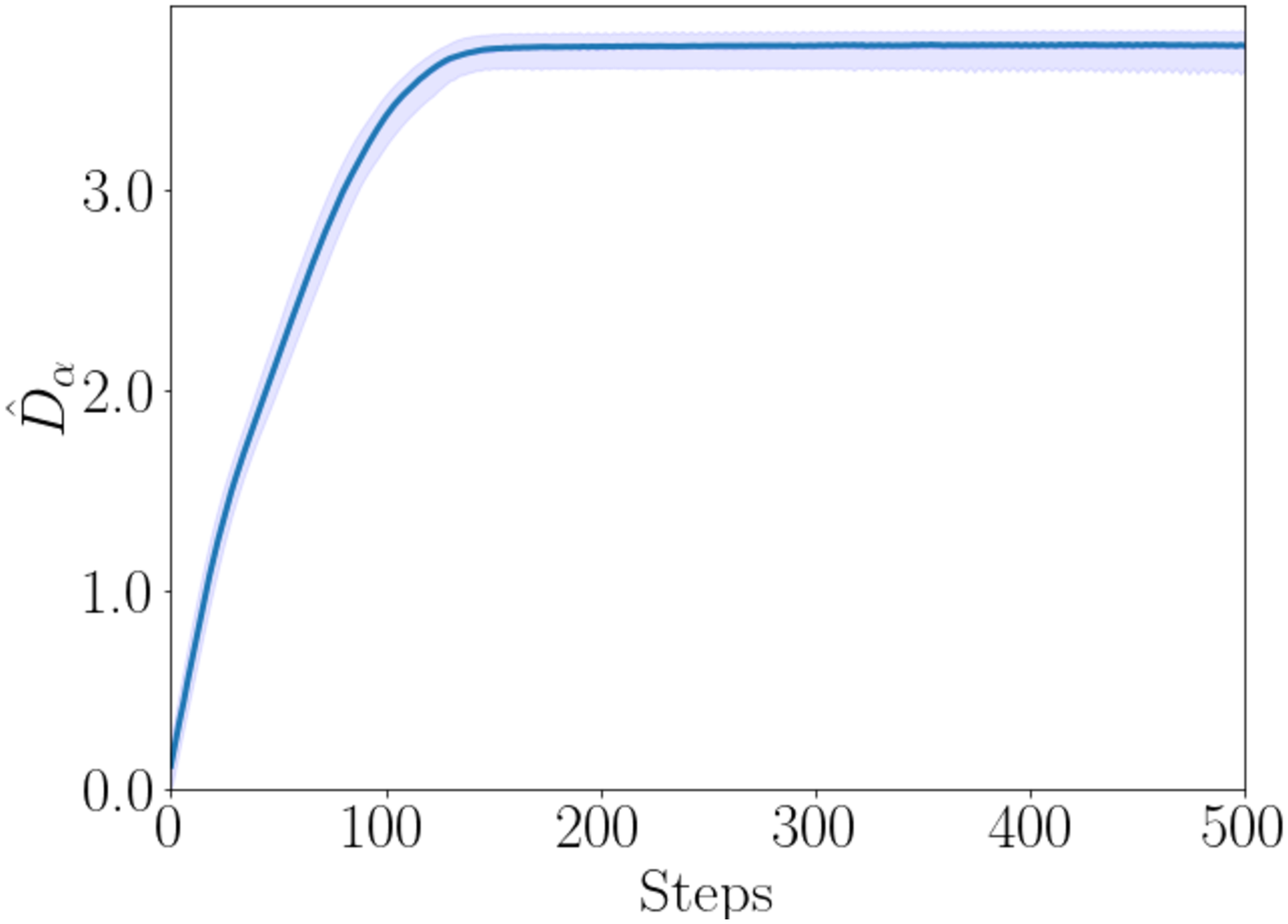}
            \subcaption{$\alpha=0.8$}
        \end{minipage}
    \end{tabular}
    \caption{Results of estimating the $\alpha$ divergence for $\alpha=0.2, 0.5$ and $0.8$, over the number of learning steps during the optimization.
        The $y$-axis of each graph represents the estimated value of the $\alpha$ divergence, and the $x$-axis of each graph represents the learning step.
        The solid blue line shows the median of the estimates of the $\alpha$ divergence.
        The dark blue area shows the ranges of the estimates between the 45th and 55th percentiles, and the light blue area shows the range of the estimates between the 5th and 95th percentile quartiles.}
    \label{Fig_convergence_fof_alpha_from_02_to_08}
\end{figure}

\paragraph{Results.}
Figure \ref{Fig_convergence_fof_alpha_from_m3_to_p3} and \ref{Fig_convergence_fof_alpha_from_02_to_08} show the results of estimating the $\alpha$ divergence over the number of learning steps during the optimization. Figure \ref{Fig_convergence_fof_alpha_from_m3_to_p3} is for $\alpha=-3, -2, -1, 2, 3$ and $4$, and Figure \ref{Fig_convergence_fof_alpha_from_02_to_08} is for $\alpha=0.2, 0.5$, and $0.8$.
The $y$-axis of each graph represents the estimated value of the $\alpha$ divergence, and the $x$-axis of each graph represents the learning step.
The solid blue line shows the median of the estimates of the $\alpha$ divergence.
The dark blue area shows the ranges of the estimates between the 45th and 55th percentiles, and the light blue area shows the range of the estimates between the 5th and 95th percentile quartiles.

As shown in Figure \ref{Fig_convergence_fof_alpha_from_m3_to_p3}, the estimates of the $\alpha$ divergence diverged.
This corresponds to a negative divergence of the loss function $\mathcal{L}_{\alpha}(\theta_t)$ in (\ref{Eq_subsection_neumerical_experiments_convegence_various_alphas_loss_1}), and then implies that $E_{Q}[e^{T_{\theta_t}}] \rightarrow 0$ for $\alpha > 1$, and $E_{Q}[e^{T_{\theta_t}}] \rightarrow \infty$ for $\alpha < 0$ in (\ref{Eq_subsection_neumerical_experiments_convegence_various_alphas_loss_0}).
The discussion in Section  \ref{Section_Estimation_of_Balancing_Weights} suggests that $E[\nabla_{\theta} \mathcal{L}_{\alpha}(\theta)] \rightarrow  \vec{0}$.
That is, the gradients of the neural networks in this case vanished for $\alpha=-3, -2, -1, 2, 3$ and $4$.
However, as shown in Figure \ref{Fig_convergence_fof_alpha_from_02_to_08}, the estimates of the $\alpha$ divergence converge stably for $\alpha=0.2, 0.5$, and $0.8$.

\newpage
\subsection{Experiments to confirm the relationship between dimensions of dataset and steps in training}\label{subsection_neumerical_experiments_convegence_dims_data}
In this section, we report the results of numerical experiments related to the discussion in Section \ref{Section_techniques_estimating_NGD_balancing_weights}:
the results of numerical experiments to confirm the relationship between dimensions of dataset and steps in training are presented.

\paragraph{Experimental Setup.}
We generated training and test datasets of dimensions $d=2$, $3$, $4$, $5$, $6$, and $7$, and then trained an NGB model with the training dataset while estimating the $\alpha$ divergence at each learning step with the test dataset.
One hundred numerical simulations were performed for each dimension $d$.
As a result of the experiment, the median of the estimated value and ranges between the 5th and 95th percentile quartiles at each learning step are reported.

\paragraph{Synthetic Data.}
For each $d=2$, $3$, $4$, $5$, $6$, and $7$,
we generated the training and test datasets of size 5000 from $d$-dimantional normal distribution $\{X_1, X_2, \ldots, X_d\}$, such that $E[X_i]=0$, $\mathrm{Var}[X_i] = 1$ and $E[X_i\cdot X_j]=0.8$ ($i \neq j$).

\paragraph{Estimating the $\alpha$ divergence.}
The $\alpha$ divergence was estimated in the following way:
\begin{eqnarray}
    \hat{D}_{\alpha}(Q||P)(t) &=& \frac{1}{\alpha \cdot(1-\alpha)} - \frac{1}{\alpha} \hat{E}_Q\left[e^{\alpha \cdot T_{\theta_t}(\mathbf{x}^{te})}\right] -
    \frac{1}{1- \alpha}  \hat{E}_P\left[e^{(\alpha - 1) \cdot T_{\theta_t}(\mathbf{x}^{te})}\right] \label{Eq_subsection_neumerical_experiments_convegence_dims_data_loss_0} \\
    &=& \frac{1}{\alpha \cdot(1-\alpha)} - \mathcal{L}_{\alpha}(\theta_t), \label{Eq_subsection_neumerical_experiments_convegence_dims_data_loss_1}
\end{eqnarray}
where $T_{\theta_t}$ is a model at learning step $t$ in Algorithm \ref{algo_train} and $\mathbf{x}^{te}$ denotes the test dataset.
Note that, decreasing of the estimated divergence $\hat{D}_{\alpha}(Q||P)(t)$ in (\ref{Eq_subsection_neumerical_experiments_convegence_dims_data_loss_0}) implies increasing of the loss $\mathcal{L}_{\alpha}(\theta_t)$ in (\ref{Eq_subsection_neumerical_experiments_convegence_dims_data_loss_1}).

\paragraph{Implementation and Training Details.}
We used a neural network which has 3 hidden layers of 100 units in each layer.
The Adam algorithm in PyTorch was used.
For the hyperparameters in the training, the learning rate was 0.001, BathSize was 2500, and the number of epochs was 500.
A NVDIA Tesla K80 GPU was used. It took approximately four hours to conduct all simulations for each $d$.

\begin{table*}[t]
    \caption{The early stop step (“$N^{2/d}$”) and the median of the steps at which the estimated divergence reaches its maximum (“$\operatorname{median}(K_{\operatorname{max}})$”), for each dimension $d=2$, $3$, $4$, $5$, $6$, and $7$.}
    \label{Table_for_arg_max_steps_Section_7}
    \centering
    \begin{tabular}{lccccccc}
        \toprule            
        & $d=2$ & $d=3$ & $d=4$ & $d=5$ & $d=6$ & $d=7$ &\\
        \midrule              
        $N^{2/d}$             & 5000 & 292 & 71 & 30 & 17 & 11  \\
        $\operatorname{median}(K_{\operatorname{max}})$     & 130 & 112 & 130 & 136 & 50 & 50  \\
        \bottomrule          
    \end{tabular}
\end{table*}

\begin{figure}
    \centering
    \begin{tabular}{cc}
        \begin{minipage}[t]{0.55\hsize}
            \includegraphics[keepaspectratio, scale=0.29]{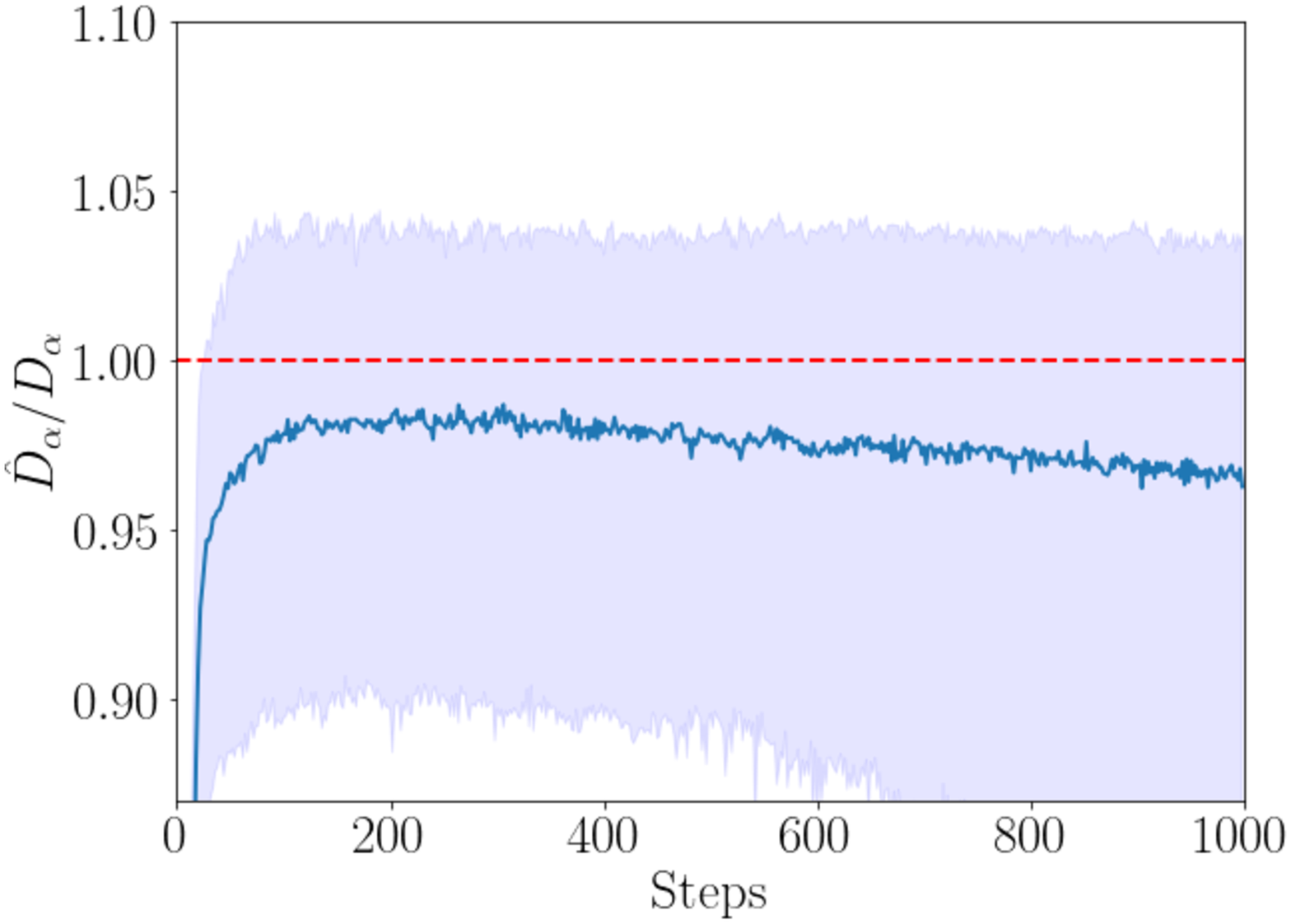}
            \subcaption{$d = 2$}
        \end{minipage} &
        \begin{minipage}[t]{0.55\hsize}
            \includegraphics[keepaspectratio, scale=0.29]{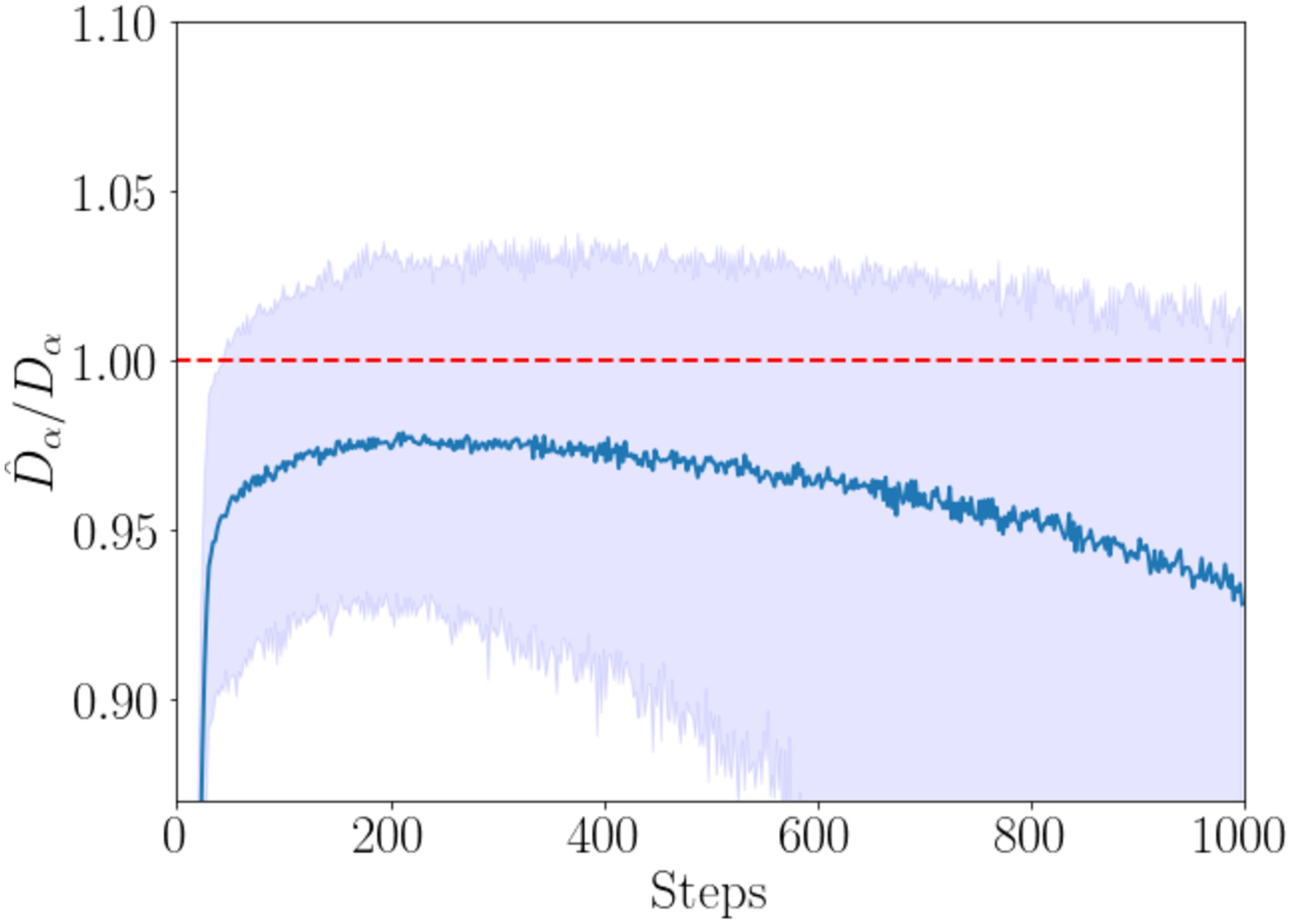}
            \subcaption{$d = 3$}
        \end{minipage} \\
        \begin{minipage}[t]{0.55\hsize}
            \includegraphics[keepaspectratio, scale=0.29]{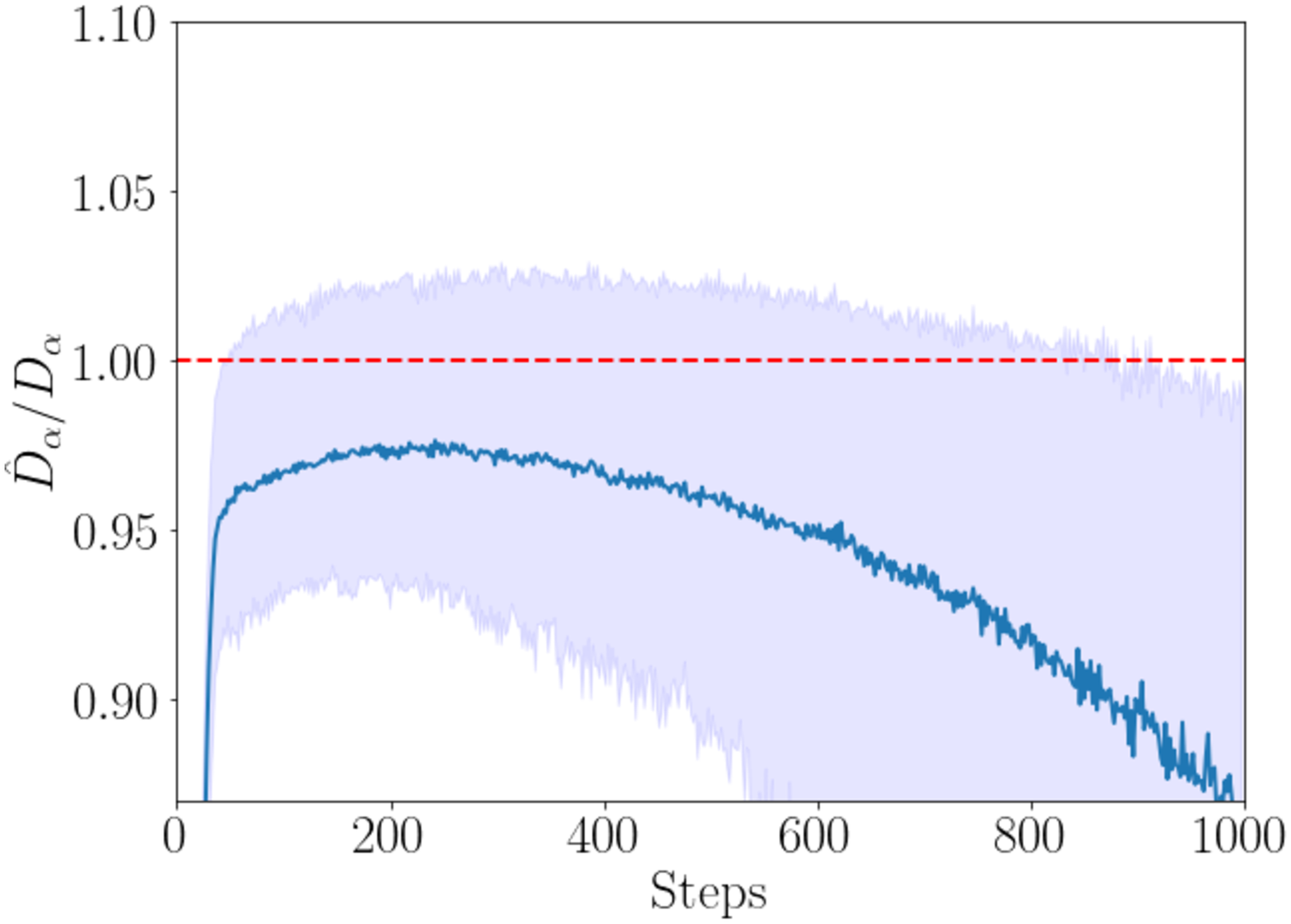}
            \subcaption{$d = 4$}
        \end{minipage} &
        \begin{minipage}[t]{0.55\hsize}
            \includegraphics[keepaspectratio, scale=0.29]{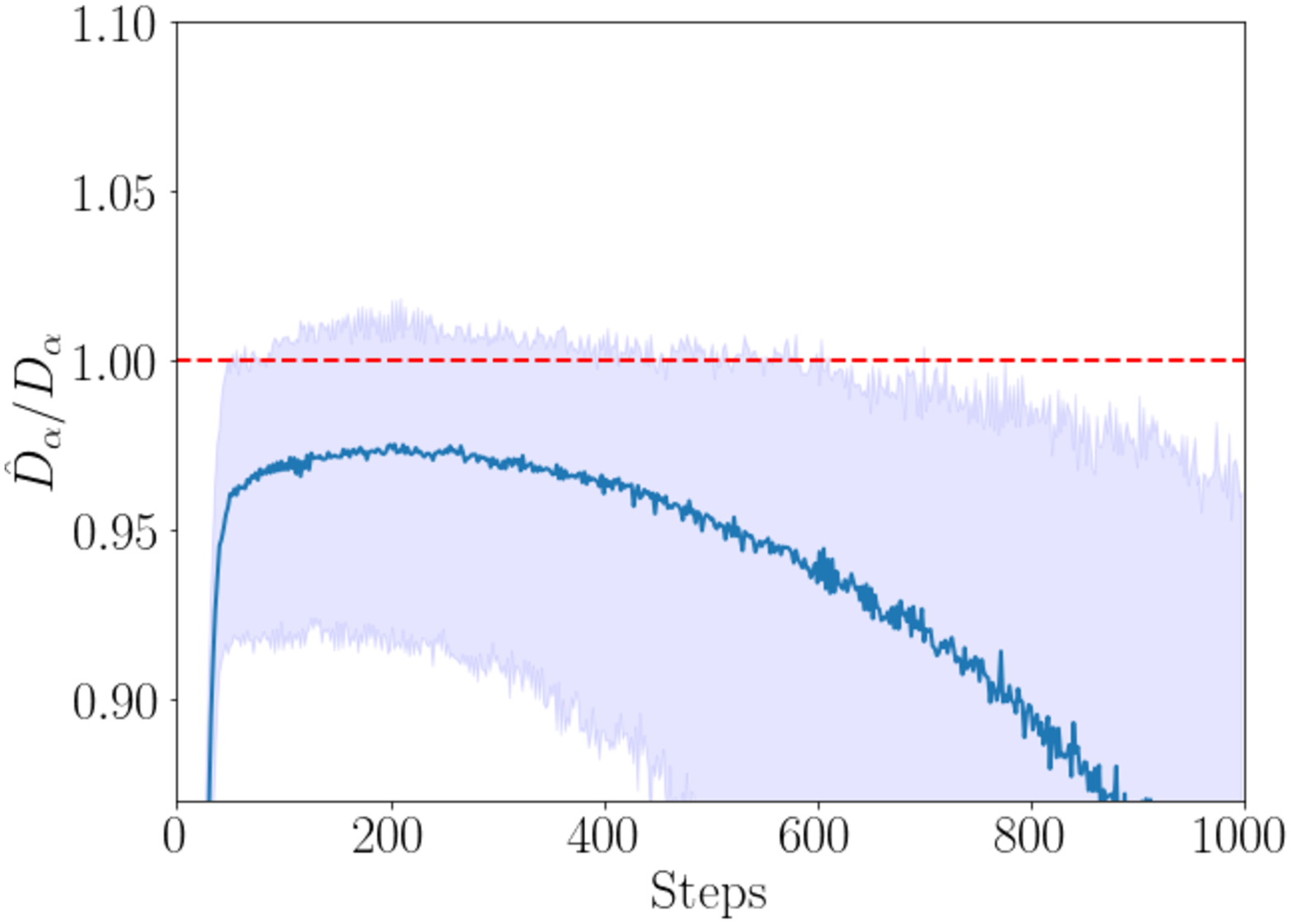}
            \subcaption{$d = 5$}
        \end{minipage} \\
        \begin{minipage}[t]{0.55\hsize}
            \includegraphics[keepaspectratio, scale=0.29]{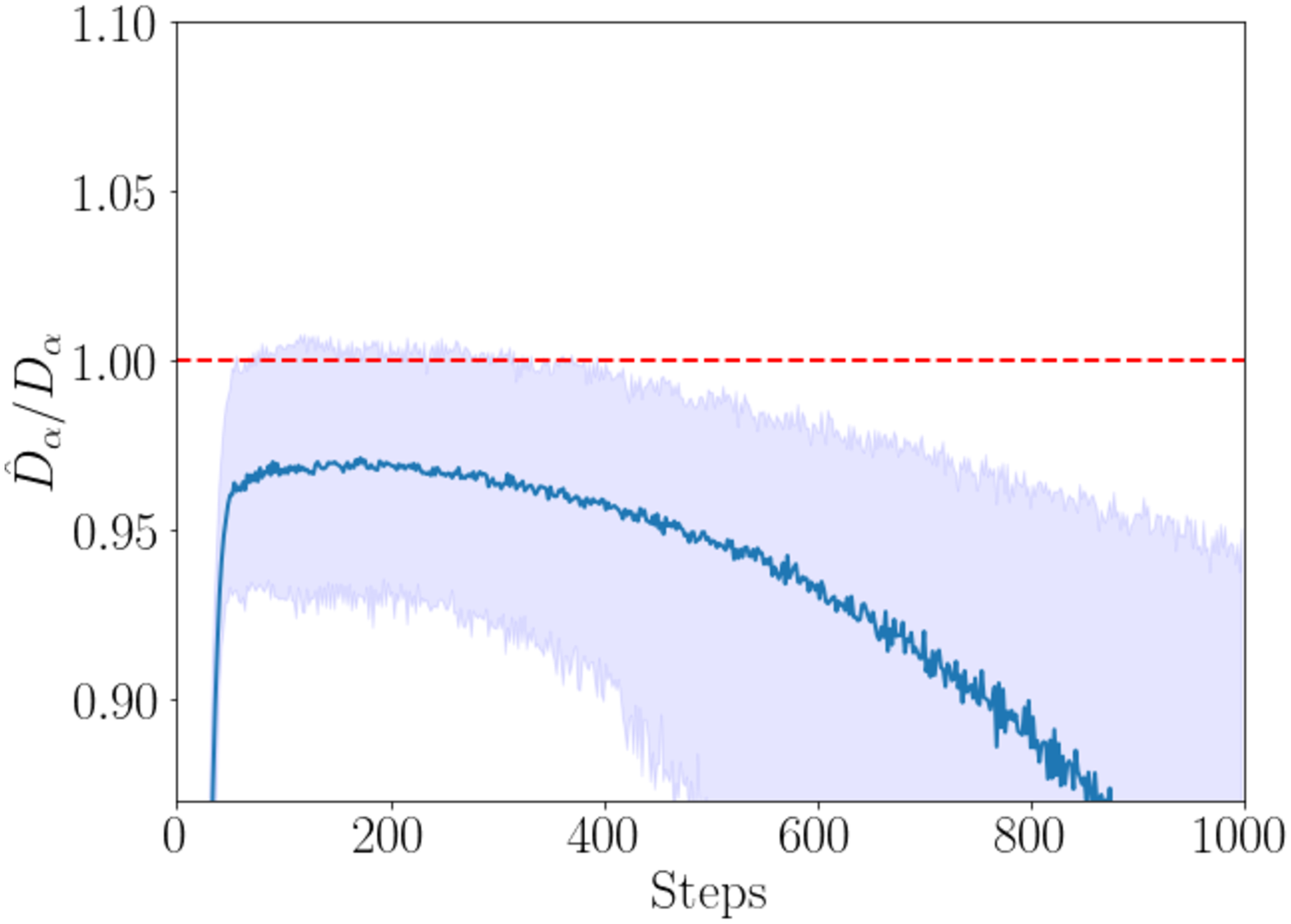}
            \subcaption{$d = 6$}
        \end{minipage} &
        \begin{minipage}[t]{0.55\hsize}
            \includegraphics[keepaspectratio, scale=0.29]{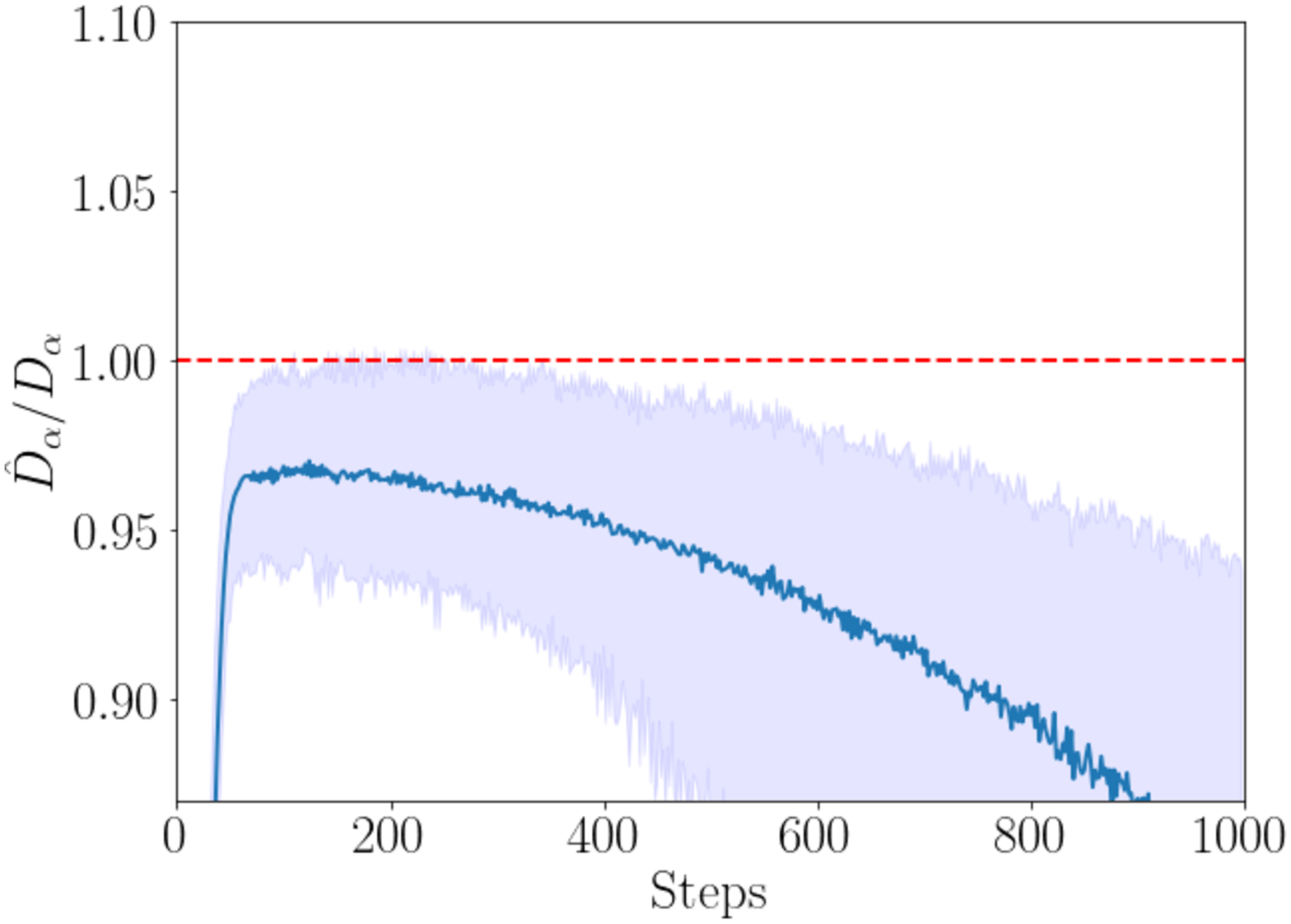}
            \subcaption{$d = 7$}
        \end{minipage}
    \end{tabular}
    \caption{Results of estimating the $\alpha$ divergence for $\alpha=-3, -2, -1, 2, 3$ and $4$, over the number of learning steps during the optimization.
        The $y$-axis of each graph represents the estimated value of the $\alpha$ divergence divided by the true value of the divergence, and the $x$-axis of each graph represents the learning step.
        The dashed red line indicates Y=1, which corresponds to the theoretical value of the estimate for each $d$.
        The solid blue line shows the median of the estimates of the $\alpha$ divergence.
        The dark blue area shows the ranges of the estimates between the 45th and 55th percentiles, and the light blue area shows the range of the estimates between the 5th and 95th percentile quartiles.}
    \label{Fig_memorization}
\end{figure}

\paragraph{Results.}
Let $K_{\operatorname{max}}$ denote the step at which the estimated divergence reaches its maximum:
\begin{equation}
    K_{\operatorname{max}} =  \underset{t} {\operatorname{argmax}} \hat{D}_{\alpha}(Q||P)(t).
\end{equation}
Table \ref{Table_for_Experiment_Section_8} lists $N^{2/d}$, the early stop step obtained from (\ref{Eq_the_early_stop_def}), and the median of $K_{\operatorname{max}}$, for each dimension $d=2$, $3$, $4$, $5$, $6$, and $7$.
In Figure \ref{Fig_memorization}, we show the results of estimating the $\alpha$ divergence over the number of learning steps during the optimization.
Since the value of the $\alpha$ divergence changes as the dimension of the dataset changes, we divided by the the estimated value of the divergence by the true value of the divergence to normalize the results of each dimension.
The $y$-axis of each graph represents the estimated value of the $\alpha$ divergence divided by the true value of the divergence, and the $x$-axis of each graph represents the learning step.
The solid blue line shows the median of the estimates of the $\alpha$ divergence.
The light blue area shows the range of the estimates between the 5th and 95th percentile quartiles.
The dashed red line indicates Y=1, which corresponds to the theoretical value of the estimate for each $d$.

As shown in Table \ref{Table_for_arg_max_steps_Section_7}, the steps from the early stop method and those at which the estimates decreased were approximately consistent, except in the case of $d=2$.
However, the estimates of the data of the low dimensions, particularly $d=2$, decreased earlier than the early stop method suggests.
This may be because $C$ in (\ref{Eq_the_early_stop_def}) for the data of low dimensions can be small because the neural network learns quickly when the dimensions of the data are low.
However, Figure \ref{Fig_memorization} shows that the estimates of the divergence decreased slowly when the dimensions of the data are low, and they decreased more quickly when the dimensions of the data were higher.
These results suggest that the curse of dimensionality of balancing is easier to observe when dimensions of data are higher.

\newpage
\subsection{Experiments for esitimating causal effects of joint and multidimensional interventions with different sample sizes}\label{subsection_neumerical_experiments_causal_effect}
In this section, we report the results of numerical experiments related to the discussion in Section \ref{Section_Sample_Size_Requirements}:
the results of numerical experiments for esitimating causal effects of joint and multidimensional interventions with different sample sizes are presented.

\paragraph{Experimental Setup.}
The following two experiments were conducted, in which synthetic data of size $N=\allowbreak 1000,\allowbreak 10000$, and $100000$ were generated using the method developed by \citeauthor{vegetabile2021nonparametric}(\citeyear{vegetabile2021nonparametric}).
\begin{itemize}
    \item Experiment 1. An experiment on estimating the causal effect of a single intervention, especially for continuous intervention, $E_{\widetilde{P}_1}[Y|A, \mathbf{X}]$ where
    $\widetilde{P}_1=P(Y, \overline{do}(A), \mathbf{X})$.
    \item Experiment 2. An experiment to estimate the causal effect of a mixture of both arbitrary discrete and continuous interventions,
        $E_{\widetilde{P}_2}[Y|A, \mathbf{X}]$ where
        $\widetilde{P}_2=P(Y, \overline{do}(A), \overline{do}(\mathbf{X}_1),  \overline{do}(X_2), \overline{do}(\mathbf{X}_3)]$.
\end{itemize}

\paragraph{Experimental Details.}
Experiments 1 and 2 were conducted using the following steps.

Step 1: We created training dataset of size $N=1000,\allowbreak 10000$, and $100000$, and test dataset with size $N=1000$.
The training dataset were generated using the method developed by \citeauthor{vegetabile2021nonparametric}(\citeyear{vegetabile2021nonparametric}).
The test dataset were generated from the following distribution:
\begin{itemize}
\item Experiment 1. $\widetilde{P}_1 = P(Y, \overline{do}(A), \mathbf{X})$,
\item Experiment 2. $\widetilde{P}_2 = P(Y, \overline{do}(A), \overline{do}(\mathbf{X}_1), \overline{do}(X_2), \overline{do}(\mathbf{X}_3))$,
\end{itemize}
where $P$ denotes the distribution of the training dataset.
To create the test dataset, we shuffled the dataset generated from the same distribution as the training dataset.

Step 2: The balancing weights were estimated for each experiment.
We estimated $\mathit{BW}(A,\mathbf{X}:T_{\theta_{*}})$ for Experiment 1,
and $\mathit{BW}(A, \mathbf{X}_1, X_2, \mathbf{X}_3:T_{\theta_{*}})$ for Experiment 2.

Step 3: We created models for each experiment using the linear regression (LR) or the gradient boosting tree (GBT) algorithm with our weights from the previous step.
The hyperparameters were tuned to create models of GBT.

Step 4: We estimate the average causal effects $E_{\widetilde{P}_1}[\allowbreak Y \allowbreak|A,\allowbreak \mathbf{X}\allowbreak]$ and $E_{\widetilde{P}_2}[\allowbreak Y \allowbreak|A,\allowbreak \mathbf{X}\allowbreak]$
using the predictions of the models from Step 2 with the test dataset.
Finally, we report the mean squared error (RMSE) between the true and estimated values.

\paragraph{Baseline Method.}
The main baseline method used in our experiments is entropy balancing \cite{tubbicke2022entropy}.
We compared our method with the method for balancing $\mathbf{X}$ with $A$
for each of the moments from 1 to 4.
For Experiment 1, both our method and the baseline method estimated the same target:
$E_{\widetilde{P}_1}[Y|A, \mathbf{X}]$ where $\widetilde{P}_1(Y, A, \mathbf{X}) = P(Y, \overline{do}(A), \mathbf{X})$.
However, no existing method can fully deal with the target of Experiment 2:
$E_{\widetilde{P}_2}[Y|A, \mathbf{X}]$ where $\widetilde{P}_2=P(Y, \overline{do}(A), \overline{do}(\mathbf{X}_1),  \overline{do}(X_2), \overline{do}(\mathbf{X}_3)]$.
Therefore, the same entropy balancing as in Experiment 1 was used in Experiment 2.
This may be an unfair comparison to the baseline method.
In addition, we included a “naïve” estimation, using algorithms with no sample weights, as a baseline.
For the calculation of entropy balancing weights, WeightIt library in R was used. \footnote{https://cran.r-project.org/web/packages/WeightIt/index.html}.

\paragraph{Training Data Set.}
Specifically, we used the following steps to generate the dataset.
First, $\mathbf{W} = (W_1, W_2, W_3, W_4, W_5)$ were generated independently,
such that
$W_1 \sim \mathcal{N}(-0.5, 1)$,
$W_2 \sim \mathcal{N}(1, 1)$,
$W_3 \sim \mathcal{N}(0, 1)$,
$W_4 \sim \mathcal{N}(1, 1)$,
and $\mathcal{X}_{W_5} =\{0, 1, 2\}$ with
$P(W_5 = 0) = 0.70$ and $P(W_5 = 1) = P(W_5 = 2) = 0.15$.
Second, $A$ and $Y$ were generated as follows:
\begin{eqnarray}
    A &\sim& \mathscr{X}^2(df=3,\, \mu_{A}(W_1, W_2, W_4, W_5)), \nonumber \\
    Y &=& \frac{1}{50}\left[ \left( - 0.15 A^2 + A(W_1^2+ W_2^2) -15 \right) \right. \nonumber \\
    && \quad +\left((W_1 + 3)^2 + 2(W_2-25)^2 + W_3  \right)  \nonumber \\
    && \quad \left. - C + \varepsilon \right],  \label{Eq_Y}
\end{eqnarray}
where $\mu_{A}(W_1, W_2, W_4, W_5) = 5|W_1|+6|W_2|+|W_4| + a$, and
$a = 0$ if $W_5=0$, and $a = 1$ if $W_5=1$, and $a = 5$ if $W_5=2$,
and $C = E[(W_1+3)^2]+ 2E[(W_2 - 25)^2] + E[ W_3]$, and $\varepsilon \sim \mathcal{N}(0, 1)$.
Here, $\mathscr{X}^2(df=n, \mu)$ is the noncentral $\chi^2$ distribution with n degrees of freedom
and a noncentral parameter $\mu$.
Finally, we create new variables $\mathbf{X} = (\mathbf{X}_1, X_2, \mathbf{X}_3)$,
as observed values of $\mathbf{W}$ using the following transformation:
\begin{eqnarray}
    \mathbf{X}_1 &=& (X_{(1,1)}, X_{(1,2)}, X_{(1,3)}),   \label{trans_dataX1} \\
    &\text{where}&  X_{(1,1)} = \exp\left( W_1/2 \right), \nonumber \\
    && X_{(1,2)}= W_2/(1 + \exp(W_1))+ 10,    \nonumber \\
    && X_{(1,3)}= W_1 W_3/25 + 0.6, \nonumber \\
    X_2  &=& (W_4 - 1)^2, \label{trans_dataX2} \\
    \mathbf{X_3} &=& \begin{cases}
        (1, 0)  \quad \text{if  $W_5=0$},  \\
        (0, 1)  \quad \text{if  $W_5=1$},  \\
        (0, 0)  \quad \text{if  $W_5=2$}. \label{trans_dataX3}
    \end{cases}
\end{eqnarray}

\paragraph{Test Data Set.}
We first generated dataset from the same distribution as the training dataset.
Second, the dataset were shuffled by the index, with the following divided parts treated as a single piece of data:
for Experiment 1, $A$ and $\mathbf{X}$ were shuffled by the index, and for Experiment 2,
each of $A$, $\mathbf{X}_1$, $X_2$ and  $\mathbf{X}_3$ were shuffled by the index.
Third, using the inverse transformation of Eq. (\ref{trans_dataX1})-(\ref{trans_dataX3}),
we calculated $(W_1, W_2, W_3, W_4, W_5)$ from $\mathbf{X_1}=(X_{(1,1)}, X_{(1,2)}, X_{(1,3)})$,
$X_2$, and $\mathbf{X_3}$ of the shuffled dataset:
\begin{eqnarray}
    W_1 &=& 2 \log X_{(1,1)}, \quad  W_2 = X_{(1,2)} \cdot (1 + X_{(1,1)}^2),\nonumber \\
    W_3 &=& \frac{25(X_{(1,3)} - 0.6)}{2\log X_{(1,1)}}, \quad W_4 = \sqrt{X_2} + 1, \nonumber \\
    W_5 &=& \begin{cases}
        0 \quad \text{if  $\mathbf{X_3}=(1, 0)$},  \\
        1 \quad \text{if  $\mathbf{X_3}=(0, 1)$},  \\
        2 \quad \text{if  $\mathbf{X_3}=(0, 0)$}.
    \end{cases} \nonumber
\end{eqnarray}
Finally, the true values of $Y$ for causal effects
were calculated using the terms in  Eq. (\ref{Eq_Y}) without the term $\varepsilon$.

\paragraph{Implementation and Training Details.}
\underline{$N=1000$:}
For experiments with the dataset of size $N=1000$, we used a neural network which has 10 hidden layers of 100 units in each layer.
$\alpha=0.5$ was used to estimate the divergence.
The Adam algorithm in PyTorch was used.
For the hyperparameters in the training, the learning rate was 0.0001, BathSize was 1000, and the number of epochs was 70.
A NVDIA Tesla K80 GPU was used. It took approximately 40 min to conduct all the simulations for each experiment.

\underline{$N=10000$:}
For experiments with the dataset of size $N=10000$,
We used a neural network which has 10 hidden layers of 100 units in each layer.
$\alpha=0.5$ was used to estimate the divergence.
The Adam algorithm in PyTorch was used.
For the hyperparameters in the training, the learning rate was 0.0001, BathSize was 2500, and the number of epochs was 200.
A NVDIA Tesla K80 GPU was used. It took approximately 7 h to conduct all the simulations for each experiment.

\underline{$N=100000$:}
For experiments with the dataset of size $N=100000$,
We used a neural network which has 10 hidden layers of 100 units in each layer.
$\alpha=0.5$ was used to estimate the divergence.
The Adam algorithm in PyTorch was used.
For the hyperparameters in the training, the learning rate was 0.0001, BathSize was 2500, and the number of epochs was 200.
A NVDIA Tesla K80 GPU was used. It took approximately 78 h to conduct all the simulations for each experiment.

\begin{table*}[t]
    \caption{Average RMSE for estimation in Experiments 1 and 2 for dataset of size $N=1000$, $10000$, and $100000$.
        For entropy balancing, the number to the right side of the method name,
        “($\mathit{m}$),” denotes the number of moments that are balanced.
        The results from 100 simulations are in the form of “mean (std. err.)” .}
    \label{Table_for_Experiment_Section_8}
    \centering
    \subcaption{$N=1000$}
    \label{Table_for_Experiment_Section_8_for_N_1000}
    \begin{tabular}{lcccc}
        \toprule            
        & \multicolumn{2}{c}{Experiment 1} & \multicolumn{2}{c}{Experiment 2} \\
        \textbf{Method}       & \textbf{LR} & \textbf{GBT} 	& \textbf{LR} & \textbf{GBT} \\
        \midrule             
        Unweighted            & 1.347(0.039) & 0.739(0.066)     & 1.347(0.033)  & 0.741(0.068) \\
        Entropy Balancing(1)  & 1.303(0.056) & 0.724(0.058)     & 1.303(0.052) & 0.726(0.060) \\
        Entropy Balancing(2)  & 1.206(0.029) & \textbf{0.693(0.056)}     & 1.206(0.026) & \textbf{0.698(0.055)} \\
        Entropy Balancing(3)  & \textbf{1.201(0.026)} & 0.690(0.054)     & \textbf{1.201(0.024)} & \textbf{0.698(0.061)} \\
        Entropy Balancing(4)  & 1.203(0.027) & 0.699(0.057)     & 1.203(0.025) & 0.699(0.061) \\
        NBW                   & 1.347(0.039) & 0.745(0.065)     & 1.347(0.034) & 0.738(0.063) \\
        \bottomrule          
    \end{tabular}
    \centering
    \subcaption{$N=10000$}
    \label{Table_for_Experiment_Section_8_for_N_10000}
    \begin{tabular}{lcccc}
        \toprule            
        & \multicolumn{2}{c}{Experiment 1} & \multicolumn{2}{c}{Experiment 2} \\
        \textbf{Method} & \textbf{LR} & \textbf{GBT} 	& \textbf{LR} & \textbf{GBT} \\
        \midrule             
        Unweighted            & 1.342(0.030) & 0.489(0.035)     & 1.342(0.026) & 0.489(0.039) \\
        Entropy Balancing(1)  & 1.295(0.033) & 0.486(0.026)     & 1.295(0.030) & 0.487(0.035) \\
        Entropy Balancing(2)  & 1.194(0.025) & 0.466(0.036)     & 1.194(0.025) & 0.468(0.041) \\
        Entropy Balancing(3)  & \textbf{1.187(0.025)} & 0.459(0.032)     & \textbf{1.187(0.024)} & 0.457(0.036) \\
        Entropy Balancing(4)  & 1.189(0.024) & \textbf{0.457(0.035)}     & 1.189(0.023) & \textbf{0.452(0.034)} \\
        NBW                   & 1.274(0.038) & 0.488(0.035)     & 1.273(0.031) & 0.485(0.032) \\
        \bottomrule          
    \end{tabular}
    \centering
    \subcaption{$N=100000$}
    \label{Table_for_Experiment_Section_8_for_N_100000}
    \begin{tabular}{lcccc}
        \toprule            
        & \multicolumn{2}{c}{Experiment 1} & \multicolumn{2}{c}{Experiment 2} \\
        \textbf{Method} 	& \textbf{LR} & \textbf{GBT} 	& \textbf{LR} & \textbf{GBT} \\
        \midrule             
        Unweighted            & 1.342(0.027) & 0.457(0.048)     & 1.342(0.023) & 0.459(0.044) \\
        Entropy Balancing(1)  & 1.299(0.029) & 0.453(0.037)     & 1.298(0.027) & 0.455(0.036) \\
        Entropy Balancing(2)  & 1.195(0.025) & 0.391(0.034)     & 1.194(0.023) & 0.386(0.039) \\
        Entropy Balancing(3)  & \textbf{1.186(0.024)} & 0.361(0.025)     & \textbf{1.186(0.023)} & 0.360(0.023) \\
        Entropy Balancing(4)  & 1.188(0.024) & \textbf{0.353(0.022)}     & 1.187(0.023) & \textbf{0.356(0.020)} \\
        NBW                   & 1.239(0.095) & 0.376(0.033)     & 1.252(0.080) & 0.388(0.030) \\
        \bottomrule          
    \end{tabular}
\end{table*}

\newpage
\paragraph{Results.}
We report the average and standard errors of  the root mean squared error (RMSE) between the estimated and true values of the average causal effects for synthetic data of size $N=1000$, $10000$, and $100000$.
Table \ref{Table_for_Experiment_Section_8} lists the results of Experiments 1 and 2 for each $N$.
Each result is in the form of “mean (std. err.)” from 100 simulations.

As shown in all the results, the results of NBW were less accurate than those of the entropy-balancing method.
Moreover, the results for $N=1000$ shows that NBW were less accurate than the unweighted estimation.
However, as seen in all results for $N=100000$, the accuracy of NBW was superior to that of the unweighted estimation, which was close to the accuracy of the entropy-balancing method.
These results imply that the sample size requirements of the proposed method are larger than those of the entropy balancing method.

\newpage
\section{Back-Propagation Algorithm using Neural Balancing Weights} \label{Subsection_backporp_algo}
We show a back-propagation algorithm using NBW for MSE loss in Algorithm \ref{algo_cbp}.
The MSE loss here is calculated by the mean of the element wise product of both the original squared errors and the balancing weights.
\begin{figure}
    \begin{algorithm}[H]
        \caption{Back-Propagation Algorithm using Neural Balancing Weights}\label{algo_cbp}
        \begin{algorithmic}[1]
            \Require
            Data $(y, \mathbf{x}_1, \mathbf{x}_2, \ldots, \mathbf{x}_n, \mathbf{z}) = \{(y^i, \mathbf{x}_1^i, \mathbf{x}_2^i, \ldots, \mathbf{x}_n^i, \mathbf{z}^i)|i=1,2,\ldots,N\}$\\
            \hspace{0.27cm} A Neural Balancing Weight Model $T$
            \Ensure A Neural Network Model $f_{\phi}$ for Estimating $E_{\widetilde{P}}[Y|\mathbf{X},\mathbf{Z}]$
            \Repeat
            \State $\hat{y} \leftarrow f_{\phi}(\mathbf{x}_1, \mathbf{x}_2, \ldots, \mathbf{x}_n, \mathbf{z})$ \hspace{5.35cm}  // Forward Propagation
            \State $\mathit{BW}(\mathbf{x}_1, \mathbf{x}_2, \ldots, \mathbf{x}_n, \mathbf{z})
            \leftarrow  \frac{e^{-T(\mathbf{x}_1, \mathbf{x}_2, \ldots, \mathbf{x}_n, \mathbf{z})}}{\mathit{MEAN}(e^{-T(\mathbf{x}_1, \mathbf{x}_2, \ldots, \mathbf{x}_n, \mathbf{z})})}$
            \State $\mathit{Err}_{\phi} \leftarrow y - \hat{y}$  \hspace{7cm}  // Obtaining Errors for $f_{\phi}$
            \State $\mathcal{L}({\phi}) \leftarrow \mathit{MEAN}(\mathit{Err}_{\phi}\otimes \mathit{Err}_{\phi} \otimes \mathit{BW}(\mathbf{x}_1, \mathbf{x}_2, \ldots, \mathbf{x}_n, \mathbf{z}))$
            \hspace{1cm} // Calculating  Loss $\mathcal{L}({\phi})$
            \State $\phi \leftarrow  \phi  -  \nabla \mathcal{L}({\phi})$
            \Until{convergence}
        \end{algorithmic}
    \end{algorithm}
\end{figure}

\end{document}